\documentclass{article}

\usepackage{microtype}
\usepackage{graphicx}
\usepackage{subfigure}
\usepackage{booktabs} % for professional tables

\usepackage{hyperref}
\usepackage{url}

\usepackage[accepted]{icml2025}

\usepackage{amsmath}
\usepackage{amssymb}
\usepackage{mathtools}
\usepackage{amsthm}

\usepackage[capitalize,noabbrev]{cleveref}

\theoremstyle{plain}
\newtheorem{theorem}{Theorem}[section]

\newtheorem{lemma}[theorem]{Lemma}
\newtheorem{corollary}[theorem]{Corollary}
\theoremstyle{definition}
\newtheorem{definition}[theorem]{Definition}
\newtheorem{assumption}[theorem]{Assumption}
\theoremstyle{remark}

\usepackage{tabularx}
\usepackage{comment}
\newcommand{\indep}{\perp \!\!\! \perp}

\newcommand{\argmax}{\mathop{\mathrm{arg\,max}}}
\newcommand{\argmin}{\mathop{\mathrm{arg\,min}}}
\usepackage{booktabs} % For professional-looking tables
\usepackage{multirow} % For merging rows
\usepackage{siunitx} % For numerical alignment

\usepackage[textsize=tiny]{todonotes}

\icmltitlerunning{Enhancing Treatment Effect Estimation via Active Learning: A Counterfactual Covering Perspective}

\begin{document}

\twocolumn[
\icmltitle{Enhancing Treatment Effect Estimation via Active Learning: \\A Counterfactual Covering Perspective}

% It is OKAY to include author information, even for blind
% submissions: the style file will automatically remove it for you
% unless you've provided the [accepted] option to the icml2025
% package.

% List of affiliations: The first argument should be a (short)
% identifier you will use later to specify author affiliations
% Academic affiliations should list Department, University, City, Region, Country
% Industry affiliations should list Company, City, Region, Country

% You can specify symbols, otherwise they are numbered in order.
% Ideally, you should not use this facility. Affiliations will be numbered
% in order of appearance and this is the preferred way.
\icmlsetsymbol{equal}{*}

\begin{icmlauthorlist}
\icmlauthor{Hechuan Wen}{uq}
\icmlauthor{Tong Chen}{uq}
\icmlauthor{Mingming Gong}{unimelb,mzuai}
\icmlauthor{Li Kheng Chai}{hwq}
\icmlauthor{Shazia Sadiq}{uq}
\icmlauthor{Hongzhi Yin}{uq}
%\icmlauthor{}{sch}
%\icmlauthor{}{sch}
%\icmlauthor{}{sch}
\end{icmlauthorlist}

\icmlaffiliation{uq}{School of EECS, The University of Queensland, Australia}
\icmlaffiliation{hwq}{Health and Wellbeing Queensland, Australia}
\icmlaffiliation{unimelb}{School of Mathematics and Statistics, The University of Melbourne, Australia}
\icmlaffiliation{mzuai}{Department of Machine Learning, Mohamed
 bin Zayed University of Artificial Intelligence, United Arab Emirates}

\icmlcorrespondingauthor{Hongzhi Yin}{h.yin1@uq.edu.au}

% You may provide any keywords that you
% find helpful for describing your paper; these are used to populate
% the "keywords" metadata in the PDF but will not be shown in the document
\icmlkeywords{Machine Learning, ICML}

\vskip 0.3in
]

% this must go after the closing bracket ] following \twocolumn[ ...

% This command actually creates the footnote in the first column
% listing the affiliations and the copyright notice.
% The command takes one argument, which is text to display at the start of the footnote.
% The \icmlEqualContribution command is standard text for equal contribution.
% Remove it (just {}) if you do not need this facility.

%\printAffiliationsAndNotice{}  % leave blank if no need to mention equal contribution
\printAffiliationsAndNotice{} % otherwise use the standard text.

\begin{abstract}

%In this paper, we consider designing effective data acquisition criterion, under limited labeling budget, to strengthen the generalizability of the treatment effect estimator with more quality labeled data. 

%Numerous complex algorithms for treatment effect estimation have been developed in recent years. However, estimating treatment effects remains challenging, as these sophisticated estimators often struggle to realize their full potential when faced with insufficiently labeled training sets, such as the significant absence of counterfactual samples.
Although numerous complex algorithms for treatment effect estimation have been developed in recent years, their effectiveness remains limited when handling insufficiently labeled training sets due to the high cost of labeling the effect after treatment, e.g., expensive tumor imaging or biopsy procedures needed to evaluate treatment effects. Therefore, it becomes essential to actively incorporate more high-quality labeled data, all while adhering to a constrained labeling budget. To enable data-efficient treatment effect estimation, we formalize the problem through rigorous theoretical analysis within the active learning context, where the derived key measures -- \textit{factual} and \textit{counterfactual covering radius} determine the risk upper bound. To reduce the bound, we propose a greedy radius reduction algorithm, which excels under an idealized, balanced data distribution. To generalize to more realistic data distributions, we further propose FCCM, which transforms the optimization objective into the \textit{Factual} and \textit{Counterfactual Coverage Maximization} to ensure effective radius reduction during data acquisition. Furthermore, benchmarking FCCM against other baselines demonstrates its superiority across both fully synthetic and semi-synthetic datasets. Code: \url{https://github.com/uqhwen2/FCCM}

%Anonymous code is at: \url{https://anonymous.4open.science/r/PaperID_2979}.

%promotes the \textit{Factual} and \textit{Counterfactual Coverage} Maximization to gain a significantly higher radius reduction by relaxing the full coverage constraint on the dataset

%We empirically observe that a relaxation of less than 1\% coverage could result in huge risk upper bound reduction and significantly better estimation results.

%Inspired by that, we propose a straightforward greedy radius reduction algorithm, however, the reduction in counterfactual radius is hindered when applied to real-world datasets, despite its theoretical optimality in idealized scenarios.

\end{abstract}

\section{Introduction\label{section:intro}}

Understanding the causal effects of interventions is essential for making informed decisions, positioning treatment effect estimation as a fundamental tool with broad applications across diverse domains, including randomized control trials (RCTs) in medication \cite{pilat2015exploring}, A/B testing for business decision-making \cite{kohavi2015online} and government policy evaluation \cite{mackay2020government}, etc. However, real-world scenarios often involve the trade-off between cost and return, e.g., the high cost of tumor imaging or biopsy at scale frequently limits the amount of treatment effects (labels) collected from the individuals given a particular drug, which in turn impacts the accuracy of estimation. Such a challenge highlights the need for designing data-efficient treatment effect estimation methods, which can be formulated as the following optimization problem:
\begin{equation}
    \min_{\mathcal{S}\subset \mathcal{D}} \epsilon(f_{\mathcal{S}}) ~~~~\text{s.t.}~~|\mathcal{S}|\leq B,\label{eq:problem_1}
\end{equation} where $\epsilon$ is a risk metric e.g. precision in estimation of heterogeneous effect (PEHE) \cite{shalit2017estimating},  $\mathcal{D}$ is the pool set containing all candidate samples, $f_{\mathcal{S}}$ is the regression model trained on subset $\mathcal{S}$, $\mathcal{D}$ is the pool set, and $B$ is the labeling budget.

Depending on whether the treatment indicator $t$ is provided alongside with the covariates in pool set $\mathcal{D}$ or not, there are two branches of research related to the formulation in Eq. (\ref{eq:problem_1}). If $t$ is not given, one is to build the dataset $\mathcal{S}$ via active experimental design \cite{addanki2022sample,connolly2023task,ghadiri2024finite}, i.e., a subset of $\mathcal{D}$ consisting only of covariates is selected, which is then partitioned to receive treatments, and subsequently annotated with treatment effects. Otherwise, given the pool set $\mathcal{D}$ consisting of both the covariates and the treatment indicator \cite{sundin2019active,qin2021budgeted,jesson2021causal,wen2024progressive}, the algorithm only deals with selection from the pool set, then let the oracle label the selected samples. Although both aim to facilitate reliable, data-efficient predictions by selectively expanding the labeled dataset for training, studies on the latter are scarcer compared with those on the former. Despite the under-exploration of the latter setting, it has been widely encountered in real-world applications, such as collecting customer preferences from those who have already received different services, tracking patients' side effects from those who have already been administered different drugs, etc. Therefore, in this paper, we study the latter scenario where both the covariates and the treatment indicator are known in the pool set, due to the high cost of labeling the treatment effect at scale and limited budget, only a subset of them are selected for the oracle to label. 

Unlike traditional optimization problems where abundant training data is assumed for treatment effect estimation \cite{shalit2017estimating,shi2019adapting,jesson2020identifying,wang2024optimal}, Eq. (\ref{eq:problem_1}) is in-essence an active learning (AL) problem \cite{settles2009active} and considers the practical issue of the scarcity of labeled training data, in conjunction with a limited labeling budget to expand the dataset for training a more generalizable model. Despite the simple formulation, solving the problem is NP-hard due to the combinatorial nature of selecting the optimal subset of data points to label \cite{tsang2005core,settles2009active}. From a general AL perspective, %assuming treatments are binary, Eq. (\ref{eq:problem_1}), can be considered a special case of multi-class AL \cite{joshi2009multi,jain2009active,guo2015active} as it expects querying from two classes. However, the classification methods do not commonly apply to regression problems. Nevertheless, this does not preclude 
by simplifying the pre-assigned treatment as a feature variable, AL-based regression methods \cite{gal2017deep,sener2018active,ash2019deep} might be directly applicable to Eq. (\ref{eq:problem_1}). However, straightforward adoption of AL for Eq. (\ref{eq:problem_1}) is suboptimal as it omits the distribution alignment between different treatment groups during data acquisition.

Thus far, some designated AL approaches have been proposed to address the constrained regression problem in Eq. (\ref{eq:problem_1}) with observational data. \citet{qin2021budgeted} formalize a theoretical framework QHTE that directly extends the treatment effect estimation risk upper bound \cite{shalit2017estimating} by the core-set \cite{tsang2005core} approach, however, the derived optimizable quantities do not account for the importance of promoting distribution alignment during data acquisition. Following that, \citet{jesson2021causal} propose $\mu\rho$BALD from an information theory perspective, which reduces the distributional discrepancy between treatment groups by scaling the acquisition criterion with the inverse of counterfactual uncertainty. However, such a method relies heavily on the accuracy of the quantified uncertainty and the training of complex estimators, e.g., deep kernel learning \cite{wilson2016deep}. To get the best of both worlds, \citet{wen2024progressive} devise a simple yet performant algorithm MACAL, which considers the reduction of distributional discrepancy while remaining model-independent during data acquisition. However, MACAL has to query the data in pairs (i.e., treated/control group each has one), which hinders its generalizability when optimality can be achieved by querying from one treatment group. Additionally, it is hard to obtain the overall risk upper bound convergence despite the convergence analysis on sub-objectives.

\textbf{Contribution.} Considering the aforementioned theoretical and practical limitations of existing methods for data-efficient treatment effect estimation, this paper presents a three-fold contribution to this area of research. 1). We establish a theoretical framework rooted in the active learning paradigm, specifically tailored to address Eq. (\ref{eq:problem_1}) which is not directly optimizable. Unlike Causal-BALD, the proposed theorem outlines the \textit{model-independent} reducible quantities as the optimization alternative, i.e., \textit{factual} and \textit{counterfactual} covering radius. Our theorem further generalizes QHTE by additionally accounting for the distribution alignment with the \textit{counterfactual} covering radius. 2). In contrast to MACAL, we propose two model-independent algorithms, both can obtain the optimality by a single data acquisition instead of enforcing the pair query, to minimize the covering radius-based objective: a greedy radius reduction method that works well in idealized distributions, and a greedy \textit{factual} and \textit{counterfactual} coverage maximization method -- FCCM that allows for greater flexibility on the data distribution of the pool set; 3). We demonstrate the superiority of FCCM against other baselines on real-world covariates with extensive performance evaluations and qualitative visualizations.

\begin{comment}
    \textbf{Contribution.} In this paper, we establish a solid theoretical framework with straightforward overall risk convergence for data-efficient treatment effect estimation, and the contributions are summarized as follows:
\begin{itemize}
    \item Our proposed Theorem \ref{theorem:overall} unveils the crucial impact of the factual covering radius, as well as the distinctive role of the \textit{counterfactual covering radius}, on risk convergence.
    \item We illustrate the challenge of effectively reducing the risk upper bound with greedy Algorithm \ref{alg:fccs}, which is constrained by the full coverage of the dataset, on the real-world covariates. Then, by re-formalizing the optimization objective into a maximum coverage problem, we further propose the factual and counterfactual coverage maximum (FCCM) Algorithm \ref{alg:fccm} with theoretical guarantee.
    \item We demonstrate the superiority of the proposed algorithm against other baselines with extensive risk evaluations for performance and visualizations for further interpretability.
\end{itemize}

\end{comment}

\section{Preliminaries}
%\subsection{Data-Efficient Causal Effect Estimation}

%\textbf{Notation.}  We use 

\textbf{Treatment effect estimation.} In this paper, we estimate the treatment effect under the potential outcome framework \cite{imbens2015causal}. Let the covariate, treatment, and treatment outcome spaces be denoted as $\mathcal{X}$, $\mathcal{T}$, and $\mathcal{Y}$, respectively. We train the treatment effect estimator $f_{\mathcal{D}}:\mathcal{X}\times\mathcal{T}\rightarrow\mathcal{Y}$ based on the dataset $\mathcal{D}=\{\mathbf{x}_{i}, t_{i}, y_{i}\}_{i=1}^{N}$, where $\mathbf{x}_i$, $t_i$, $y_i$ are respectively the feature vectors, treatment assignment, treatment outcome that correspond to the $i$-th individual. For now, we consider binary treatments $t\in\{0,1\}$, and denote $Y^{t=1}$ and $Y^{t=0}$ as the potential outcomes with treatment assignment $t=1$ and $t=0$ respectively. The ground truth individual treatment effect (ITE) for individual $\mathbf{x}$ is defined as \cite{shalit2017estimating}:
\begin{equation}\label{eq:true_effect}
   \tau(\mathbf{x}) = \mathbb{E}[Y^{t=1} - Y^{t=0}| \mathbf{x}].
\end{equation} 
To evaluate the performance of the trained model $f_{\mathcal{D}}$, we adopt the expected precision in estimation of heterogeneous effect (PEHE) \cite{hill2011bayesian}, which is the go-to choice for various treatment effect estimation tasks \cite{shalit2017estimating,louizos2017causal,shi2019adapting,jesson2021causal,wang2024optimal}:

\begin{definition}\label{definition:pehe}
    \textit{The expected PEHE of the estimator $f$ with squared loss metric $\xi(\cdot)$ is defined as:}
\begin{equation}
    \epsilon_{\text{PEHE}}(f)=\int_{\mathcal{X}}\xi(\mathbf{x};f)p(\mathbf{x})d\mathbf{x},
    \label{eq:pehe}
\end{equation} where $\xi(\mathbf{x};f)=(\hat{\tau}(\mathbf{x})-\tau(\mathbf{x}))^{2}$, $\tau(\mathbf{x})$ is the ground truth effect defined in (\ref{eq:true_effect}), and $\hat{\tau}(\mathbf{x})=f(\mathbf{x},t=1)-f(\mathbf{x},t=0)$ is its estimation. The lower the PEHE value, the better the model performance.
\end{definition}

For the identifiability of the treatment effect $\tau(\mathbf{x})$, the following assumptions from the causal inference literature are the sufficient conditions to let it hold \cite{shalit2017estimating, pearl2009causality}:

\begin{assumption}[Consistency\label{assumption:consistency}]
Only one potential outcome is seen by each unit given the treatment $t$, i.e., $y=Y^{t=0}$ if $t=0$ or $y=Y^{t=1}$ if $t=1$.
\end{assumption}

\begin{assumption}[Strong Ignorability\label{assumption:strong_ingore}]
The independence relation $\{Y^{t=0}, Y^{t=1}\}\indep t | \mathbf{x}$ and the conditional probability $0<p(t=1|\mathbf{x})<1$ hold for all $\mathbf{x}$.%, where treatment assignment $t$ is independent to the potential outcomes $\{Y^{t=0}, Y^{t=1}\}$ given the covariate \mathbf{x}. %Note that the potential outcomes $Y$ use a different notation w.r.t. the observed ones $y$.
\end{assumption}

\textbf{Integrating active learning.} In situations where the labels $y_i$ in $\mathcal{D}$ are unavailable, we thus introduce AL to build a labeled training set $\mathcal{S}$ out of $\mathcal{D}$. To enhance treatment effect estimators via AL, we 1). feed the oracle-labeled dataset $\mathcal{S}=\{\mathbf{x}_{i}, t_{i}, y_{i}\}_{i=1}^{k}$ for training estimator $f_{\mathcal{S}}$; 2). evaluate the performance of the model; 3). determine if training can be terminated based on performance or labeling budget; 4). if \textit{no} to the above, select the unlabeled subset $\Tilde{\mathcal{S}}^{*}$ from the pool set $\mathcal{D}=\{\mathbf{x}_{i}, t_{i}\}_{i=1}^{n}$ for the oracle to label; 5). expand the fully labeled subset $\mathcal{S}$ with newly labeled data points $\Tilde{\mathcal{S}}^{*}$ and return to step 1). This recursive procedure terminates when the desired performance is reached or the labeling budget is exhausted, and to achieve the objective in Eq. (\ref{eq:problem_1}), the key is to identify the best strategy to construct $\mathcal{S}$.

%In Figure \ref{fig:alproces. illustrate the overall process of enhancing treatment effect estimator performance via AL, where 1). Feed the fully labeled dataset $\mathcal{D}_{\text{train}}=\{\mathbf{x}_{i}, t_{i}, y_{i}\}_{i=1}^{N_{\text{train}}}$ for model training; 2). Evaluate the performance of the model; 3). Terminate the entire process or not based on the result satisfaction or budget allowance; 4). Send the acquired unlabeled subset $\Tilde{\mathcal{D}}$ from the pool set $\mathcal{D}_{\text{pool}}=\{\mathbf{x}_{i}, t_{i}\}_{i=1}^{N_{\text{pool}}}$ for the oracle to label; 5). Incorporate the labeled subset into the training set, then return to step 1). Such a recursive procedure terminates mostly when the desired performance is reached or the labeling budget gets exhausted.

\textbf{Related work.} In addition to the analysis in Section \ref{section:intro}, we include a more detailed review of existing work on treatment effect estimation, AL, and data-efficient treatment effect estimation in Appendix \ref{appendix:related_works}.

\begin{comment}
\begin{figure}[t!]
    \centering
    \includegraphics[width=.9\linewidth]{figures/al.pdf}
    \caption{Active learning integrated treatment effect estimation.}
    \label{fig:alprocess}
\end{figure}
\end{comment}

\section{Bounds: A Counterfactual Covering Perspective}

We upper-bound Eq. (\ref{eq:problem_1}) in a general form under the AL paradigm, using a formulation similar to that defined in \cite{sener2018active} for the classification problem. Thus, the population risk -- $\epsilon_{\text{PEHE}}(f_{\mathcal{S}})$ is constrained as:
\begin{equation}
\begin{split}
    &\epsilon_{\text{PEHE}}(f_{\mathcal{S}})\\
    =&\epsilon_{\text{PEHE}}(f_{\mathcal{S}})-\frac{1}{n}\sum^{n}_{i=1}\xi(\mathbf{x}_{i};f_{\mathcal{S}})+\frac{1}{n}\sum^{n}_{i=1}\xi(\mathbf{x}_{i};f_{\mathcal{S}})-\\
    &\frac{1}{|\mathcal{S}|}\sum^{|\mathcal{S}|}_{j=1}l(\mathbf{x}_{j}, y_{j}, t_{j};f_{\mathcal{S}})+\frac{1}{|\mathcal{S}|}\sum^{|\mathcal{S}|}_{j=1}l(\mathbf{x}_{j}, y_{j}, t_{j};f_{\mathcal{S}})\\
    \leq&\underbrace{\left|\epsilon_{\text{PEHE}}(f_{\mathcal{S}})-\frac{1}{n}\sum^{n}_{i=1}\xi(\mathbf{x}_{i};f_{\mathcal{S}})\right|}_{\text{Generalization Error}}+\\
    &\underbrace{\left|\frac{1}{n}\sum^{n}_{i=1}\xi(\mathbf{x}_{i};f_{\mathcal{S}})-\frac{1}{|\mathcal{S}|}\sum^{|\mathcal{S}|}_{j=1}l(\mathbf{x}_{j}, y_{j}, t_{j};f_{\mathcal{S}})\right|}_{\text{Subset Generalization Gap }\Delta}+\\
    &\underbrace{\frac{1}{|\mathcal{S}|}\sum^{|\mathcal{S}|}_{j=1}l(\mathbf{x}_{j}, y_{j}, t_{j};f_{\mathcal{S}})}_{\text{Empirical Training Loss}}\label{eq:population_risk_bound},
\end{split}
\end{equation} where $\xi(\cdot;f_{\mathcal{S}})$ is in Definition \ref{definition:pehe} (incalculable without the counterfactual outcomes), and $l(\cdot;f_{\mathcal{S}})$ a loss function for the labeled training set $\mathcal{S}$ with observed potential outcomes.

The expected model risk $\epsilon_{\text{PEHE}}(f_{\mathcal{S}})$ at the population level with trained estimator $f_{\mathcal{S}}$ on subset $\mathcal{S}$ is controlled by three terms as shown in Eq. (\ref{eq:population_risk_bound}). The generalization error is bounded w.r.t. the size $n$ as in the general machine learning research \cite{vapnik1999overview}. Commonly, the counterfactual effect for the same unit is rarely observable, e.g., a patient undergoes only one type of surgery, or a customer experiences only a single version of the software in an A/B test. Consequently, rendering the term $\xi(\mathbf{x}_{i};f_{\mathcal{S}})$ non-computable. Nonetheless, the key is that we can bound the gap between the critical yet incalculable term with the training loss on $\mathcal{S}$ in the remaining two terms. By assuming zero training loss in the same fashion as \cite{sener2018active}, we can explicitly formalize Eq. (\ref{eq:problem_1}) into:
%We assume that the i.i.d. sample size n is large, with a little abuse of notation, we use the integral to represent the discrete loss summation over the sample space $\mathcal{X}^{i.i.d.}$. 
\begin{equation}
\begin{split}\label{eq:new_problem}
    \min_{\mathcal{S}\subset \mathcal{D}} \left|\frac{1}{n}\sum^{n}_{i=1}\xi(\mathbf{x}_{i};f_{\mathcal{S}})-\frac{1}{|\mathcal{S}|}\sum^{|\mathcal{S}|}_{j=1}l(\mathbf{x}_{i}, y_{i}, t_{i};f_{\mathcal{S}})\right| \\
    \text{s.t.}\quad|\mathcal{S}|\leq B,\quad\quad\quad\quad\quad\quad\quad
\end{split}
\end{equation} where the subset generalization gap $\Delta$ in $\epsilon_{\text{PEHE}}(f_{\mathcal{S}})$ is maintained as the final objective.

To optimize the objective in Eq. (\ref{eq:new_problem}) that is incalculable because of $\xi(\mathbf{x}_{i};f_{\mathcal{S}})$, we instead focus on identifying reducible quantities by capping the objective with a probabilistic risk upper bound outlined in Theorem \ref{theorem:overall}. The key reducible quantities, i.e., factual covering radius $\delta_{(t,t)}$, and the counterfactual covering radius $\delta_{(t,1-t)}$ are defined below.

\begin{definition}\label{definition:covering_radius}
    The covering radius $\delta_{(t,t')}$ is defined as the radius of the smallest ball centered at the labeled samples from treatment group $t$, such that the union of these balls covers all samples within group $t'$, $\forall t'\in\{t,1-t\}$, with factual covering induced by $t'=t$ and counterfactual covering induced by $t'=1-t$. 
\end{definition}

%\begin{definition}
%    The counterfactual covering radius $\delta_{(t,1-t)}$ is defined as the radius of the smallest ball centered at the labeled samples from treatment group $t$, such that the union of these circles covers all samples within group $1-t$.
%\end{definition}

%Firstly, we bound the expected value of the interested term $\frac{1}{n}\sum^{n}_{i=1}\xi(\mathbf{x}_{i};f_{\mathcal{S}})$ with factual and counterfactual loss, namely, $\epsilon_{F}$ and $\epsilon_{CF}$. Secondly, we constrain the $\epsilon_{F}$ with the factual covering radius and the $\epsilon_{CF}$ with the counterfactual covering radius. Lastly, we conclude the probabilistic bound for Eq. (\ref{eq:new_problem}) by the Hoeffding's inequality.

\begin{figure*}[t!]
  \centering
  
  \includegraphics[scale=0.46]{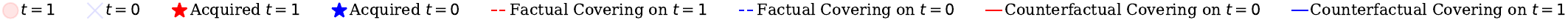}
  \subfigure[FC on $t=1$]{\includegraphics[width=0.24\textwidth]{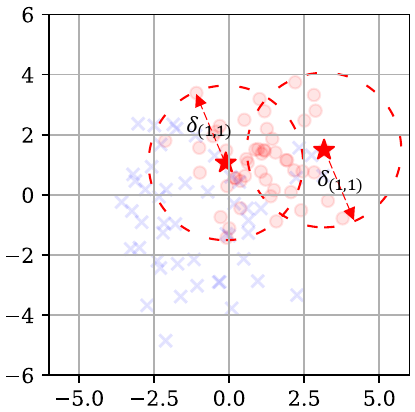}\label{fig:factual_covering_11}}
  \subfigure[CFC on $t=0$]{\includegraphics[width=0.24\textwidth]{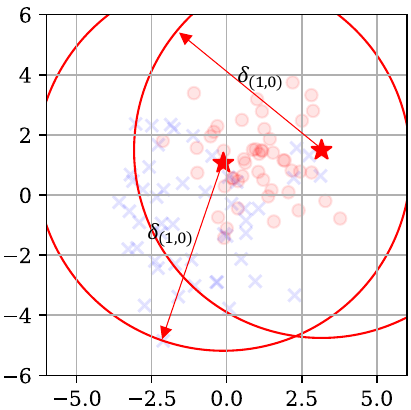}\label{fig:factual_covering_10}}
    \subfigure[FC on $t=0$]{\includegraphics[width=0.24\textwidth]{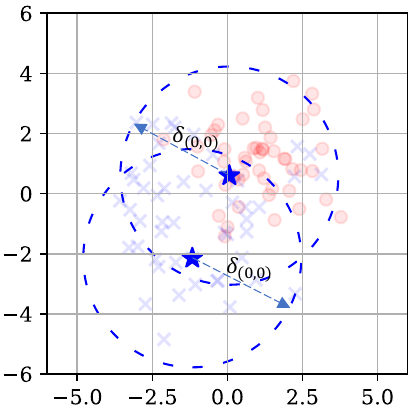}\label{fig:factual_covering_00}}
  \subfigure[CFC on $t=1$]{\includegraphics[width=0.24\textwidth]{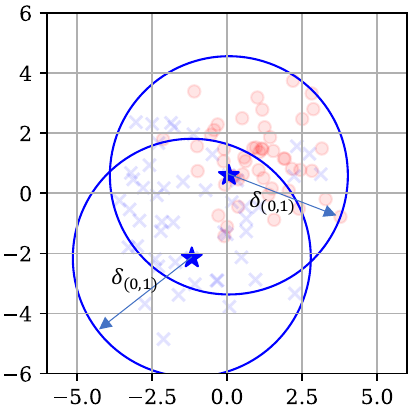}\label{fig:factual_covering_01}}
  %\vspace{-0.2cm}
  \caption{Visualization of the factual covering (FC) and the counterfactual covering (CFC) on the dataset by the acquired samples from each group. Note that each covering is constrained by the full coverage on the desired dataset with the minimum radius.\label{fig:covering_radius}}
  \vspace{-0.2cm}
\end{figure*}

To derive Theorem \ref{theorem:overall}, we further assume the Lipschitz continuity and the existence of the constant $\kappa$, with the discussions on their practicability given in Appendix \ref{appendix:discussions_on_assumptions}.

\begin{assumption}[Lipschitz Continuity]\label{assumption:Lipschitz}
    Assume that the conditional probability density function $p^{t}(y|\mathbf{x})$ is $\lambda_{t}$-Lipschitz, the squared loss function $l$ is $\lambda_{l}$-Lipschitz and $l$ is further bounded by $L_{l}$.
\end{assumption}

\begin{assumption}[Constant $\kappa$]\label{assumption:kappa}
    Let $\mathcal{H} = \{h | h: \mathcal{X}\rightarrow \mathbb{R}\}$ be a family of functions and $f:\mathcal{X}\times\mathcal{T}\rightarrow\mathcal{Y}$ be the hypothesis. Assume that there exists a constant $\kappa>0$, such that $h_{f}(\mathbf{x},t):=\frac{1}{\kappa}l_{f}(\mathbf{x},t)\in \mathcal{H}$.
\end{assumption}

\begin{theorem}\label{theorem:overall}
    Let $\mathbf{x}$ be sampled i.i.d. $n$ times from domain $\mathcal{X}$. Under Assumption \ref{assumption:Lipschitz} and Assumption \ref{assumption:kappa}, with probability at least $1-\gamma$, where $\gamma\in(0,1)$,  the subset generalization gap $\Delta$ is upper-bounded as:
    \begin{equation}
    \begin{split}
        &\left|\frac{1}{n}\sum^{n}_{i=1}\xi(\mathbf{x}_{i};f_{\mathcal{S}})-\frac{1}{|\mathcal{S}|}\sum^{|\mathcal{S}|}_{j=1}l(\mathbf{x}_{i}, y_{i}, t_{i};f_{\mathcal{S}})\right|\\
        \leq&\sum_{t\in\{0,1\}}\kappa_{t}\,\left(\delta_{(t,t)}
        +\delta_{(t,1-t)}\right) + 2\,\kappa_{\mathcal{H}}+\sqrt{\frac{L^{2}_{l}\log \frac{1}{\gamma}}{2n}},
    \end{split}
    \end{equation} where the constants $\kappa_{t}=2\,(\lambda_{l}+\frac{1}{3}\lambda_{t}L^{\frac{3}{2}}_{l})$, and $\kappa_{\mathcal{H}}=\kappa\cdot\text{IPM}_{\mathcal{H}}(p^{t=1}(\mathbf{x}), p^{t=0}(\mathbf{x}))$ with $\text{IPM}_{\mathcal{H}}(\cdot,\cdot)$ being the integral probability metric induced by $\mathcal{H}$, and $p^{t}(\mathbf{x})$ denotes the density distribution of treatment group $t$.
\end{theorem} 

Proof of the theorem is provided in Appendix \ref{appendix:theorem_1}. To enhance interpretation, we visualize the factual and counterfactual covering radius in Figure \ref{fig:covering_radius} with four labeled samples from a random two-dimensional toy dataset. For example, Figure \ref{fig:factual_covering_11} shows the factual covering with radius $\delta_{(1,1)}$ such that the union of the circles centered at the two labeled $t=1$ sample covers \textit{all} samples from group $t=1$. Noting that the covering radius $\delta_{(t,t')}$ decreases monotonically as the size of labeled set $\mathcal{S}$ grows under the AL paradigm, we present a further corollary to reveal the convergence of the subset generalization gap $\Delta$ given the fixed-size pool set $\mathcal{D}$.

\begin{corollary}[Informal]\label{corollary:covergence}
    Let $n$ be fixed, given $\delta_{(t,t')}$ decreases monotonically as $\mathcal{S}$ grows under the AL paradigm, then, with probability at least $1-\gamma$, we have:
    \begin{equation}
        \Delta=\mathcal{O}(\delta_{(1,1)})+\mathcal{O}(\delta_{(1,0)})+\mathcal{O}(\delta_{(0,0)})+\mathcal{O}(\delta_{(0,1)}).\label{eq:convergence}
    \end{equation}
\end{corollary} 

Observing that the factual covering radius is qualitatively related to the sample diversity showing an inverse trend, e.g., if the two marked samples in Figure \ref{fig:factual_covering_11} were overlapped (low diversity), the radius $\delta_{(1,1)}$ were larger, and the counterfactual covering radius is also qualitatively related to the distributional discrepancy with an inverse trend, e.g., if the two marked $t=1$ samples in Figure \ref{fig:factual_covering_10} were further away from the center of the group $t=0$, the radius $\delta_{(1,0)}$ were larger.  Thus, it is evident that the objective is to minimize the covering radius as much as possible within the labeling budget, thereby reducing the risk upper bound. Intuitively, the radius reduction given limited centers is analogous to the $k$-Center problem \cite{wolf2011facility}. However, under the treatment effect estimation setting, the derived bound introduces the counterfactual covering radius which fundamentally differs from the classical $k$-Center problem, where the centers only cover data points from the same class (i.e., treatment group in our case), as per Figure \ref{fig:factual_covering_11} and \ref{fig:factual_covering_00}. %Such that, the $k$-center solution does not work optimally to the joint reduction of the factual and counterfactual radius. 

\section{Methodologies}

In this section, we propose two greedy algorithms to minimize the four covering radii in Eq. (\ref{eq:convergence}). The first algorithm draws direct inspiration from the Corollary \ref{corollary:covergence} and extends the core-set solution \cite{tsang2005core,sener2018active} into the counterfactual covering perspective, while the second algorithm provides more flexibility on the data distribution and relaxes the full coverage constraint to achieve stronger radius reduction under the same labeling budget.

\subsection{Factual and Counterfactual Radii Reduction\label{section:fccs}}

Denote $\mathbf{x}^{t}_{i}$, $\mathcal{D}_{t}$, and $\mathcal{S}_{t}$ as the individual covariate vector, pool set, and labeled training set for the treatment group $t$, respectively. Note that $d(\cdot,\cdot)$ is a distance metric, and $\Tilde{\mathcal{S}}_{t}$ is a proxy collection which is explicitly explained in Appendix \ref{appendix:detailed_alg_1}. Motivated by the Corollary \ref{corollary:covergence}, the objective is to find the optimal subset $\mathcal{S}$ that minimizes the sum of the factual and the counterfactual covering radii as follows:
\begin{align}\label{eq:radius_reduction}
    \min_{\mathcal{S}=\mathcal{S}_{0}\cup \mathcal{S}_{1}, |\mathcal{S}|\leq B} \delta_{(1,1)} + \delta_{(1,0)} + \delta_{(0,0)} + \delta_{(0,1)},
    %& \left\{
    %\begin{aligned}
        %\delta_{(1,1)}&=\max_{i\in \mathcal{D}_{1}\backslash \mathcal{S}_{1}}\min_{j\in \mathcal{S}_{1}} d(\mathbf{x}^{t=1}_{i},\mathbf{x}^{t=1}_{j})\\
        %\delta_{(1,0)}&=\max_{i\in \mathcal{D}_{0}\backslash \Tilde{S}_{0}}\min_{j\in \mathcal{S}_{1}} d(\mathbf{x}^{t=0}_{i},\mathbf{x}^{t=1}_{j})\\
        %\delta_{(0,0)}&=\max_{i\in \mathcal{D}_{0}\backslash \mathcal{S}_{0}}\min_{j\in \mathcal{S}_{0}} d(\mathbf{x}^{t=0}_{i},\mathbf{x}^{t=0}_{j})\\
        %\delta_{(0,1)}&=\max_{i\in \mathcal{D}_{1}\backslash \Tilde{S}_{1}}\min_{j\in \mathcal{S}_{0}} d(\mathbf{x}^{t=1}_{i},\mathbf{x}^{t=0}_{j}).
    %\end{aligned}
    %\right.
\end{align} and for $t\in\{0,1\}$, $\delta_{(t,t)}=\max_{i\in \mathcal{D}_{t}\backslash \mathcal{S}_{t}}\min_{j\in \mathcal{S}_{t}} d(\mathbf{x}^{t}_{i},\mathbf{x}^{t}_{j})$, $\delta_{(t,1-t)}=\max_{i\in \mathcal{D}_{1-t}\backslash \Tilde{S}_{1-t}}\min_{j\in \mathcal{S}_{t}} d(\mathbf{x}^{1-t}_{i},\mathbf{x}^{t}_{j})$.

However, to say the least, this minimization is as difficult as solving the classic $k$-Center problem that is NP-hard \cite{cook1998combi}. The minimization of the factual covering radii, i.e., $\delta_{(1,1)}$ and $\delta_{(0,0)}$, is essentially the classical $k$-Center problem, while the minimization of the counterfactual covering radii, i.e., $\delta_{(1,0)}$ and $\delta_{(0,1)}$, involves the unconventional covering of one group by the centers from the other group, e.g., covering \textit{all} samples within group $t=1$ by the centers from group $t=0$ or vise versa, as visually depicted in Figure \ref{fig:factual_covering_10} and \ref{fig:factual_covering_01}.

To solve Eq. (\ref{eq:radius_reduction}), we provide a greedy radius reduction algorithm in Algorithm \ref{alg:fccs} (\textit{sketch}, with details in Appendix \ref{appendix:detailed_alg_1}), where the largest radius among the four radii is prioritized to be reduced. Recall Assumption \ref{assumption:strong_ingore} (Strong Ignorability), the counterfactual covering radii returned by our proposed algorithm can be effectively reduced under this assumption for the data distribution between the treatment groups, i.e., $0<p(t=1|\mathbf{x})<1$. Furthermore,  denoting the minimal covering radius under the optimal solution as ${OPT}_{\delta_{(\cdot,\cdot)}}$, we give the theoretical guarantee for Algorithm \ref{alg:fccs} as follows:

\begin{comment}
, we have the following assumption on data distribution:

\begin{definition}
    Let $\mathcal{S}^{*}_{t}$ of size $k_{t}$ denote the optimal (OPT) subset for treatment group $t$. For each point $v^{t}\in \mathcal{S}^{*}_{t}$, let the cluster of $v^{t}$ be $\mathcal{C}^{*}_{v^{t}}=\{u^{t}\in\mathcal{D}_{t}: d(u^{t},v^{t})=\min_{v'\in \mathcal{S}^{*}_{t}}d(u^{t},v')\}$. As such, we have partitions $\mathcal{C}^{*}_{v^{t}_1},\mathcal{C}^{*}_{v^{t}_2},\dots,\mathcal{C}^{*}_{v^{t}_{k_{t}}}$, where each point $u^{t}\in\mathcal{D}_{t}$ is placed in the closest $\mathcal{C}^{*}_{v^{t}_i}$ w.r.t. $v^{t}_{i}\in \mathcal{S}^{*}_{t}$.\label{definition:optimal_partition}
\end{definition}

\begin{assumption}
    %For $u^{1-t}\in \mathcal{D}_{1-t}$, $\forall v^{1-t}\in \Tilde{\mathcal{S}}^{*}_{1-t}$, $\exists a^{t}\in \mathcal{D}_{t}$, s.t. $d(a^{t}, v^{1-t})\leq\max_{u^{1-t}\in \mathcal{C}^{*}_{v^{1-t}}}d(u^{1-t},v^{1-t})$.
    $\Tilde{\mathcal{S}}_{1-t}\subset\mathcal{D}_{1-t}\cap\mathcal{D}_{t}$.\label{assumption:distribution}
\end{assumption}

\end{comment}
\begin{algorithm}[t!]
\caption{Greedy Radius Reduction (\textit{Sketch})}
\begin{algorithmic}[1]\label{alg:fccs}
   \STATE \textbf{Input:} $\mathcal{D}_{1}$, $\mathcal{D}_{0}$; randomly initialized $\mathcal{S}_{1}$, $\mathcal{S}_{0}$; budget $B$; pseudo operator $\Gamma=\argmax\min\,d(\cdot,\cdot)$
   \STATE $\mathcal{S}^{\text{init}}\leftarrow\mathcal{S}_{1}\cup\mathcal{S}_{0}, \mathcal{S}\leftarrow\mathcal{S}^{\text{init}}, \Tilde{S}_{1}\leftarrow\varnothing,\Tilde{S}_{0}\leftarrow\varnothing$
    \WHILE{$|\mathcal{S}|<|S^{\text{init}}|+B$}
    
    %\STATE Calculate $\delta_{(1,1)}, \delta_{(1,0)}, \delta_{(0,0)}, \text{and }\delta_{(0,1)}$

    \STATE Calculate $\delta = \max\{\delta_{(1,1)}, \delta_{(1,0)}, \delta_{(0,0)}, \delta_{(0,1)}\}$

    \IF{$\delta = \delta_{(1,1)}$ or $\delta = \delta_{(0,0)}$}

    \STATE Find point $a$ to reduce the factual radius via $\Gamma$.

    \ELSE \STATE Find proxy point $a'$ to reduce the counterfactual radius via $\Gamma$, and then compute corresponded $a$.
    
    \ENDIF
    
    \STATE $\mathcal{S}\leftarrow\mathcal{S}\cup\{a\}$ \#$a$ is not labeled.
    
    \ENDWHILE
    \STATE \textbf{Output:} $\mathcal{S}$ \#Labels in $\mathcal{S}$ is not used in querying.
\end{algorithmic}
\end{algorithm}

%Literally, Assumption \ref{assumption:distribution} describes that it is possible to find the point from the pool set $\mathcal{D}_{t}$ for group $t$ that also in the proxy collection $\Tilde{\mathcal{S}}_{1-t}$. Subsequently, 

\begin{theorem}\label{theorem:2opt}
Under Assumption \ref{assumption:strong_ingore}, the sum of the covering radii returned by Algorithm \ref{alg:fccs} is upper-bounded by $2\times\sum_{t\in \{0,1\}}(\text{OPT}_{\delta_{(t,t)}}+\text{OPT}_{\delta_{(t,1-t)}})$.
\end{theorem}

Proof is provided in Appendix \ref{appendix:theorem_2}. It is noted that data distribution can strongly affect the effectiveness of radius reduction by Algorithm \ref{alg:fccs}. For example, if two treatment groups share identical distributions, Eq. (\ref{eq:radius_reduction}) simply reduces to the conventional $k$-Center problem, for which a quick convergence of the covering radii is foreseeable. However, partially overlapped data distributions between groups can bring exponential challenges in effectively minimizing the sum of radii. In Figure \ref{fig:radius_convergence}, we illustrate the reduction of radii and their sum (denoted as Bound) in the normalized form against the size of acquired data, e.g., 1.0 represents the max value, 0.5 represents half of the max value. Under the identical distribution scenario, in Figure \ref{fig:ideal_convergence}, the factual and counterfactual covering radii decline synchronously (five plots fully overlap with each other) under a quick risk convergence to zero, because Eq. (\ref{eq:radius_reduction}) is reduced to the simple $k$-Center problem (where Assumption \ref{assumption:strong_ingore} surely satisfies) which guarantees a $2$-$OPT$ approximation and the greedy nature of Algorithm \ref{alg:fccs} forces the synchronous reduction. However, real-world data distribution commonly witnesses large group-wise discrepancies, e.g., in CMNIST benchmark \cite{jesson2021quantifying}, a significantly slower bound convergence is observed in Figure \ref{fig:realistic_convergence} due to the difficulty of consistently reducing the counterfactual covering radii $\delta_{(1,0)}$ and $\delta_{(0,1)}$. Note that the synchronous reduction of $\delta_{(1,1)}$ and $\delta_{(0,0)}$ is also forced by the greedy nature of Algorithm 1.

\begin{figure}[t!]
  \centering

  \subfigure[Identical Distribution]{\includegraphics[width=0.24\textwidth]{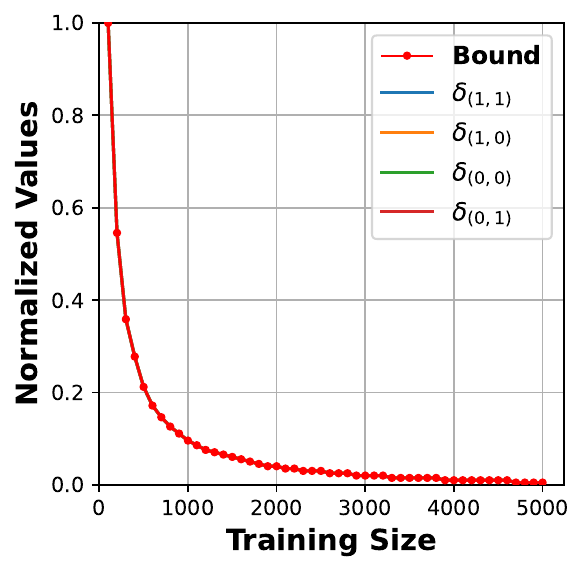}\label{fig:ideal_convergence}}
  \subfigure[CMNIST]{\includegraphics[width=0.224\textwidth]{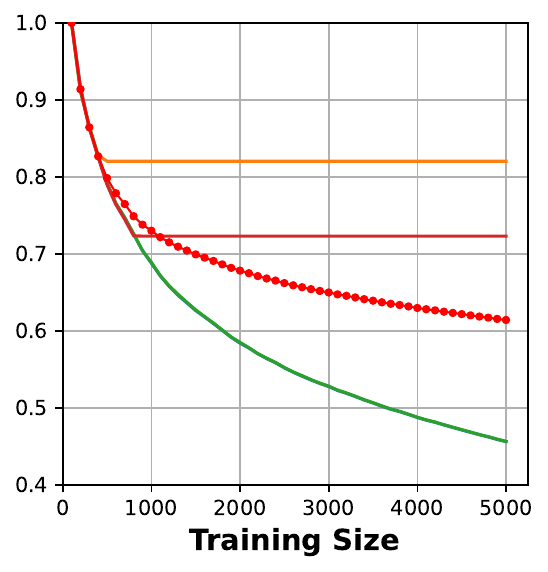}\label{fig:realistic_convergence}}
%\vspace{-0.4cm}
  \caption{Visuals of the radius reduction and the descent of the Bound under ideal and realistic data distributions by Algorithm \ref{alg:fccs}.}
  \label{fig:radius_convergence}
%\vspace{-0.4cm}
\end{figure}

\textbf{Discussions:} Given the probability threshold $\gamma$ and fixed-size pool set $\mathcal{D}$, Theorem \ref{theorem:overall} provides an adjustable upper-bound which depends on the declining covering radii as $\mathcal{S}$ grows. However, the derivation of the upper bound is constrained by the full coverage on the pool set, which challenges the counterfactual covering radius reduction due to the uncontrollable real-world data distribution. In light of this, we further explore the possibility of relaxing the full coverage constraint to a more flexible condition, i.e., approximating the full coverage given the relatively smaller fixed covering radius. This underlying mentality resonates with the duality discussion of the core-set and the max coverage problem in \cite{yehuda2022active}. In the following section, we propose the factual and counterfactual coverage maximization (FCCM) solution to solve the radius reduction problem under compromised data distributions while maintaining high satisfaction with the constraint.

\subsection{Counterfactual-integrated Coverage Maximization\label{section:fccm}}

 Let the factual and the counterfactual covering balls, for treatment group $t$ be defined as follows:

\begin{definition}\label{definition:covering_ball}
    Given the fixed covering radius $\delta_{(t,t')}>0$, the covering ball $\mathcal{A}_{(t,t')}$ centered at $\mathbf{x}\in \mathcal{S}_{t}$ is defined as: $\mathcal{A}_{(t,t')}(\mathbf{x})=\{\mathbf{x'}\in \mathcal{D}_{t'}:\|\mathbf{x}-\mathbf{x'}\|\leq\delta_{(t,t')}\}$, $\forall t'\in\{t,1-t\}$, with the factual covering ball induced by $t'=t$ and the counterfactual covering ball induced by $t'=1-t$.
\end{definition}

%\begin{definition}
%    Given the fixed counterfactual radius $\delta_{(t,1-t)}>0$, the counterfactual covering ball $\mathcal{A}_{(t,1-t)}$ centered at $\mathbf{x}\in \mathcal{D}_{t}$ is defined as: $\mathcal{A}_{(t,1-t)}(\mathbf{x})=\{\mathbf{x'}\in \mathcal{D}_{1-t}:\|\mathbf{x}-\mathbf{x'}\|\leq\delta_{(t,1-t)}\}$.
%\end{definition}

%Let $\mathcal{A}$ and $P(\mathcal{A})\in(0,1]$ respectively denote the overall covered region and the mean coverage rate (averaged from the four coverages introduced below). 

Therefore, let the union of the factual covering balls be $\mathcal{A}_{F}^{t}=\bigcup_{\mathbf{x}\in \mathcal{S}_{t}} \mathcal{A}_{(t,t)}(\mathbf{x})$ and the factual coverage be $P(\mathcal{A}_{F}^{t})=\frac{|\mathcal{A}_{F}^{t}|}{|\mathcal{S}_{t}|}\in(0,1]$; let the union of the counterfactual covering balls be $\mathcal{A}_{CF}^{t}=\bigcup_{\mathbf{x}\in \mathcal{S}_{t}} \mathcal{A}_{(t,1-t)}(\mathbf{x})$ and the counterfactual coverage be $P(\mathcal{A}_{CF}^{t})=\frac{|\mathcal{A}_{CF}^{t}|}{|\mathcal{S}_{1-t}|}\in[0,1]$. Then, we further define the mean coverage $P(\mathcal{A})$ and sum of the radii $\delta_{\text{sum}}$:
\begin{equation}
\begin{split}\label{eq:mean_coverage}
    P(\mathcal{A})=\frac{1}{4}(P(\mathcal{A}_{F}^{t=1})+P(\mathcal{A}_{CF}^{t=1})+P(\mathcal{A}_{F}^{t=0})+P(\mathcal{A}_{CF}^{t=0})),\\
    \delta_{\text{sum}}=\delta_{(1,1)} + \delta_{(1,0)} + \delta_{(0,0)} + \delta_{(0,1)}.~~~~~~~~~~~~~~~
\end{split}
\end{equation} With the underlying full coverage constraint, Eq. (\ref{eq:radius_reduction}) can be explicitly expressed as: \begin{equation}
\begin{split}\min_{\mathcal{S}\in\mathcal{D}}\,\delta_{\text{sum}}\quad\text{s.t. } P(\mathcal{A})-1=0\label{eq:full_coverage_constraint}.
\end{split}
\end{equation} Noting that in Section \ref{section:fccs}, we discuss the dilemma of reducing $\delta_{(t,1-t)}$ due to the large discrepancy that exists in the realistic dataset. To further suppress the interested bound -- $\delta_{\text{sum}}$, we transform Eq. (\ref{eq:full_coverage_constraint}) into the mean coverage maximization in Eq. (\ref{eq:coverage_maximization}) to maximally satisfy the equality constraint given smaller radius for the bound: \begin{algorithm}[t!]
\caption{FCCM\label{alg:fccm}}
\begin{algorithmic}[1]
    \STATE \textbf{Input:} Covariate matrix $\mathbf{X}\in\mathcal{D}$, random $\mathcal{S}^{\text{init}}$; $\mathcal{S}\leftarrow\mathcal{S}^{\text{init}}$; radius $\delta_{(t,t')},\forall t,t'\in\{0,1\}$; weight $\alpha$; budget $B$
    
    \STATE $\mathcal{E}_{(t,t)}=\{(\mathbf{x}^{t},\mathbf{x}'):\mathbf{x}'\in \mathcal{A}^{t}_{F}(\mathbf{x}^{t})\}$, 
    $\mathcal{E}_{(t,1-t)}=\{(\mathbf{x}^{t},\mathbf{x}'): \mathbf{x}'\in \mathcal{A}_{CF}^{t}(\mathbf{x}^{t})\}$, $\mathcal{E}^{t=1}=\mathcal{E}_{(t,t)}\cup\mathcal{E}_{(t,1-t)}$\

    $W(\mathbf{x}^t, \mathbf{x}') =
    \begin{cases}
    1, & \text{if } \mathbf{x}' \in \mathcal{A}_{F}^t(\mathbf{x}^t) \\
    \alpha, & \text{if } \mathbf{x}' \in \mathcal{A}_{CF}^t(\mathbf{x}^t) \\
    \end{cases}$,
    
    $\text{Weighted \,}\mathcal{G}=(\mathcal{V}=\mathbf{X},\mathcal{E}=\mathcal{E}^{t=1}\cup \mathcal{E}^{t=0},\mathcal{W}=W)$

    \FOR{$v\in \mathcal{S}^{\text{init}}$}
    \STATE Remove the edges to the covered vertices: $\{(\mathbf{x}',\mathbf{x}^{t})\in \mathcal{E}_{(t,t)},(\mathbf{x}^{t},\mathbf{x}')\in \mathcal{E}_{(t,1-t)}:(v^{t},\mathbf{x}^{t})\in \mathcal{E}_{(t,t)},\forall t\in\{0,1\})\}$ 
    \ENDFOR
    
    \WHILE{$|\mathcal{S}|\leq|S^{\text{init}}|+B$} 

    \STATE $\mathcal{S}=\mathcal{S}\cup\{v\}$, where $v$ is the vertice with the highest scaled out-degree in graph $\mathcal{G}$ (see Appendix \ref{appendix:more_explanation_alg_2}).
    
    \STATE Remove the edges to the covered vertices: $\{(\mathbf{x}',\mathbf{x}^{t})\in \mathcal{E}_{(t,t)},(\mathbf{x}^{t},\mathbf{x}')\in \mathcal{E}_{(t,1-t)}:(v^{t},\mathbf{x}^{t})\in \mathcal{E}_{(t,t)}),\forall t\in\{0,1\}\}$ 

    \ENDWHILE
    \STATE \textbf{Output:} $\mathcal{S}$ \#Labels in $\mathcal{S}$ is not used in querying.
\end{algorithmic}
\end{algorithm} \begin{equation}\label{eq:coverage_maximization}
    \max_{\mathcal{S}\in\mathcal{D}} P(\mathcal{A}).
\end{equation} To solve (\ref{eq:coverage_maximization}), we propose a greedy solution -- factual and counterfactual coverage maximization (FCCM) in Algorithm \ref{alg:fccm}. Specifically, FCCM constructs a weighted graph $\mathcal{G}$ with the node $V$ by the entire covariate matrix  $\mathbf{X}\in\mathcal{D}$, and each node $v^{t}\in\mathcal{D}_{t}$ builds the directed edge $e(v^{t},u^{t})$ (with unit weight) pointing to the node $u^{t}$ within the ball $\mathcal{A}_{(t,t)}(v^{t})$, thus constructing a directed graph $\mathcal{G}_{t}$. Then, graph $\mathcal{G}_t$ is further expanded by adding the weighted edges $e(v^{t},u^{1-t})$ (with weight $\alpha$) for $v^{t}\in\mathcal{D}_{t}$ pointing to the node $u^{1-t}$ within the counterfactual ball $\mathcal{A}_{(t,1-t)}(v^{t})$. Once  $\mathcal{G}_{t=1}$ and $\mathcal{G}_{t=0}$ are both obtained, they are naturally connected to build the final weighted graph $\mathcal{G}$. Note that FCCM differs from ProbCover \cite{yehuda2022active} which only works on a single class/group, whereas FCCM not only handles binary treatment groups but also integrates the distinctive counterfactual covering to solve Eq. (\ref{eq:coverage_maximization}).

Unlike Algorithm \ref{alg:fccs}, which can be regarded as a \textit{top-down} approach that optimizes the risk upper bound throughout the AL process with the equality constraint satisfied at each step, Algorithm \ref{alg:fccm} adopts a \textit{bottom-up} approach, working on satisfying the equality constraint under a fixed bound. To further give Algorithm \ref{alg:fccm} a theoretical guarantee in Theorem \ref{theorem:coverage} (with proof provided in Appendix \ref{appendix:theorem_3}), we assume that the full factual and counterfactual coverage are possible by the returned labeled set $\mathcal{S}$ as follows:

\begin{assumption}\label{assumption:full_coverage}
    Given the fixed covering radius $\delta_{(t,t)}$ and $\delta_{(t,1-t)}$, there exists the optimal solution $\mathcal{S}^{*}_{t}$, $\mathcal{S}^{*}_{t}\subset \mathcal{S}^{*}$ for treatment group $t$, $\forall t\in\{0,1\}$, such that $\mathcal{A}_{F}^{t}\bigcup\mathcal{A}_{CF}^{t}=\mathcal{D}.$
    %, resulting in partitions $\mathcal{C}^{*}_{s_{i}} = \{u\mid u\in \mathcal{D}_{t}\wedge d(u, s_i)\leq \delta_{(t,t)}\}\cup\{u\mid u\in \mathcal{D}_{1-t}\wedge d(u, s_i)\leq \delta_{(t,1-t)}\}$, such that $\bigcup_{i\leq k_{t}}\mathcal{C}^{*}_{s_{i}}=\mathcal{D}$.
\end{assumption}

%We can have the simple visualization of the Assumption 2 in Figure to help understand.

\begin{figure}[t]
  \centering
  \subfigure[Coverage on CMNIST]{\includegraphics[width=0.235\textwidth]{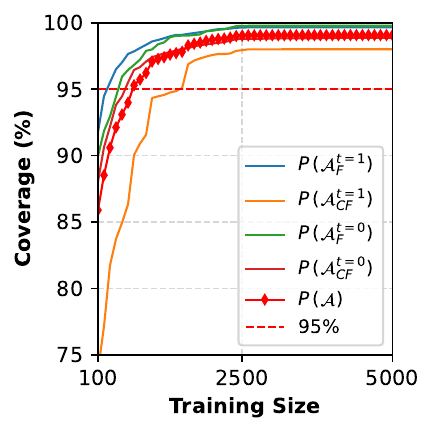}\label{fig:coverage}}
  \subfigure[Gain vs. Loss]{\includegraphics[width=0.235\textwidth]{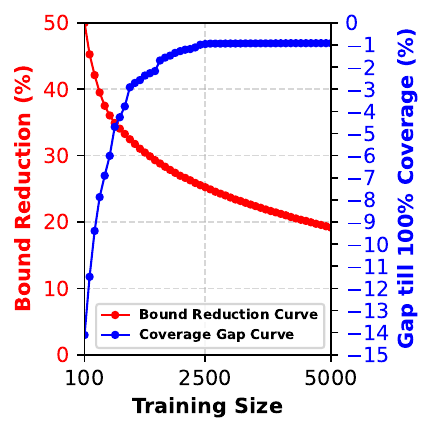}\label{fig:comparision}}
%\vspace{-0.3cm}
  \caption{Visualization of the high coverage by Algorithm \ref{alg:fccm} on CMNIST, and reduction gain over mean coverage loss by Algorithm \ref{alg:fccm} when compared to Algorithm \ref{alg:fccs}.}
  \label{fig:counterfactual_covering}
%\vspace{-0.5cm}
\end{figure}

\begin{theorem}\label{theorem:coverage}
    Under Assumption \ref{assumption:full_coverage}, Algorithm \ref{alg:fccm} is a $(1-\frac{1}{e})$ -- approximation for the full coverage constraint on the equally weighted graph and unscaled out-degree.
\end{theorem}

\textbf{Exploratory analysis}: Following our discussions on Algorithm \ref{alg:fccs}, we conduct the exploratory analysis on the CMNIST benchmark \cite{jesson2021quantifying} in Figure \ref{fig:coverage} to demonstrate the high coverage efficiency by Algorithm \ref{alg:fccm} on the real-world covariates. Further experiment on the performance compared to Algorithm \ref{alg:fccs} is accessible in Figure \ref{fig:comparision}, where both the high reduction of the sum of the radii and the low compromise of the coverage are observed, e.g., maximally 25\% reduction (red line) with 1\% mean coverage loss (blue line) around the training size of 2500.

%\textbf{Approximate the covering radius}: Details on estimating the fixed-size covering radius as the hyperparameter is provided in Appendix \ref{appendix:estimating_radius}.

\subsection{Approximating the Covering Radius\label{section:estimating_radius}}

The hyperparameters, i.e., $\delta_{(t,t)}$ and $\delta_{(t,1-t)}$ are crucial for the success of obtaining a relatively lower upper bound without heavily intruding into the full coverage constraint. For simplicity, we set the same size for all four radii and denote them as $\delta$ uniformly, in Figure \ref{fig:determine_delta}, we experiment different values of $\delta$ (normalized w.r.t. the max distance between points), and calculate the converged coverage on the growing training set given the pre-set covering radius. Thus, we estimate a narrower range for the relatively small radius with an achievable high mean coverage around the 95\% threshold, to reduce the search space for further hyperparameter tuning. 

\textbf{Uniform $\delta$:} The rationale to use uniform $\delta$ is to avoid making the search space prohibitively large -- on the order of $\mathcal{O}(m^4)$ if each radius has $m$ candidate settings. The key insight to reduce the search complexity is that, by the definition of the mean coverage $P(\mathcal{A})$ in Eq. (\ref{eq:mean_coverage}), Definition \ref{definition:covering_radius}, and Definition \ref{definition:covering_ball}, each radius is independent and each sub-term of $P(\mathcal{A})$ increases monotonically with its corresponding radius, which leads to the mean coverage $P(\mathcal{A})$ increases monotonically. Thus, by the independence and monotonicity, the search for four radii stays in the same direction, making the initial search space $\mathcal{O}(m)$ to identify the smallest radius for satisfying the 95\% mean coverage threshold. 

\begin{figure}[t!]
  \centering

  \subfigure[TOY]{\includegraphics[width=0.168\textwidth]{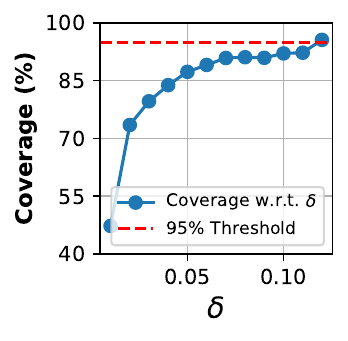}\label{fig:toy_delta}}
  \subfigure[IBM]{\includegraphics[width=0.145\textwidth]{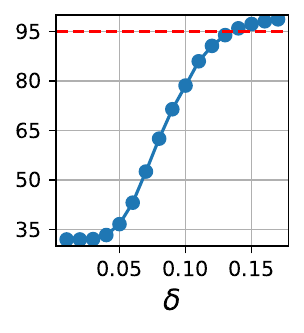}\label{fig:ibm_delta}}
  \subfigure[CMNIST]{\includegraphics[width=0.15\textwidth]{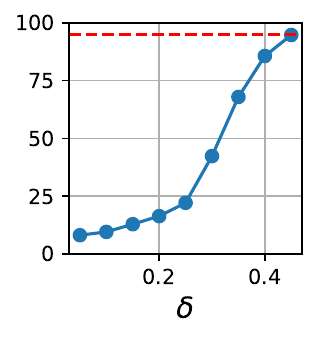}\label{fig:cmnist_delta}}
%\vspace{-0.3cm}
  \caption{Estimating a smaller range for the covering radius $\delta$ around the 95\% coverage threshold by Algorithm \ref{alg:fccm}.}
  \label{fig:determine_delta}
%\vspace{-0.4cm}
\end{figure}

\textbf{Different $\delta$:} Though the shared value among four radii allows for effective and efficient initial hyperparameter tuning, it could potentially fail if the distribution discrepancy between the two treatment groups is very large. This is because the counterfactual covering radius can be far larger than the factual covering radius, i.e., $\delta_{(t,1-t)}\gg\delta_{(t,t)}$, to let the counterfactual coverage $P(\mathcal{A}^{t}_{CF})$ to get close to full, lowering the utility of the uniform radius value. As such, if the distribution discrepancy between treatment groups is large, it is then viable to identify a different value for each covering radius to maintain the high coverage under the linear complexity due to the independence and monotonicity.

\begin{comment}
The following comment out part is discontinued for proofing. Find the new one next page.

We claim the following from Berlind \& Urner (2015). For $p, p'\in[0,1]$, \begin{equation}
    p_{y\sim p}(y)\leq p_{y\sim p'}(y) + |p-p'|
\end{equation}
\begin{subequations}
    \begin{align}
        &\int_{\mathcal{X}}(f^{t=1}(\mathbf{x}_{i})-\hat{f}^{t=1}(\mathbf{x}_{i}))^{2}p(\mathbf{x}_{i},t=0)d\mathbf{x}\\
    \leq&\int_{\mathcal{X}}(f^{t=1}(\mathbf{x}_{i})-\hat{f}^{t=1}(\mathbf{x}_{i}))^{2}(p(\mathbf{x}_{j},t=1)+|p(\mathbf{x}_{i},t=0)-p(\mathbf{x}_{j},t=1)|)d\mathbf{x}\\
    \leq&\int_{\mathcal{X}}l_{\mathbf{x}_{i}}(p(\mathbf{x}_{j},t=1)+|p(\mathbf{x}_{i},t=0)-p(\mathbf{x}_{i},t=1)+p(\mathbf{x}_{i},t=1)-p(\mathbf{x}_{j},t=1)|)d\mathbf{x}\\
    \leq&\int_{\mathcal{X}}l_{\mathbf{x}_{i}}(p(\mathbf{x}_{j},t=1)+|p(\mathbf{x}_{i},t=0)-p(\mathbf{x}_{i},t=1)|+|p(\mathbf{x}_{i},t=1)-p(\mathbf{x}_{j},t=1)|)d\mathbf{x}\\
    \leq&\int_{\mathcal{X}}l_{\mathbf{x}_{i}}p(\mathbf{x}_{j},t=1)d\mathbf{x}+L_{\tau}\int_{\mathcal{X}}l_{\mathbf{x}_{i}}d\mathbf{x}+\delta^{t=1}_{CF}\lambda_{t=1}\int_{\mathcal{X}}l_{\mathbf{x}_{i}}d\mathbf{x}
    \end{align}
\end{subequations}
\end{comment}

\section{Experiments}

%Since the counterfactual effects are hardly observed in the real world, we thus take the widely-recognized practice to utilize the fully-synthetic and semi-synthetic datasets for model performance evaluations.

%(among them 6k with treatment assignment $t=1$ and 10k with the flipped assignment $t=0$)

\textbf{Datasets:} \textbf{Toy} -- a simulated 2-dimensional toy dataset based on 16,000 randomly generated samples. \textbf{IBM} \cite{shimoni2018benchmarking} -- a high-dimensional tabular dataset based on the publicly available Linked Births and Infant Deaths Database. Each simulation contains 25,000 samples with 177 real-world covariates randomly selected from a cohort of 100,000 individuals; \textbf{CMNIST} \cite{jesson2021quantifying} -- This dataset contains 60,000 image samples (10 classes) of size 28$\times$28, which are adapted from MINIST \cite{lecun1998mnist} benchmark. Further details are deferred to Appendix \ref{appendix:dataset}.

\textbf{Metric:} The PEHE defined in Definition \ref{definition:pehe} with the squared root empirical expression: $\sqrt{\epsilon_{\text{PEHE}}}=\sqrt{\Sigma_{i=1}^{N}((y^{t=1}_{i}-y^{t=0}_{i})-\tau_{i})^{2}/N}$, is used for measuring the risk of the estimator at the individual level.

\begin{table}[t!]
\centering
\caption{Summary of the Acquisition Setup and Testing}\label{table:acquisition_summary}
\small
\begin{tabularx}{.48\textwidth}{X *{6}{>{\centering\arraybackslash}X}} % Using X for the first column, left-aligned
\toprule
\multicolumn{1}{l}{Dataset} & \multicolumn{1}{c}{Start} & \multicolumn{1}{c}{Length} & \multicolumn{1}{c}{Steps} & \multicolumn{1}{c}{Pool} & \multicolumn{1}{c}{Val} & \multicolumn{1}{c}{Test} \\
\midrule
%Toy &10&10&25&2625&750&875\\
TOY &ALL*&1&50&7200&2880&1600\\
IBM &ALL*&50&50&2891&3180&6250\\
CMNIST &ALL*&50&50&16706&10500&18000\\
\bottomrule
\end{tabularx}
%\vspace{-0.4cm}
\end{table}

\textbf{Baselines:} We compare FCCM to two groups of models, namely, the general AL model from the broader research field: BADGE \cite{ash2019deep}, BAIT \cite{ash2021gone}, and LCMD \cite{holzmuller2023framework}. And the designated model for treatment effect estimation with AL: QHTE \cite{qin2021budgeted}, Causal-Bald \cite{jesson2021causal} (with variants $\mu$BALD, $\rho$BALD, and $\mu\rho$BALD), and MACAL \cite{wen2024progressive} are the designated algorithms proposed to deal with the treatment effect estimation with AL.

\textbf{Estimators:} The same setup as described in \cite{jesson2021causal,wen2024progressive} is adopted here by utilizing the following two open-source estimators: \textbf{DUE-DNN} \cite{van2021feature} for tabular data, and \textbf{DUE-CNN} \cite{van2021feature} for image data. Model training details can be found in Appendix \ref{appendix:model_trianing}.

%It is one of the SOTA deep kernel learning frameworks with the multi-layer perceptron as the common feature extractor and two sparse Gaussian process regressions defined over the extracted latent features as the downstream estimators for different treatment groups' effect estimations. 

%It is a variant of the DUE model especially catering for the image-as-input experiment. It has a similar structure as DUE-DNN besides the latent feature extractor being replaced by the convolutional neural network (CNN), e.g., the ResNet \cite{he2016deep} is embedded. The computation resources and hyperparameter selection are described in Appendix \ref{appendix:hyperparameters}.  

\begin{comment}
    \begin{table}[t!]
\centering
\caption{Ablation Study of the Counterfactual Covering Radii}\label{table:ablation_study}
\small
\begin{tabularx}{.48\textwidth}{X *{5}{>{\centering\arraybackslash}X}} % Using X for the first column, left-aligned and vertical lines
\toprule
\multicolumn{1}{c}{Query Stage} & \multicolumn{1}{c}{1/5} & \multicolumn{1}{c}{2/5} & \multicolumn{1}{c}{3/5} & \multicolumn{1}{c}{4/5} & \multicolumn{1}{c}{5/5}\\
\midrule
Performance &+8\%&+16\%&+15\%&+14\%&+14\%\\
%TOY &0\%&0\%&18\%&18\%&16\%\\
%IBM &-1\%&-1\%&-1\%&0\%&0\%\\
%CMNIST &24\%&30\%&28\%&25\%&24\%\\
\bottomrule
\end{tabularx}
\end{table}
\end{comment}

\begin{figure*}[t!]
  \centering

  \includegraphics[width=.1\textwidth]{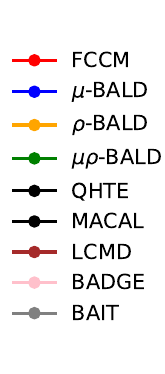}
  \subfigure[TOY]{\includegraphics[width=0.295\textwidth]{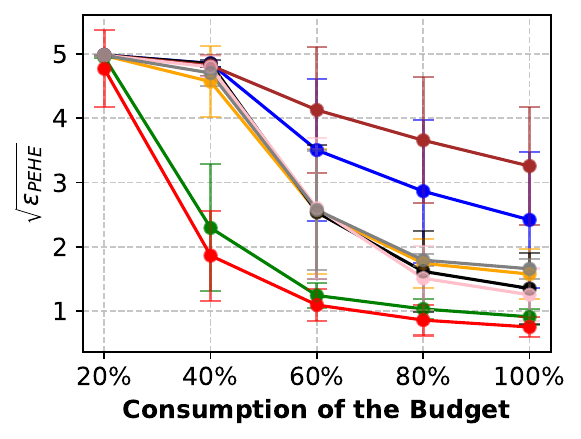}\label{fig:toy_results}}
    \subfigure[IBM]{\includegraphics[width=0.295\textwidth]{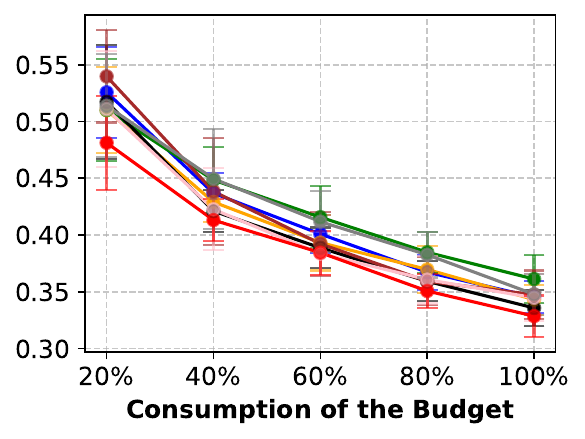}}
  \subfigure[CMNIST]{\includegraphics[width=0.295\textwidth]{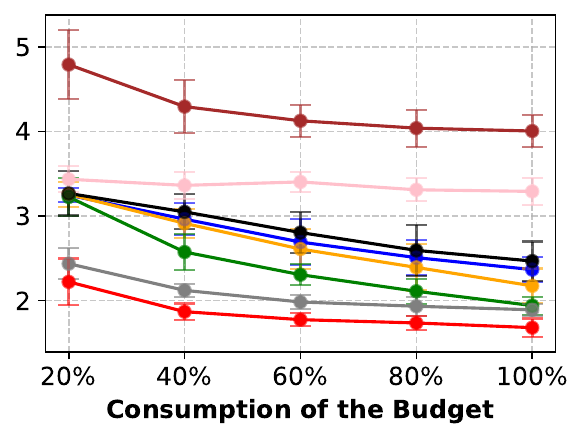}}
\vspace{-0.1cm}
  \caption{All plots are the mean values averaged from 10 simulations associated with the standard deviation as the error bar. Note that all models at 0\% exhibit the same performance given the fixed estimators and are thus neglected. The performance under 2\% granularity is presented in Appendix \ref{appendix:high_resolution}.}
  \label{fig:main_results}
\vspace{-0.2cm}
\end{figure*}

\begin{figure*}[t!]
  \centering
    
  \subfigure[Data Distribution]{\includegraphics[width=0.235\textwidth]{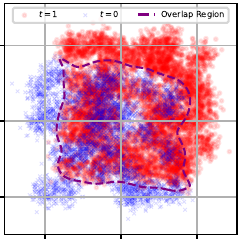}\label{fig:data_distribution}}
  \subfigure[FCCM]{\includegraphics[width=0.24\textwidth]{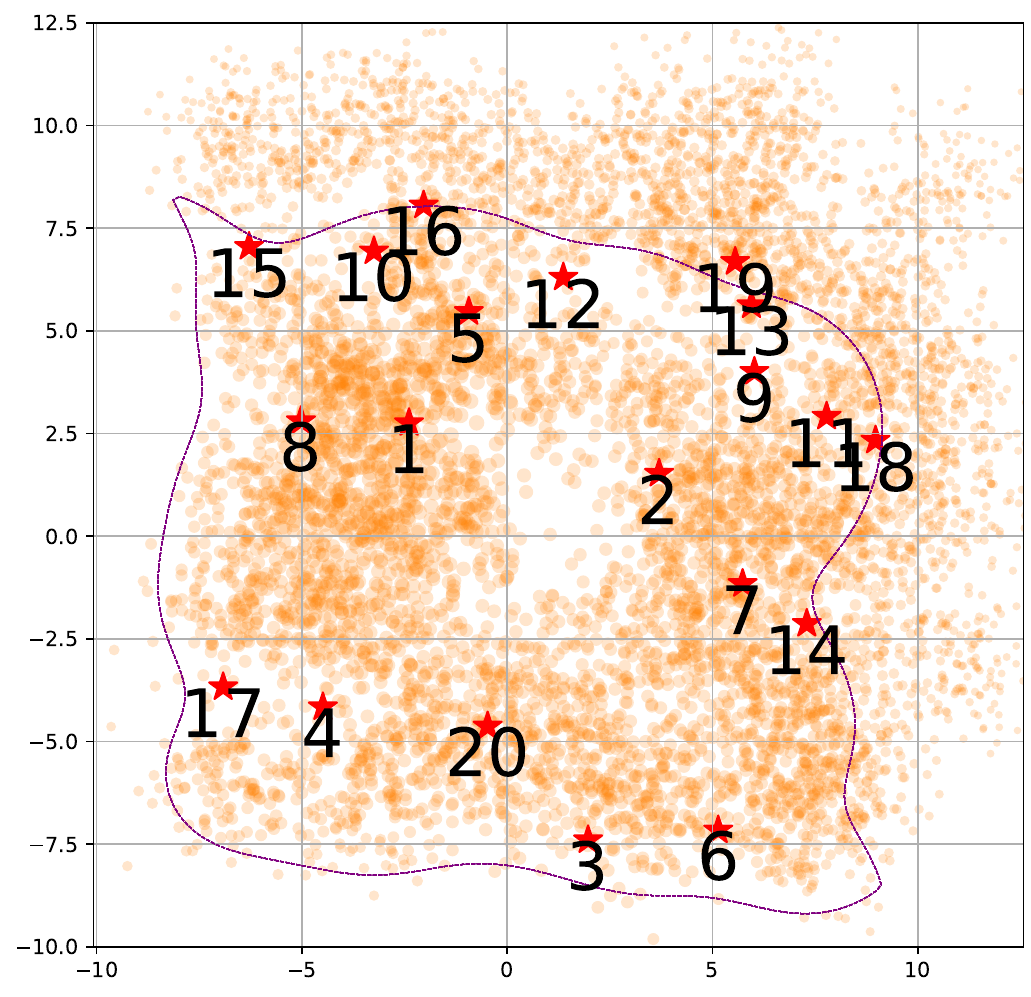}\label{fig:FCCM_acquisition}}
    \subfigure[$\mu\rho$BALD]{\includegraphics[width=0.24\textwidth]{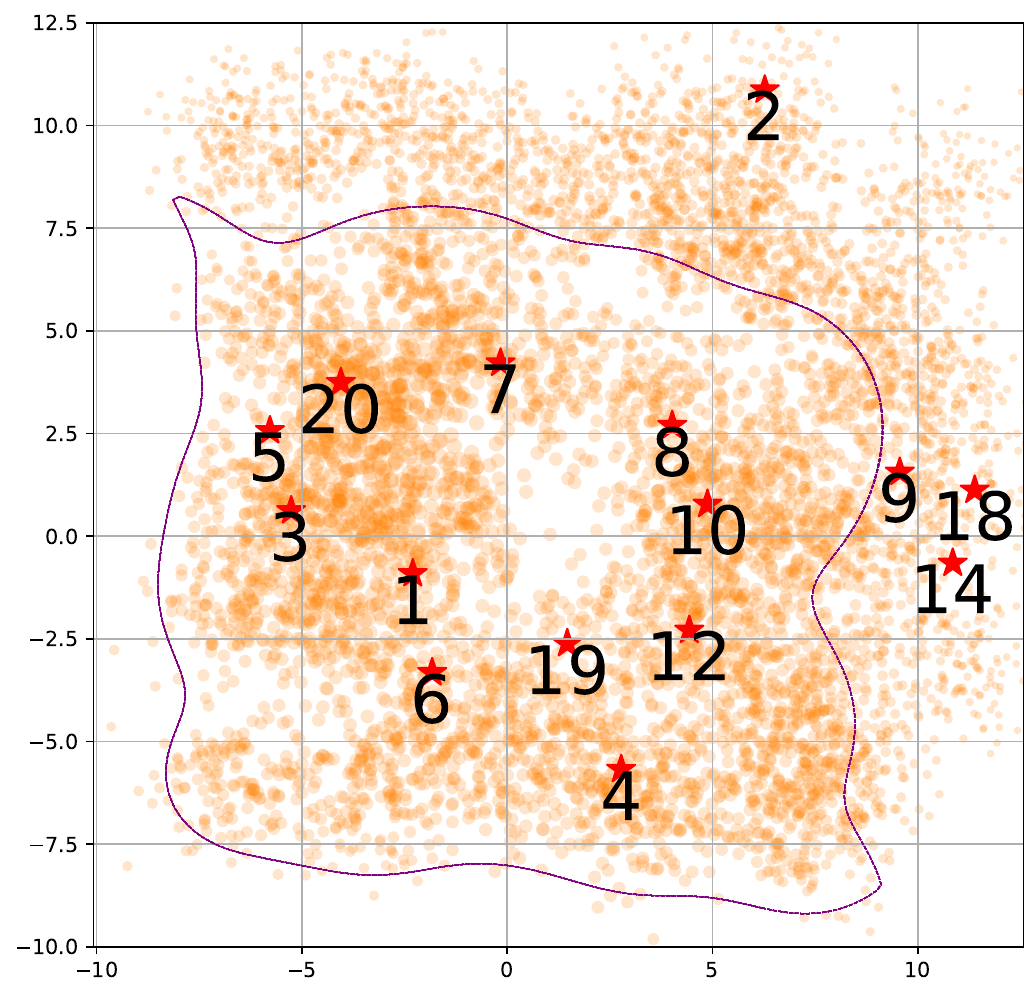}\label{fig:BALD_acquisition}}
  \subfigure[BAIT]{\includegraphics[width=0.26\textwidth]{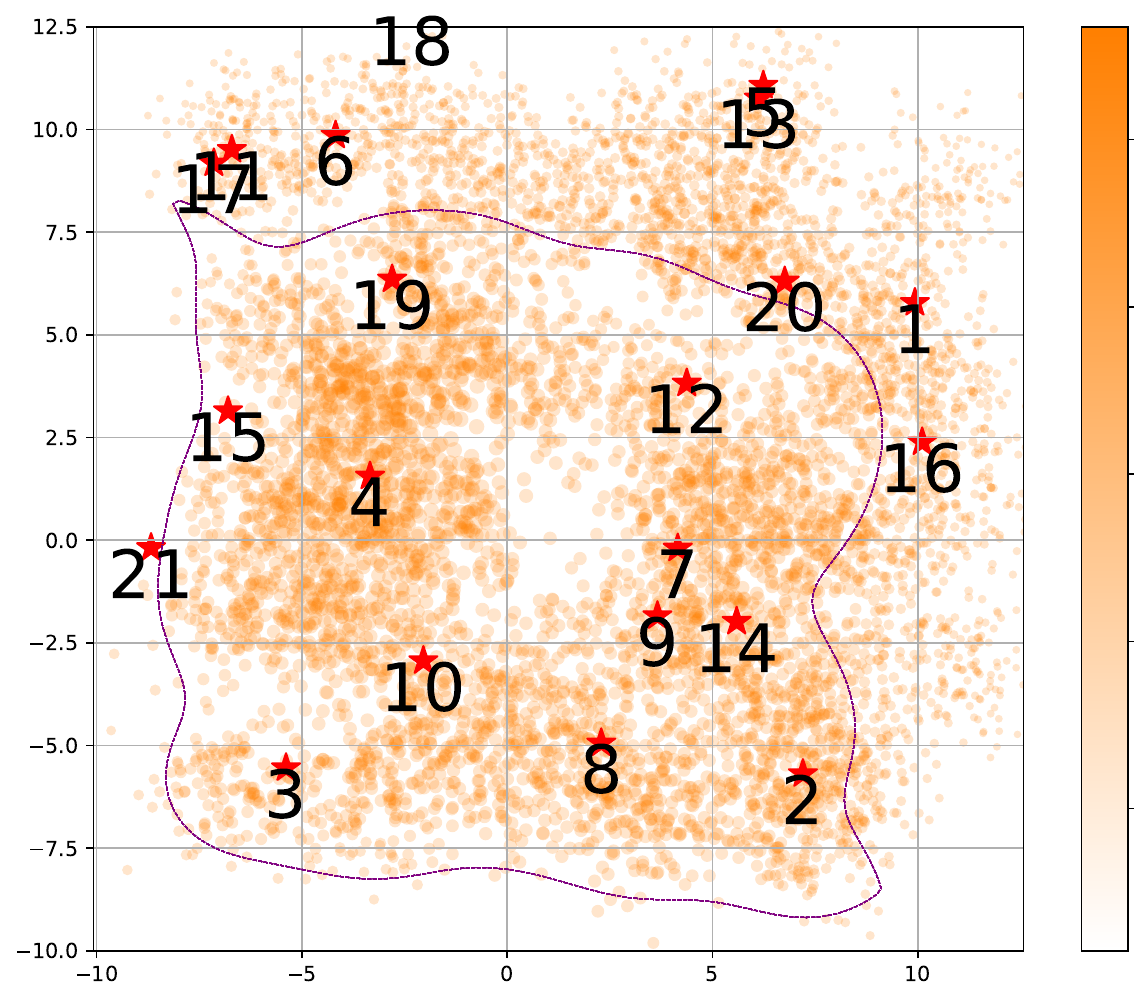}\label{fig:BAIT_acquisition}}
\vspace{-0.2cm}
  \caption{Visualizations of the toy dataset distribution, and the actual acquisition of the data by FCCM, $\mu\rho$BALD, and BAIT. The size of the data point and color from (b) to (c) is adjusted to its associated density, with deeper color representing higher density and vice versa.}
  \label{fig:data_acquisition}
\vspace{-0.2cm}
\end{figure*}

\textbf{Evaluation scheme:} The details of the data acquisition setup is summarized in Table \ref{table:acquisition_summary}, where we initialize the training set $\mathcal{S}$ with the entire labeled samples (denoted as ALL*) from group $t=0$ and start acquisition only on the sample from $t=1$, which simulates scenarios with a significant number of missing counterfactual samples. Then, a fixed step length is enforced at each acquisition step with fifty data acquisition steps. Note that for Toy dataset we do single data acquisition to show fine-granularity results, and batch-mode acquisition for the other two datasets. Each evaluation is done by the estimator trained from the last best checkpoint without completely re-train from scratch.

\subsection{Risk Evaluations}

In Figure \ref{fig:main_results}, it is observed that our proposed method is generally served as the risk lower bound in all three datasets. Its outstanding performance empirically proves the superiority of our method which considers the joint coverage on the factual and counterfactual data throughout the process of the querying. This acquisition scheme leads to a lower estimation risk that can be further explained qualitatively with the two underlying properties: i). querying from a high-density region; ii). considering the satisfactions of the positivity assumption alongside. The first property leverages the generalizability of the trained estimator to the acquired samples' neighborhoods, where a high-density neighborhood accounts for higher loss. The second property considers the pivotal overlapping assumption for treatment effect estimation, for which robust estimation toward an individual can be derived by pairing such an individual's factual or counterfactual part from the overlapping region. Thus, our proposed method outperforms the other baselines by prioritizing the data acquisition toward the overlapping region with high density, while others cannot do both essentially.

Given the above qualitative analysis, it is explainable that the other method designed for data-efficient treatment effect estimation (DTEE), e.g., $\mu\rho$BALD, underperforms our method. For example, although $\mu\rho$BALD bias the data acquisition toward the overlapping region, it does not further embed the property to query from the high-density region, and thus accounts for less risk. The same mentality applies to the other DTEE methods, e.g., MACAL, QHTE, etc. Also, it is observed that the baselines from the general AL field, e.g., LCMD and BAIT, underperform our method by not directly considering the data acquisition toward the overlapping region, which is pivotal for treatment effect estimation. However, it is interestingly observed that BAIT can outperform many DTEE baselines on CMNIST, but it underperforms on TOY and IBM. Furthermore, the current DTEE baselines cannot consistently outperform the general AL methods across various datasets.

\subsection{Acquisition Visualization}

The original distribution of the two-dimensional toy dataset is given in Figure \ref{fig:data_distribution} for reference, where the overlapping region is plotted in a dashed purple line by kernel density estimation \cite{scott2015multivariate}. From Figure \ref{fig:BALD_acquisition} to \ref{fig:BAIT_acquisition}, the density of each data point is calculated in by Line 2 in Algorithm \ref{alg:fccm}, the higher the out-degree of each node, the denser it is as visually observed in the overlapped region. Under the fine granularity of querying one sample at a time, it's observed that in Figure \ref{fig:FCCM_acquisition} FCCM consistently delivers the high priority to query from the overlapping region with high density, while $\mu\rho$BALD is witnessed to have more queries fall outside the overlapping region and less priority on query from the high-density region (query No.2 is fall at the edge of the distribution, where sparse and no overlapping is observed) as shown in Figure \ref{fig:BALD_acquisition}. Also, BAIT's acquisition spreads out the entire data on group $t=1
$ without considering the presence of the counterfactual samples and thus queries the least from the desired region as shown in Figure \ref{fig:BAIT_acquisition}. Their associated performance by 20 acquired samples (40\% consumption of budget) also seen with a significant gap as illustrated in Figure \ref{fig:toy_results}.

\subsection{Ablation and Sensitivity Study\label{section:ablation}}

\begin{table}[h!]
\centering
\caption{Ablation Study of the Counterfactual Covering Radii}\label{table:ablation_study}
\small
\begin{tabularx}{.48\textwidth}{@{}l l *{5}{>{\centering\arraybackslash}X}} % Added an extra X column for "Method"
\toprule
\multicolumn{1}{@{}l}{\multirow{2}{*}{Dataset}} & \multicolumn{1}{l}{\multirow{2}{*}{Method}} & \multicolumn{5}{c}{Consumption of the Total Budget} \\ % Added "Method" column
\cmidrule(l){3-7}
& & \multicolumn{1}{c}{1/5} & \multicolumn{1}{c}{2/5} & \multicolumn{1}{c}{3/5} & \multicolumn{1}{c}{4/5} & \multicolumn{1}{c}{5/5} \\
\toprule
\multirow{3}{*}{TOY} & FCCM- & 4.7680 & 2.2496 & 1.3372 & 1.0545 & 0.9024 \\  
& FCCM & 4.7664 & 1.8655 & 1.0978 & 0.8637 & 0.7565 \\ 
\cmidrule(l){2-7}
& Gain &+0\%&+17\%&+18\%&+18\%&+16\%\\
\midrule
\multirow{3}{*}{IBM} & FCCM- & 0.4745 & 0.4088 & 0.3797 & 0.3512 & 0.3291 \\  
& FCCM & 0.4813 & 0.4132 & 0.3845 & 0.3507 & 0.3286 \\ 
\cmidrule(l){2-7}
& Gain &-1\%&-1\%&-1\%&+0\%&+0\%\\
\midrule
\multirow{3}{*}{CMNIST} & FCCM- & 2.9250 & 2.6627 & 2.4652 & 2.3105 & 2.2073 \\  
& FCCM & 2.2207 & 1.8681 & 1.7735 & 1.7344 & 1.6790 \\ 
\cmidrule(l){2-7}
& Gain &+24\%&+30\%&+28\%&+25\%&+24\%\\
\bottomrule
\end{tabularx}
\end{table}

The ablation study is conducted to study the effect of maximizing the counterfactual coverage, e.g., $P(\mathcal{A}^{t=1}_{CF})$. Essentially, two models are compared, our proposed method FCCM, and the $\text{FCCM}^{-}$ (by setting counterfactual covering radii $\delta_{(1,0)}=0$ and $\delta_{(0,1)}=0$, and this action tailors the ProbCover \cite{yehuda2022active} to align with the context for binary-class AL). In Table \ref{table:ablation_study}, we capture five stages of the total query steps, and calculate the performance gain by $\frac{\sqrt{\epsilon_{\text{PEHE}}}_{\text{FCCM}^{-}}-\sqrt{\epsilon_{\text{PEHE}}}_{\text{FCCM}}}{\sqrt{\epsilon_{\text{PEHE}}}_{\text{FCCM}^{-}}}\times100\%$. It is noted that the performance gain on Toy and CMNIST datasets are significant, however, with neck-to-neck performance observed on the IBM dataset. This phenomenon is explainable since the treated and control distributions on IBM are heavily overlapped such that the high-density region on the treated sample ($t=1$) is in fact the high-density region on the counterfactual side ($t=0$). We provide the visualization of the IBM and CMNIST in Appendix \ref{appendix:discussion_ablation} for further discussions of the underlying rationale. Additionally, the sensitivity analysis of $\alpha$ and the covering radii $\delta_{(1,1)}$ and $\delta_{(1,0)}$ is presented in Appendix \ref{appendix:sensitivity_study}. There, we generally observe that a non-zero weight $\alpha$ is influential for dataset with less overlap, and that a larger covering radius can lead to greater risk reduction in the early stage, albeit at the cost of reduced performance in the later stage.

\begin{comment}
\begin{table}[t!]
\centering
\caption{Ablation Study of the Counterfactual Covering Radii}\label{table:ablation_study}
\small
\begin{tabularx}{.48\textwidth}{X *{5}{>{\centering\arraybackslash}X}} % Using X for the first column, left-aligned and vertical lines
\toprule
\multicolumn{1}{c}{\multirow{2}{*}{Dataset}} & \multicolumn{5}{c}{Consumption of the Total Budget}\\ % Centering header text
\cmidrule(l){2-6}
& \multicolumn{1}{c}{1/5} & \multicolumn{1}{c}{2/5} & \multicolumn{1}{c}{3/5} & \multicolumn{1}{c}{4/5} & \multicolumn{1}{c}{5/5}\\
\midrule
TOY &+0\%&+17\%&+18\%&+18\%&+16\%\\
IBM &-1\%&-1\%&-1\%&+0\%&+0\%\\
CMNIST &+24\%&+30\%&+28\%&+25\%&+24\%\\
\midrule
MEAN &+8\%&+16\%&+15\%&+14\%&+14\%\\
\bottomrule
\end{tabularx}
%\vspace{-0.6cm}
\end{table}
\end{comment}

\section{Conclusion}

We formalize the data-efficient treatment effect estimation problem under a solid theoretical framework, where the convergence of the risk upper bound is governed by the reduction of the derived factual and counterfactual covering radii. To reduce the bound, we propose a greedy radius reduction algorithm, which is $2-OPT$ under an idealized data distribution assumption. To generalize to more realistic data distributions for higher radius reduction, we further propose FCCM, which transforms the optimization objective into the factual and counterfactual coverage maximization with a $(1-\frac{1}{e})$--approximation to the full coverage constraint. Also, benchmarking with other baselines further proves the superiority of FCCM on solving the data-efficient treatment effect estimation problem.

%To handle more general scenario, we further propose the model-independent algorithm -- FCCM, which transforms the radius reduction problem into the factual and counterfactual coverage maximization problem, whose empirical performance witnesses a compromise of 1\% mean coverage for a maximally 25\% bound reduction gain on the real-world dataset. Benchmarking with other baselines further proves the superiority of FCCM on solving the treatment effect estimation with limited budget.

\section*{Limitation}

FCCM is designed to better handle the partially overlapped data for a quicker bound reduction while maintaining high coverage (Figure \ref{fig:comparision}). As such, in scenarios where the two treatment groups have non-overlapping regions in the raw feature space (e.g., biased treatment assignment), the data acquisition of FCCM will be challenged as there are no overlapping, counterfactual pairs to be identified. One possible remedy is to operate FCCM in the latent space where the inter-group distributions are aligned by methods like \cite{shalit2017estimating,zhang2020learning,wang2024optimal}, but it also offsets the model-independent advantage of FCCM.

\section*{Impact Statement}

This paper presents work whose goal is to advance the field of 
Machine Learning. There are many potential societal consequences 
of our work, none which we feel must be specifically highlighted here.

% In the unusual situation where you want a paper to appear in the
% references without citing it in the main text, use \nocite
%\nocite{langley00}

\bibliography{example_paper}

\begin{thebibliography}{59}
\providecommand{\natexlab}[1]{#1}
\providecommand{\url}[1]{\texttt{#1}}
\expandafter\ifx\csname urlstyle\endcsname\relax
  \providecommand{\doi}[1]{doi: #1}\else
  \providecommand{\doi}{doi: \begingroup \urlstyle{rm}\Url}\fi

\bibitem[Addanki et~al.(2022)Addanki, Arbour, Mai, Musco, and Rao]{addanki2022sample}
Addanki, R., Arbour, D., Mai, T., Musco, C., and Rao, A.
\newblock Sample constrained treatment effect estimation.
\newblock \emph{NeurIPS}, 35:\penalty0 5417--5430, 2022.

\bibitem[Alaa \& Van Der~Schaar(2017)Alaa and Van Der~Schaar]{alaa2017bayesian}
Alaa, A.~M. and Van Der~Schaar, M.
\newblock Bayesian inference of individualized treatment effects using multi-task gaussian processes.
\newblock \emph{NeurIPS}, 2017.

\bibitem[Ash et~al.(2021)Ash, Goel, Krishnamurthy, and Kakade]{ash2021gone}
Ash, J., Goel, S., Krishnamurthy, A., and Kakade, S.
\newblock Gone fishing: Neural active learning with fisher embeddings.
\newblock \emph{NeurIPS}, 34:\penalty0 8927--8939, 2021.

\bibitem[Ash et~al.(2019)Ash, Zhang, Krishnamurthy, Langford, and Agarwal]{ash2019deep}
Ash, J.~T., Zhang, C., Krishnamurthy, A., Langford, J., and Agarwal, A.
\newblock Deep batch active learning by diverse, uncertain gradient lower bounds.
\newblock In \emph{International Conference on Learning Representations}, 2019.

\bibitem[Chen et~al.(2024)Chen, Cai, Yang, Qiao, Yan, Li, and Hao]{chen2024doubly}
Chen, W., Cai, R., Yang, Z., Qiao, J., Yan, Y., Li, Z., and Hao, Z.
\newblock Doubly robust causal effect estimation under networked interference via targeted learning.
\newblock \emph{arXiv preprint arXiv:2405.03342}, 2024.

\bibitem[Connolly et~al.(2023)Connolly, Moore, Schwedes, Adam, Willis, Feige, and Frye]{connolly2023task}
Connolly, B., Moore, K., Schwedes, T., Adam, A., Willis, G., Feige, I., and Frye, C.
\newblock Task-specific experimental design for treatment effect estimation.
\newblock In \emph{International Conference on Machine Learning}, pp.\  6384--6401. PMLR, 2023.

\bibitem[Cook et~al.(1998)Cook, Cunningham, Pulleyblank, and Schrijver]{cook1998combi}
Cook, W.~J., Cunningham, W.~H., Pulleyblank, W.~R., and Schrijver, A.
\newblock \emph{Combinatorial Optimization}.
\newblock Springer, 1998.

\bibitem[Deng et~al.(2011)Deng, Pineau, and Murphy]{deng2011active}
Deng, K., Pineau, J., and Murphy, S.
\newblock Active learning for personalizing treatment.
\newblock In \emph{2011 IEEE Symposium on Adaptive Dynamic Programming and Reinforcement Learning (ADPRL)}, pp.\  32--39. IEEE, 2011.

\bibitem[Dinitz(2019)]{dinitz2019lecture4}
Dinitz, M.
\newblock Lecture 4: Approximation algorithms - vertex cover and set cover.
\newblock \url{https://www.cs.jhu.edu/~mdinitz/classes/ApproxAlgorithms/Spring2019/Lectures/lecture4.pdf}, 2019.

\bibitem[Fujii \& Kashima(2016)Fujii and Kashima]{fujii2016budgeted}
Fujii, K. and Kashima, H.
\newblock Budgeted stream-based active learning via adaptive submodular maximization.
\newblock \emph{NeurIPS}, 29, 2016.

\bibitem[Gal et~al.(2017)Gal, Islam, and Ghahramani]{gal2017deep}
Gal, Y., Islam, R., and Ghahramani, Z.
\newblock Deep bayesian active learning with image data.
\newblock In \emph{International conference on machine learning}, pp.\  1183--1192. PMLR, 2017.

\bibitem[Ghadiri et~al.(2024)Ghadiri, Arbour, Mai, Musco, and Rao]{ghadiri2024finite}
Ghadiri, M., Arbour, D., Mai, T., Musco, C., and Rao, A.~B.
\newblock Finite population regression adjustment and non-asymptotic guarantees for treatment effect estimation.
\newblock \emph{NeurIPS}, 36, 2024.

\bibitem[Hill(2011)]{hill2011bayesian}
Hill, J.~L.
\newblock Bayesian nonparametric modeling for causal inference.
\newblock \emph{Journal of Computational and Graphical Statistics}, 20\penalty0 (1):\penalty0 217--240, 2011.

\bibitem[Holzm{\"u}ller et~al.(2023)Holzm{\"u}ller, Zaverkin, K{\"a}stner, and Steinwart]{holzmuller2023framework}
Holzm{\"u}ller, D., Zaverkin, V., K{\"a}stner, J., and Steinwart, I.
\newblock A framework and benchmark for deep batch active learning for regression.
\newblock \emph{Journal of Machine Learning Research}, 24\penalty0 (164):\penalty0 1--81, 2023.

\bibitem[Imbens \& Rubin(2015)Imbens and Rubin]{imbens2015causal}
Imbens, G.~W. and Rubin, D.~B.
\newblock \emph{Causal inference in statistics, social, and biomedical sciences}.
\newblock Cambridge University Press, 2015.

\bibitem[Jesson et~al.(2020)Jesson, Mindermann, Shalit, and Gal]{jesson2020identifying}
Jesson, A., Mindermann, S., Shalit, U., and Gal, Y.
\newblock Identifying causal-effect inference failure with uncertainty-aware models.
\newblock \emph{NeurIPS}, 33:\penalty0 11637--11649, 2020.

\bibitem[Jesson et~al.(2021{\natexlab{a}})Jesson, Mindermann, Gal, and Shalit]{jesson2021quantifying}
Jesson, A., Mindermann, S., Gal, Y., and Shalit, U.
\newblock Quantifying ignorance in individual-level causal-effect estimates under hidden confounding.
\newblock In \emph{International Conference on Machine Learning}, pp.\  4829--4838. PMLR, 2021{\natexlab{a}}.

\bibitem[Jesson et~al.(2021{\natexlab{b}})Jesson, Tigas, van Amersfoort, Kirsch, Shalit, and Gal]{jesson2021causal}
Jesson, A., Tigas, P., van Amersfoort, J., Kirsch, A., Shalit, U., and Gal, Y.
\newblock Causal-bald: Deep bayesian active learning of outcomes to infer treatment-effects from observational data.
\newblock \emph{NeurIPS}, 34:\penalty0 30465--30478, 2021{\natexlab{b}}.

\bibitem[Kallus(2020)]{kallus2020deepmatch}
Kallus, N.
\newblock Deepmatch: Balancing deep covariate representations for causal inference using adversarial training.
\newblock In \emph{International Conference on Machine Learning}, pp.\  5067--5077. PMLR, 2020.

\bibitem[Kirsch et~al.(2021)Kirsch, Farquhar, Atighehchian, Jesson, Branchaud-Charron, and Gal]{kirsch2021stochastic}
Kirsch, A., Farquhar, S., Atighehchian, P., Jesson, A., Branchaud-Charron, F., and Gal, Y.
\newblock Stochastic batch acquisition: A simple baseline for deep active learning.
\newblock \emph{arXiv preprint arXiv:2106.12059}, 2021.

\bibitem[Kohavi \& Longbotham(2015)Kohavi and Longbotham]{kohavi2015online}
Kohavi, R. and Longbotham, R.
\newblock Online controlled experiments and a/b tests.
\newblock \emph{Encyclopedia of machine learning and data mining}, pp.\  1--11, 2015.

\bibitem[LeCun(1998)]{lecun1998mnist}
LeCun, Y.
\newblock The mnist database of handwritten digits.
\newblock \emph{http://yann. lecun. com/exdb/mnist/}, 1998.

\bibitem[Lin et~al.(2023)Lin, Zhang, Lu, Bao, Takeuchi, and Kashima]{lin2023estimating}
Lin, X., Zhang, G., Lu, X., Bao, H., Takeuchi, K., and Kashima, H.
\newblock Estimating treatment effects under heterogeneous interference.
\newblock In \emph{Joint European Conference on Machine Learning and Knowledge Discovery in Databases}, pp.\  576--592. Springer, 2023.

\bibitem[Lin et~al.(2024)Lin, Zhang, Lu, and Kashima]{lin2024treatment}
Lin, X., Zhang, G., Lu, X., and Kashima, H.
\newblock Treatment effect estimation under unknown interference.
\newblock In \emph{Pacific-Asia Conference on Knowledge Discovery and Data Mining}, pp.\  28--42. Springer, 2024.

\bibitem[Lin et~al.(2025)Lin, Bao, Cui, Takeuchi, and Kashima]{lin2025scalable}
Lin, X., Bao, H., Cui, Y., Takeuchi, K., and Kashima, H.
\newblock Scalable individual treatment effect estimator for large graphs.
\newblock \emph{Machine Learning}, 114\penalty0 (1):\penalty0 1--19, 2025.

\bibitem[Louizos et~al.(2017)Louizos, Shalit, Mooij, Sontag, Zemel, and Welling]{louizos2017causal}
Louizos, C., Shalit, U., Mooij, J.~M., Sontag, D., Zemel, R., and Welling, M.
\newblock Causal effect inference with deep latent-variable models.
\newblock \emph{NeurIPS}, 30, 2017.

\bibitem[Ma et~al.(2022)Ma, Wan, Yang, Li, Hecht, and Teevan]{ma2022learning}
Ma, J., Wan, M., Yang, L., Li, J., Hecht, B., and Teevan, J.
\newblock Learning causal effects on hypergraphs.
\newblock In \emph{Proceedings of the 28th ACM SIGKDD Conference on Knowledge Discovery and Data Mining}, pp.\  1202--1212, 2022.

\bibitem[Ma \& Tresp(2021)Ma and Tresp]{ma2021causal}
Ma, Y. and Tresp, V.
\newblock Causal inference under networked interference and intervention policy enhancement.
\newblock In \emph{International Conference on Artificial Intelligence and Statistics}, pp.\  3700--3708. PMLR, 2021.

\bibitem[MacKay(2020)]{mackay2020government}
MacKay, D.
\newblock Government policy experiments and the ethics of randomization.
\newblock \emph{Philosophy \& Public Affairs}, 48\penalty0 (4):\penalty0 319--352, 2020.

\bibitem[Pearl(2009)]{pearl2009causality}
Pearl, J.
\newblock \emph{Causality}.
\newblock Cambridge university press, 2009.

\bibitem[Pilat et~al.(2015)Pilat, Frech, Wagner, Kr{\"u}ger, Hillebrecht, Pons-K{\"u}hnemann, Scheibelhut, B{\"o}deker, and Mooren]{pilat2015exploring}
Pilat, C., Frech, T., Wagner, A., Kr{\"u}ger, K., Hillebrecht, A., Pons-K{\"u}hnemann, J., Scheibelhut, C., B{\"o}deker, R.-H., and Mooren, F.-C.
\newblock Exploring effects of a natural combination medicine on exercise-induced inflammatory immune response: A double-blind rct.
\newblock \emph{Scandinavian Journal of Medicine \& Science in Sports}, 25\penalty0 (4):\penalty0 534--542, 2015.

\bibitem[Pinsler et~al.(2019)Pinsler, Gordon, Nalisnick, and Hern{\'a}ndez-Lobato]{pinsler2019bayesian}
Pinsler, R., Gordon, J., Nalisnick, E., and Hern{\'a}ndez-Lobato, J.~M.
\newblock Bayesian batch active learning as sparse subset approximation.
\newblock \emph{NeurIPS}, 32, 2019.

\bibitem[Qin et~al.(2021)Qin, Wang, and Zhou]{qin2021budgeted}
Qin, T., Wang, T.-Z., and Zhou, Z.-H.
\newblock Budgeted heterogeneous treatment effect estimation.
\newblock In \emph{International Conference on Machine Learning}, pp.\  8693--8702. PMLR, 2021.

\bibitem[Rakesh et~al.(2018)Rakesh, Guo, Moraffah, Agarwal, and Liu]{rakesh2018linked}
Rakesh, V., Guo, R., Moraffah, R., Agarwal, N., and Liu, H.
\newblock Linked causal variational autoencoder for inferring paired spillover effects.
\newblock In \emph{Proceedings of the 27th ACM International Conference on Information and Knowledge Management}, pp.\  1679--1682, 2018.

\bibitem[Ren et~al.(2021)Ren, Xiao, Chang, Huang, Li, Gupta, Chen, and Wang]{ren2021survey}
Ren, P., Xiao, Y., Chang, X., Huang, P.-Y., Li, Z., Gupta, B.~B., Chen, X., and Wang, X.
\newblock A survey of deep active learning.
\newblock \emph{ACM computing surveys (CSUR)}, 54\penalty0 (9):\penalty0 1--40, 2021.

\bibitem[Robins et~al.(1994)Robins, Rotnitzky, and Zhao]{robins1994estimation}
Robins, J.~M., Rotnitzky, A., and Zhao, L.~P.
\newblock Estimation of regression coefficients when some regressors are not always observed.
\newblock \emph{Journal of the American statistical Association}, 89\penalty0 (427):\penalty0 846--866, 1994.

\bibitem[Rosenbaum \& Rubin(1983)Rosenbaum and Rubin]{rosenbaum1983central}
Rosenbaum, P.~R. and Rubin, D.~B.
\newblock The central role of the propensity score in observational studies for causal effects.
\newblock \emph{Biometrika}, 70\penalty0 (1):\penalty0 41--55, 1983.

\bibitem[Scott(2015)]{scott2015multivariate}
Scott, D.~W.
\newblock \emph{Multivariate density estimation: theory, practice, and visualization}.
\newblock John Wiley \& Sons, 2015.

\bibitem[Sener \& Savarese(2018)Sener and Savarese]{sener2018active}
Sener, O. and Savarese, S.
\newblock Active learning for convolutional neural networks: A core-set approach.
\newblock In \emph{International Conference on Learning Representations}, 2018.

\bibitem[Settles(2009)]{settles2009active}
Settles, B.
\newblock Active learning literature survey.
\newblock \emph{Computer Sciences Technical Report}, 2009.

\bibitem[Shalit et~al.(2017)Shalit, Johansson, and Sontag]{shalit2017estimating}
Shalit, U., Johansson, F.~D., and Sontag, D.
\newblock Estimating individual treatment effect: generalization bounds and algorithms.
\newblock In \emph{International Conference on Machine Learning}, pp.\  3076--3085. PMLR, 2017.

\bibitem[Shi et~al.(2019)Shi, Blei, and Veitch]{shi2019adapting}
Shi, C., Blei, D., and Veitch, V.
\newblock Adapting neural networks for the estimation of treatment effects.
\newblock \emph{NeurIPS}, 32, 2019.

\bibitem[Shimoni et~al.(2018)Shimoni, Yanover, Karavani, and Goldschmnidt]{shimoni2018benchmarking}
Shimoni, Y., Yanover, C., Karavani, E., and Goldschmnidt, Y.
\newblock Benchmarking framework for performance-evaluation of causal inference analysis.
\newblock \emph{arXiv preprint arXiv:1802.05046}, 2018.

\bibitem[Smith(1918)]{smith1918standard}
Smith, K.
\newblock On the standard deviations of adjusted and interpolated values of an observed polynomial function and its constants and the guidance they give towards a proper choice of the distribution of observations.
\newblock \emph{Biometrika}, 12\penalty0 (1/2):\penalty0 1--85, 1918.

\bibitem[Sundin et~al.(2019)Sundin, Schulam, Siivola, Vehtari, Saria, and Kaski]{sundin2019active}
Sundin, I., Schulam, P., Siivola, E., Vehtari, A., Saria, S., and Kaski, S.
\newblock Active learning for decision-making from imbalanced observational data.
\newblock In \emph{International conference on machine learning}, pp.\  6046--6055. PMLR, 2019.

\bibitem[Tsang et~al.(2005)Tsang, Kwok, Cheung, and Cristianini]{tsang2005core}
Tsang, I.~W., Kwok, J.~T., Cheung, P.-M., and Cristianini, N.
\newblock Core vector machines: Fast svm training on very large data sets.
\newblock \emph{Journal of Machine Learning Research}, 6\penalty0 (4), 2005.

\bibitem[Van~Amersfoort et~al.(2021)Van~Amersfoort, Smith, Jesson, Key, and Gal]{van2021feature}
Van~Amersfoort, J., Smith, L., Jesson, A., Key, O., and Gal, Y.
\newblock On feature collapse and deep kernel learning for single forward pass uncertainty.
\newblock \emph{arXiv preprint arXiv:2102.11409}, 2021.

\bibitem[Vapnik(1999)]{vapnik1999overview}
Vapnik, V.~N.
\newblock An overview of statistical learning theory.
\newblock \emph{IEEE transactions on neural networks}, 10\penalty0 (5):\penalty0 988--999, 1999.

\bibitem[Wang et~al.(2024)Wang, Fan, Chen, Li, Liu, Liu, Dai, Wang, Dong, and Tang]{wang2024optimal}
Wang, H., Fan, J., Chen, Z., Li, H., Liu, W., Liu, T., Dai, Q., Wang, Y., Dong, Z., and Tang, R.
\newblock Optimal transport for treatment effect estimation.
\newblock \emph{NeurIPS}, 36, 2024.

\bibitem[Wang et~al.(2015)Wang, Hu, Yuan, and Lu]{wang2015active}
Wang, L., Hu, X., Yuan, B., and Lu, J.
\newblock Active learning via query synthesis and nearest neighbour search.
\newblock \emph{Neurocomputing}, 147:\penalty0 426--434, 2015.

\bibitem[Wen et~al.(2025)Wen, Chen, Ye, Chai, Sadiq, and Yin]{wen2024progressive}
Wen, H., Chen, T., Ye, G., Chai, L.~K., Sadiq, S., and Yin, H.
\newblock Progressive generalization risk reduction for data-efficient causal effect estimation.
\newblock In \emph{Proceedings of 31st ACM SIGKDD Conference on Knowledge Discovery and Data Mining}, volume V.1, pp.\  1575--1586, 2025.

\bibitem[Wilson et~al.(2016)Wilson, Hu, Salakhutdinov, and Xing]{wilson2016deep}
Wilson, A.~G., Hu, Z., Salakhutdinov, R., and Xing, E.~P.
\newblock Deep kernel learning.
\newblock In \emph{Artificial intelligence and statistics}, pp.\  370--378. PMLR, 2016.

\bibitem[Wolf(2011)]{wolf2011facility}
Wolf, G.~W.
\newblock Facility location: concepts, models, algorithms and case studies, 2011.

\bibitem[Wu(2018)]{wu2018pool}
Wu, D.
\newblock Pool-based sequential active learning for regression.
\newblock \emph{IEEE transactions on neural networks and learning systems}, 30\penalty0 (5):\penalty0 1348--1359, 2018.

\bibitem[Yao et~al.(2018)Yao, Li, Li, Huai, Gao, and Zhang]{yao2018representation}
Yao, L., Li, S., Li, Y., Huai, M., Gao, J., and Zhang, A.
\newblock Representation learning for treatment effect estimation from observational data.
\newblock \emph{NeurIPS}, 31, 2018.

\bibitem[Yehuda et~al.(2022)Yehuda, Dekel, Hacohen, and Weinshall]{yehuda2022active}
Yehuda, O., Dekel, A., Hacohen, G., and Weinshall, D.
\newblock Active learning through a covering lens.
\newblock \emph{NeurIPS}, 35:\penalty0 22354--22367, 2022.

\bibitem[Yoon et~al.(2018)Yoon, Jordon, and Van Der~Schaar]{yoon2018ganite}
Yoon, J., Jordon, J., and Van Der~Schaar, M.
\newblock Ganite: Estimation of individualized treatment effects using generative adversarial nets.
\newblock In \emph{International Conference on Learning Representations}, 2018.

\bibitem[Zhan et~al.(2022)Zhan, Wang, Huang, Xiong, Dou, and Chan]{zhan2022comparative}
Zhan, X., Wang, Q., Huang, K.-h., Xiong, H., Dou, D., and Chan, A.~B.
\newblock A comparative survey of deep active learning.
\newblock \emph{arXiv preprint arXiv:2203.13450}, 2022.

\bibitem[Zhang et~al.(2020)Zhang, Bellot, and Schaar]{zhang2020learning}
Zhang, Y., Bellot, A., and Schaar, M.
\newblock Learning overlapping representations for the estimation of individualized treatment effects.
\newblock In \emph{International Conference on Artificial Intelligence and Statistics}, pp.\  1005--1014. PMLR, 2020.

\end{thebibliography}
\bibliographystyle{icml2025}

%%%%%%%%%%%%%%%%%%%%%%%%%%%%%%%%%%%%%%%%%%%%%%%%%%%%%%%%%%%%%%%%%%%%%%%%%%%%%%%
%%%%%%%%%%%%%%%%%%%%%%%%%%%%%%%%%%%%%%%%%%%%%%%%%%%%%%%%%%%%%%%%%%%%%%%%%%%%%%%
% APPENDIX
%%%%%%%%%%%%%%%%%%%%%%%%%%%%%%%%%%%%%%%%%%%%%%%%%%%%%%%%%%%%%%%%%%%%%%%%%%%%%%%
%%%%%%%%%%%%%%%%%%%%%%%%%%%%%%%%%%%%%%%%%%%%%%%%%%%%%%%%%%%%%%%%%%%%%%%%%%%%%%%
\newpage
\appendix
\onecolumn

\section{Theory}

\subsection{Proof of Theorem \ref{theorem:overall}\label{appendix:theorem_1}}

\textbf{Assumption \ref{assumption:Lipschitz}} (Lipschitz Continuity)\textbf{.} \emph{Assume that the conditional probability density function $p^{t}(y|\mathbf{x})$ is $\lambda_{t}$-Lipschitz, the squared loss function $l$ is $\lambda_{l}$-Lipschitz and $l$ is further bounded by $L_{l}$.
}

\textbf{Assumption \ref{assumption:kappa}} (Constant $\kappa$)\textbf{.} \emph{
    Let $\mathcal{H} = \{h | h: \mathcal{X}\rightarrow \mathbb{R}\}$ be a family of functions and $f:\mathcal{X}\times\mathcal{T}\rightarrow\mathcal{Y}$ be the hypothesis. Assume that there exists a constant $\kappa>0$, such that $h_{f}(\mathbf{x},t):=\frac{1}{\kappa}l_{f}(\mathbf{x},t)\in \mathcal{H}$.
}

\textbf{Theorem \ref{theorem:overall}.} \emph{Let $\mathbf{x}$ be sampled i.i.d. $n$ times from domain $\mathcal{X}$. Under Assumption \ref{assumption:Lipschitz} and Assumption \ref{assumption:kappa}, with probability at least $1-\gamma$, where $\gamma\in(0,1)$,  the subset generalization gap $\Delta$ is upper-bounded as:
    \begin{equation}
    \begin{split}
        &\left|\frac{1}{n}\sum^{n}_{i=1}\xi(\mathbf{x}_{i};f_{\mathcal{S}})-\frac{1}{|\mathcal{S}|}\sum^{|\mathcal{S}|}_{j=1}l(\mathbf{x}_{i}, y_{i}, t_{i};f_{\mathcal{S}})\right|\\
        \leq&\sum_{t\in\{0,1\}}\kappa_{t}\,\left(\delta_{(t,t)}
        +\delta_{(t,1-t)}\right) + 2\,\kappa_{\mathcal{H}}+\sqrt{\frac{L^{2}_{l}\log \frac{1}{\gamma}}{2n}},
    \end{split}
    \end{equation} where the constants $\kappa_{t}=2\,(\lambda_{l}+\frac{1}{3}\lambda_{t}L^{\frac{3}{2}}_{l})$, and $\kappa_{\mathcal{H}}=\kappa\cdot\text{IPM}_{\mathcal{H}}(p^{t=1}(\mathbf{x}), p^{t=0}(\mathbf{x}))$ with $\text{IPM}_{\mathcal{H}}(\cdot,\cdot)$ denotes the integral probability metric induced by $\mathcal{H}$ and $p^{t}(\mathbf{x})$ denotes the density distribution of treatment group $t$.
}

\begin{proof}[Proof of Theorem \ref{theorem:overall}]
%Before we apply Hoeffding's inequality, we start the proof by bounding the expected loss of the interested term $\frac{1}{n}\sum^{n}_{i=1}\xi(\mathbf{x}_{i};f_{\mathcal{S}})$, i.e.,  $\mathbb{E}\left[\frac{1}{n}\sum^{n}_{i=1}\xi(\mathbf{x}_{i};f_{\mathcal{S}})\right]=\mathbb{E}_{x\sim \mathcal{X}^{i.i.d.}}\left[\xi(x;f_{\mathcal{S}})\right]$. In the following, we have the bounded loss $l(x;f_{\mathcal{S}}):=(\tau(\mathbf{x})-\hat{\tau}(\mathbf{x}))^{2}\leq L_{l}$, thus
\begin{subequations}

The proof is done in three main steps: Firstly, we bound the expected value of the interested term $\frac{1}{n}\sum^{n}_{i=1}\xi(\mathbf{x}_{i};f_{\mathcal{S}})$ with factual and counterfactual loss, namely, $\epsilon_{F}$ and $\epsilon_{CF}$ over the domain of the pool set. Secondly, we constrain the $\epsilon_{F}$ with the factual covering radius $\delta_{(t,t)}$ and the $\epsilon_{CF}$ with the counterfactual covering radius $\delta_{(t,1-t)}$. Lastly, we conclude the probabilistic bound by Hoeffding's inequality.

\begin{align}
        \mathbb{E}\left[\frac{1}{n}\sum^{n}_{i=1}\xi(\mathbf{x}_{i};f_{\mathcal{S}})\right]&\leq\,2\,(\epsilon_{F}+\epsilon_{CF})\label{eq:decomposition}\\
        &\leq\sum_{t\in\{0,1\}}2\,\mu_{t}\left(\epsilon_{\mathcal{S}_{t}}+\delta_{(t,t)}
        (\lambda_{l}+\frac{1}{3}\lambda_{t}L^{\frac{3}{2}}_{l})\right)+\sum_{t\in\{0,1\}}2\,\mu_{1-t}\left(\epsilon_{\mathcal{S}_{t}}+\delta_{(t,1-t)}(\lambda_{l}+\frac{1}{3}\lambda_{t}L^{\frac{3}{2}}_{l})\right) + 2\,\kappa_{\mathcal{H}}\label{eq:F_CF_decompisition}\\
        &\leq\sum_{t\in\{0,1\}}2\,(\lambda_{l}+\frac{1}{3}\lambda_{t}L^{\frac{3}{2}}_{l})\,\left(\delta_{(t,t)}
        +\delta_{(t,1-t)}\right) + 2\,\kappa_{\mathcal{H}}.\label{eq:general_form}
        %&\leq2((\delta^{(t=1)\rightarrow(t=1)}+\delta^{(t=1)\rightarrow(t=0)})(\lambda_{l}+\frac{1}{3}\lambda_{t=1}L^{\frac{3}{2}}_{l})+\\
        %&(\delta^{(t=0)\rightarrow(t=0)}+\delta^{(t=0)\rightarrow(t=1)})(\lambda_{l}+\frac{1}{3}\lambda_{t=0}L^{\frac{3}{2}}_{l})+B_{\Phi}\cdot\text{IPM}_{G}(p^{t=1}(r), p^{t=0}(r)))
    \end{align}
    The inequality (\ref{eq:decomposition}) is by Lemma \ref{lemma:expected_risk}, the inequality (\ref{eq:F_CF_decompisition}) is by Lemma \ref{lemma:factual_error} and \ref{lemma:counterfactual_error}, the equality in (\ref{eq:general_form}) is by the zero training loss assumption and the fact that $\mu_{t}\leq1,\forall t\in\{0,1\}$. With Hoeffding's inequality, we have:
   \begin{align}
        \mathbb{P}\left(\frac{1}{n}\sum^{n}_{i=1}\xi(\mathbf{x}_{i};f_{\mathcal{S}})-\mathbb{E}\left[\frac{1}{n}\sum^{n}_{i=1}\xi(\mathbf{x}_{i};f_{\mathcal{S}})\right]\geq\epsilon
        \right)=\exp{\left(-\frac{2n\epsilon^{2}}{L^{2}_{l}}\right)}.
   \end{align}

   By setting $\gamma=\exp{\left(-\frac{2n\epsilon^{2}}{L^{2}_{l}}\right)}$, we solve for the bounding gap $\epsilon=\sqrt{\frac{L^{2}_{l}\log\frac{1}{\gamma}}{2n}}$. Thus, we can derive that with probability of at least $1-\gamma$, the following inequality holds:
    \begin{equation}
        \frac{1}{n}\sum^{n}_{i=1}\xi(\mathbf{x}_{i};f_{\mathcal{S}})\leq\mathbb{E}\left[\frac{1}{n}\sum^{n}_{i=1}\xi(\mathbf{x}_{i};f_{\mathcal{S}})\right]+\sqrt{\frac{L^{2}_{l}\log\frac{1}{\gamma}}{2n}}\label{eq:hoeffdings_inequality}
    \end{equation}

   To conclude, we have:
    \begin{align}
        &\left|\frac{1}{n}\sum^{n}_{i=1}\xi(\mathbf{x}_{i};f_{\mathcal{S}})-\frac{1}{|\mathcal{S}|}\sum^{|\mathcal{S}|}_{j=1}l(\mathbf{x}_{i}, y_{i}, t_{i};f_{\mathcal{S}})\right|\\
        =&\frac{1}{n}\sum^{n}_{i=1}\xi(\mathbf{x}_{i};f_{\mathcal{S}})\label{eq:zero_loss}\\
        \leq&\mathbb{E}\left[\frac{1}{n}\sum^{n}_{i=1}\xi(\mathbf{x}_{i};f_{\mathcal{S}})\right]+\sqrt{\frac{L^{2}_{l}\log\frac{1}{\gamma}}{2n}}\label{eq:hoeffding's}\\
        \leq&\sum_{t\in\{0,1\}}2\,(\lambda_{l}+\frac{1}{3}\lambda_{t}L^{\frac{3}{2}}_{l})\,\left(\delta_{(t,t)}
        +\delta_{(t,1-t)}\right) + 2\,\kappa_{\mathcal{H}}+\sqrt{\frac{L^{2}_{l}\log \frac{1}{\gamma}}{2n}}.\label{eq:conclusion}
        %\leq\mathcal{O}(\delta^{(t=1)\rightarrow(t=1)})+\mathcal{O}(\delta^{(t=0)\rightarrow(t=1)})+\mathcal{O}(\delta^{(t=0)\rightarrow(t=0)})+\mathcal{O}(\delta^{(t=1)\rightarrow(t=0)})+\mathcal{O}(\sqrt{\frac{1}{n}})
    \end{align}
\end{subequations}

The equality (\ref{eq:zero_loss}) is by the zero training loss assumption, the inequality (\ref{eq:hoeffding's}) is by incorporating (\ref{eq:hoeffdings_inequality}), the final inequality (\ref{eq:conclusion}) is concluded by incorporating (\ref{eq:general_form}).

\textbf{Discussion on the bound}: Given the i.i.d. sampled pool set from the domain $\mathcal{X}$, it is noted that in the final probabilistic bound, if the sampled two distributions were identical and we selected the entire pool set for training, simply by definition, we have both the factual and counterfactual covering radius be completely zero, and also the distributional discrepancy term $\kappa_{\mathcal{H}}$ counted by IPM be completed zero, leaving the tightness of the risk upper bound solely depends on the size of the pool set $n$ and for the given probability threshold $\gamma$.

\end{proof} 

\subsection{Proof of Lemma \ref{lemma:expected_risk}\label{appendix:lemma_1}}

\begin{definition}
    Given the loss metric $l$, the expected factual loss $\epsilon_{F}$ and counterfactual loss $\epsilon_{CF}$ are defined in a manner consistent with \cite{shalit2017estimating} as follows: 
    \begin{equation}
    \begin{split}
        &\epsilon_{F} = \int_{\mathcal{X}\times\mathcal{T}}l(\mathbf{x},t)p(\mathbf{x},t)d\mathbf{x}dt,\\
        &\epsilon_{CF} = \int_{\mathcal{X}\times\mathcal{T}}l(\mathbf{x},t)p(\mathbf{x},1-t)d\mathbf{x}dt.
    \end{split}
    \end{equation}
\end{definition}

\begin{lemma}\label{lemma:expected_risk}
    Let $\mathbf{x}$ to be sampled i.i.d. $n$ times from the domain $\mathcal{X}$. With the two-headed trained model $f_{\mathcal{S}}=\{\hat{f}^{t=1},\hat{f}^{t=0}\}$ on the selected subset $\mathcal{S}$, where $\hat{f}^{t=1},\hat{f}^{t=0}:\mathcal{X}\rightarrow\mathcal{Y}$ are the estimators for the treatment effect $y^{t}$ respectively. The expected value of $\frac{1}{n}\sum^{n}_{i=1}\xi(\mathbf{x}_{i};f_{\mathcal{S}})$ is upper-bounded as follows:
    \begin{equation}
        \mathbb{E}\left[\frac{1}{n}\sum^{n}_{i=1}\xi(\mathbf{x}_{i};f_{\mathcal{S}})\right]\leq 2(\epsilon_{F}+\epsilon_{CF}).
    \end{equation}
\end{lemma}

\begin{proof}[Proof of Lemma \ref{lemma:expected_risk}]
    \begin{subequations}
        \begin{align}
            \mathbb{E}\left[\frac{1}{n}\sum^{n}_{i=1}\xi(\mathbf{x}_{i};f_{\mathcal{S}})\right]&=\frac{1}{n}\sum^{n}_{i=1}\mathbb{E}\left[\xi(\mathbf{x}_{i};f_{\mathcal{S}})\right]\\
            &=\mathbb{E}_{\mathbf{x}\sim \mathcal{X}}\left[\xi(\mathbf{x};f_{\mathcal{S}})\right]\\
            &=\int_{\mathcal{X}}(\tau(\mathbf{x})-\hat{\tau}(\mathbf{x}))^{2}p(\mathbf{x})d\mathbf{x}\\
        &=\int_{\mathcal{X}}((y^{t=1}-y^{t=0})-(\hat{f}^{t=1}(\mathbf{x})-\hat{f}^{t=0}(\mathbf{x})))^{2}p(\mathbf{x})d\mathbf{x}\\
        &=\int_{\mathcal{X}}(\underbrace{(y^{t=1}-\hat{f}^{t=1}(\mathbf{x}))-(y^{t=0}-\hat{f}^{t=0}(\mathbf{x}))}_{\text{Swap $y^{t=0}$ and $\hat{f}^{t=1}(\mathbf{x})$}})^{2}p(\mathbf{x})d\mathbf{x}\\
        &\leq2\int_{\mathcal{X}}((y^{t=1}-\hat{f}^{t=1}(\mathbf{x}))^{2}+(y^{t=0}-\hat{f}^{t=0}(\mathbf{x}))^{2})p(\mathbf{x})d\mathbf{x}\\
        &=\underbrace{2\int_{\mathcal{X}}(y^{t=1}-\hat{f}^{t=1}(\mathbf{x}))^{2}p(x,t=1)d\mathbf{x}+2\int_{\mathcal{X}}(y^{t=1}-\hat{f}^{t=1}(\mathbf{x}))^{2}p(x,t=0)d\mathbf{x}}_{\text{Apply $p(\mathbf{x})=\int p(x,t)dt=p(x,t=1)+p(x,t=0)$}}+\\
        &2\int_{\mathcal{X}}(y^{t=0}-\hat{f}^{t=0}(\mathbf{x}))^{2}p(x,t=0)d\mathbf{x}+2\int_{\mathcal{X}}(y^{t=0}-\hat{f}^{t=0}(\mathbf{x}))^{2}p(x,t=1)d\mathbf{x}\\
        &=\underbrace{2\int_{\mathcal{X}}(y^{t=1}-\hat{f}^{t=1}(\mathbf{x}))^{2}p(x,t=1)d\mathbf{x}+2\int_{\mathcal{X}}(y^{t=0}-\hat{f}^{t=0}(\mathbf{x}))^{2}p(x,t=0)d\mathbf{x}}_{\text{Re-arrange and this term equals 2$\epsilon_{F}$ by definition}}+\\
        &\underbrace{2\int_{\mathcal{X}}(y^{t=1}-\hat{f}^{t=1}(\mathbf{x}))^{2}p(x,t=0)d\mathbf{x}+2\int_{\mathcal{X}}(y^{t=0}-\hat{f}^{t=0}(\mathbf{x}))^{2}p(x,t=1)d\mathbf{x}}_{\text{Re-arrange and this term equals 2$\epsilon_{CF}$ by definition}}\\
        &=2(\epsilon_{F}+\epsilon_{CF}).
        \end{align}
    \end{subequations}
\end{proof}

\textbf{Discussion}: Lemma \ref{lemma:expected_risk} indicates an interesting decomposition of the expected loss into factual and counterfactual errors, which is a prelude to our final probabilistic bound under AL paradigm. It is worth mentioning that a similar intermediate conclusion appears to align with observations in \cite{shalit2017estimating}.

\subsection{Proof of Lemma \ref{lemma:factual_error}\label{appendix:lemma_2}}

%In addition to that, in the following we unveil the distinctive \textit{counterfactual covering radius} which originates from the unique nature of the counterfactual prediction.

\begin{lemma}\label{lemma:factual_error}
    Denote the expected loss on subset $S_{t}$ as $\epsilon_{\mathcal{S}_{t}}$, the factual covering radius as $\delta_{(t,t)}$, let the constant $\mu_{t}=p(t)$ be the marginal probability. Assume that the conditional probability density function $p^{t}(y|\mathbf{x})$ is $\lambda_{t}$-Lipschitz, the squared loss function $l$ is $\lambda_{l}$-Lipschitz and $l$ is further bounded by $L_{l}$, the expected factual loss is bounded as follows:
    \begin{equation}
        \epsilon_{F}\leq\sum_{t\in\{0,1\}}\mu_{t}\left(\epsilon_{\mathcal{S}_{t}}+\delta_{(t,t)}(\lambda_{l}+\frac{1}{3}\lambda_{t}L^{\frac{3}{2}}_{l})\right).
    \end{equation}
\end{lemma}

\begin{proof}[Proof of Lemma \ref{lemma:factual_error}]

We start with the Tower Law for the key of the proof, let the estimation be $\hat{y}=\hat{f}(\mathbf{x})$ given the fixed treatment $t$, we have the expected loss to be decomposed in the general form:

\begin{subequations}
\begin{align}
\mathbb{E}_{\mathcal{X}}\left[(y-\hat{f}(\mathbf{x}))^{2}\right]&=\mathbb{E}_{\mathcal{X}}\left[\mathbb{E}_{\mathcal{Y}}\left[(y-\hat{f}(\mathbf{x}))^{2}\mid\mathbf{x}\right]\right]\\
&=\mathbb{E}_{\mathcal{X}}\left[\int_{\mathcal{Y}} (y - \hat{f}(\mathbf{x}))^2 p(y \mid \mathbf{x})dy\right].\label{eq:change_of_var}
% Original Loss Function
%&= \int_{\mathcal{X}} \int_{\mathcal{Y}} (y - \hat{f}^{t}(\mathbf{x}))^2 p(y \mid \mathbf{x})dy\,p(\mathbf{x})\,d\mathbf{x}\\
%&= \int_{\mathcal{Y}} \int_{\mathcal{X}}(y - \hat{f}^{t}(\mathbf{x}))^2 p(y \mid \mathbf{x})p(\mathbf{x})d\mathbf{x}\,dy\\
%&= \int_{\mathcal{Y}} \int_{\mathcal{X}}(y - \hat{f}^{t}(\mathbf{x}))^2 \frac{p(y \mid \mathbf{x})p(\mathbf{x})}{p(\mathbf{x}\mid y)}p(\mathbf{x}\mid y)d\mathbf{x}\,dy\\
%&= \int_{\mathcal{Y}} \int_{\mathcal{X}}(y - \hat{f}^{t}(\mathbf{x}))^2 p(\mathbf{x}\mid y)d\mathbf{x}\,p(y)\,dy\\
%&= \int_{\mathcal{Y}} \mathbb{E}_{\mathcal{X}}\left[(y-\hat{f}^{t}(\mathbf{x}))^{2}\mid y\right]\,p(y)\,dy\\
%&=\mathbb{E}_{\mathcal{Y}}\left[\mathbb{E}_{\mathcal{X}}\left[(y-\hat{f}^{t}(\mathbf{x}))^{2}\mid y\right]\right]
\end{align}
%Due to SUTVA and Consistency assumption \cite{imbens2015causal}, we have $y$ and the pair $(\mathbf{x},t)$ uniquely linked by the deterministic function $g$, e.g, $y=g(\mathbf{x},t)$. Thus, the conditional expectation in (\ref{eq:change_of_var}) given the fixed $y$ becomes $\mathbb{E}_{\mathcal{X}}\left[(y-\hat{f}^{t}(\mathbf{x}))^{2}\mid y\right]=(y-\hat{f}^{t}(\mathbf{x)})^{2}$, 

Denote $p^{t}(y|\mathbf{x})=p(y|\mathbf{x},t)$ and $p^{t}(\mathbf{x})=p(\mathbf{x}|t)$, we apply the Tower Law conclusion in Eq. (\ref{eq:change_of_var}) by Tow Law and further bound the expected factual error $\epsilon_{F}^{t}$ for the treatment group $t$:

\end{subequations}

    \begin{subequations}
    \begin{align}
        \epsilon^{t}_{F}=&\mathbb{E}_{\mathcal{X}^{t}}\left[(y-\hat{f}^{t}(\mathbf{x}))^{2}\right]\\
        =&\mathbb{E}_{\mathcal{X}^{t}}\left[\mathbb{E}_{\mathcal{Y}^{t}}\left[(y-\hat{f}^{t}(\mathbf{x}))^{2}\mid \mathbf{x}\right]\right]\\
        =&\mathbb{E}_{\mathcal{X}^{t}}\left[\int_{\mathcal{Y}}l^{t}_{y}p^{t}(y|\mathbf{x})dy\right]\\
        =&\mathbb{E}_{\mathcal{X}^{t}}\left[\int_{\mathcal{Y}}l^{t}_{y}(p^{t}(y|\mathbf{x})\underbrace{-p^{t}(y|\mathbf{x}')+p^{t}(y|\mathbf{x}')}_{\text{Add up to Zero}})dy\right]\\
        =&\mathbb{E}_{\mathcal{X}^{t}}\left[\int_{\mathcal{Y}}l^{t}_{y}p^{t}(y|\mathbf{x}')dy+\int_{\mathcal{Y}}l^{t}_{y}(p^{t}(y|\mathbf{x})-p^{t}(y|\mathbf{x}'))dy\right].\label{eq:first_and_second_factual}
        %\leq&\delta_{(t,t)}\lambda_{t}\int_{\mathcal{Y}^{t}}l^{t}_{y}dy+\int_{\mathcal{Y}^{t}}(l^{t}_{y}\underbrace{-l^{t}_{y'}+l^{t}_{y'}}_{\text{Add up to Zero}})p^{t}(y|\mathbf{x}')dy\\
        %=&\frac{1}{3}\delta_{(t,t)}\lambda_{t}L^{\frac{3}{2}}_{l}+\int_{\mathcal{Y}^{t}}(l^{t}_{y}-l^{t}_{y'})p^{t}(y|\mathbf{x}')dy+\int_{\mathcal{Y}^{t}}l^{t}_{y'}p^{t}(y|\mathbf{x}')dy\\
        %\leq&\frac{1}{3}\delta_{(t,t)}\lambda_{t}L^{\frac{3}{2}}_{l}+\delta_{(t,t)}\lambda_{l}+\epsilon_{\mathcal{S}_{t}}\\
        %=&\epsilon_{\mathcal{S}_{t}}+\delta_{(t,t)}(\frac{1}{3}\lambda_{t}L^{\frac{3}{2}}_{l}+\lambda_{l}),
    \end{align}
    \end{subequations} 
    We decompose the first term within the expectation in (\ref{eq:first_and_second_factual}) into the followings, by the selected $\mathbf{x}'\in\mathcal{S}_{t}$ that covers the $\mathbf{x}$ from group $t$ within the factual radius $\delta_{(t,t)}$, we bound the term as:
    \begin{subequations}
    \begin{align}
        \int_{\mathcal{Y}}l^{t}_{y}p^{t}(y|\mathbf{x}')dy&=\int_{\mathcal{Y}}(l^{t}_{y}-l^{t}_{y'}+l^{t}_{y'})p^{t}(y|\mathbf{x}')dy\\
        &=\int_{\mathcal{Y}}(l^{t}_{y}-l^{t}_{y'})p^{t}(y|\mathbf{x}')dy+\int_{\mathcal{Y}}l^{t}_{y'}\,p^{t}(y|\mathbf{x}')dy\\
        &\leq \delta_{(t,t)}\lambda_{l}+\int_{\mathcal{Y}}l^{t}_{y'}\,p^{t}(y|\mathbf{x}')dy\label{eq:loss_Lipschitz_factual}\\
        &=\delta_{(t,t)}\lambda_{l}+\epsilon_{\mathcal{S}_{t}}(\mathbf{x}'),
    \end{align} 
    \end{subequations}
    where the inequality in (\ref{eq:loss_Lipschitz_factual}) is because:
    \begin{subequations}
    \begin{align}
        \int_{\mathcal{Y}}(l^{t}_{y}-l^{t}_{y'})p^{t}(y|\mathbf{x}')dy&\leq |\mathbf{x}-\mathbf{x}'|\int_{\mathcal{Y}}\left|\frac{l^{t}_{y}-l^{t}_{y'}}{\mathbf{x}-\mathbf{x}'}\right|p^{t}(y|\mathbf{x}')dy\\
        &\leq|\mathbf{x}-\mathbf{x}'|\int_{\mathcal{Y}}\lambda_{l}p^{t}(y|\mathbf{x}')dy\\
        &\leq\delta_{(t,t)}\int_{\mathcal{Y}}\lambda_{l}p^{t}(y|\mathbf{x}')dy\\
        &=\delta_{(t,t)}\lambda_{l}\int_{\mathcal{Y}}p^{t}(y|\mathbf{x}')dy\\
        &=\delta_{(t,t)}\lambda_{l},\label{eq:integral_of_density_factual}
    \end{align}
    \end{subequations}
    for which the equality in (\ref{eq:integral_of_density_factual}) is because the integral of the density across the domain is 1:
    \begin{align}
        \int_{\mathcal{Y}}p^{t}(y|\mathbf{x}')dy = 1.
    \end{align}

The second term within the expectation in (\ref{eq:first_and_second_factual}) is bounded by:
\begin{subequations}
    \begin{align}
        \int_{\mathcal{Y}}l^{t}_{y}(p^{t}(y|\mathbf{x})-p^{t}(y|\mathbf{x}'))dy&\leq|\mathbf{x}-\mathbf{x}'|\int_{\mathcal{Y}}l^{t}_{y}\left|\frac{p^{t}(y|\mathbf{x})-p^{t}(y|\mathbf{x}')}{\mathbf{x}-\mathbf{x}'}\right|)dy\\
        &\leq\delta_{(t,t)}\lambda_{t}\int_{\mathcal{Y}}l^{t}_{y}dy\\
        &=\frac{1}{3}\delta_{(t,t)}\lambda_{t}L^{\frac{3}{2}}_{l}.
    \end{align}
\end{subequations}

Combining all the inequalities together, we have:
\begin{subequations}
    \begin{align}
        \epsilon^{t}_{F}&=\mathbb{E}_{\mathcal{X}^{t}}\left[\int_{\mathcal{Y}}l^{t}_{y}p^{t}(y|\mathbf{x}')dy+\int_{\mathcal{Y}}l^{t}_{y}(p^{t}(y|\mathbf{x})-p^{t}(y|\mathbf{x}'))dy\right]\\
        &\leq\mathbb{E}_{\mathcal{X}^{t}}\left[\epsilon_{\mathcal{S}_{t}}(\mathbf{x}')+\delta_{(t,t)}(\lambda_{l}+\frac{1}{3}\lambda_{t}L^{\frac{3}{2}}_{l})\right]\\
        &=\epsilon_{\mathcal{S}_{t}}+\delta_{(t,t)}(\lambda_{l}+\frac{1}{3}\lambda_{t}L^{\frac{3}{2}}_{l}),
    \end{align}
\end{subequations} where the last equality holds due to $\mathbb{E}[\epsilon_{\mathcal{S}_{t}}(\mathbf{x}')]=\epsilon_{\mathcal{S}_{t}}$ and the invariance of constants under expectation, i,e., $\mathbb{E}(c)=c$ for any constant $c$.

Given that $\mu_{t=1}=p(t=1)$ and $\mu_{t=0}=p(t=0)$, where $\mu_{t=1}+\mu_{t=0}=1$, we conclude the proof by expanding the expected factual loss $\epsilon_{F}$ by definition:
\begin{subequations}
    \begin{align}
        \epsilon_{F}&=\int_{\mathcal{X}\times\mathcal{T}}l(\mathbf{x},t)p(\mathbf{x},t)d\mathbf{x}dt\\
        &=\int_{\mathcal{X}}l^{t=1}(\mathbf{x})p(\mathbf{x},t=1)d\mathbf{x}+\int_{\mathcal{X}}l^{t=0}(\mathbf{x})p(\mathbf{x},t=0)d\mathbf{x}\\
        &=\int_{\mathcal{X}}l^{t=1}(\mathbf{x})p^{t=1}(\mathbf{x})p(t=1)d\mathbf{x}+\int_{\mathcal{X}}l^{t=0}(\mathbf{x})p^{t=0}(\mathbf{x})p(t=0)d\mathbf{x}\\
        &=p(t=1)\cdot\mathbb{E}_{\mathcal{X}^{t=1}}\left[(y-\hat{f}^{t=1}(\mathbf{x}))^{2}\right] + p(t=0)\cdot\mathbb{E}_{\mathcal{X}^{t=0}}\left[(y-\hat{f}^{t=0}(\mathbf{x}))^{2}\right]\\
        &=\mu_{t=1}\cdot\mathbb{E}_{\mathcal{X}^{t=1}}\left[\mathbb{E}_{\mathcal{Y}^{t=1}}\left[(y-\hat{f}^{t=1}(\mathbf{x}))^{2}\mid\mathbf{x}\right]\right] + \mu_{t=0}\cdot\mathbb{E}_{\mathcal{X}^{t=0}}\left[\mathbb{E}_{\mathcal{Y}^{t=0}}\left[(y-\hat{f}^{t=0}(\mathbf{x}))^{2}\mid\mathbf{x}\right]\right]\\
        &\leq\sum_{t\in\{0,1\}}\mu_{t}\left(\epsilon_{\mathcal{S}_{t}}+\delta_{(t,t)}
        (\lambda_{l}+\frac{1}{3}\lambda_{t}L^{\frac{3}{2}}_{l})\right).
    \end{align}
\end{subequations}

\textbf{Discussion}: Lemma \ref{lemma:factual_error} establishes a general upper bound on the factual loss under the core-set paradigm, which has been explored in prior studies \cite{sener2018active, qin2021budgeted}. Building on this foundation, we provide a more rigorous and comprehensive proof to strengthen the theoretical underpinnings. Also, we visualize the factual covering radius in Figure \ref{fig:factual_covering_11} and \ref{fig:factual_covering_00} to enhance interpretation, where the full coverage on the same class is required.
\end{proof}

\subsection{Proof of Lemma \ref{lemma:counterfactual_error}\label{appendix:lemma_3}}

\begin{definition}
\label{def:ipm}
    \textit{Let $\mathcal{H} = \{h | h: \mathcal{X}\rightarrow \mathbb{R}\}$ be a family of functions. The distribution distance measure -- integral probability metric (IPM) between two data distributions $p^{t=1}(\mathbf{x})$ and $p^{t=0}(\mathbf{x})$ over the domain $\mathcal{X}$ is defined as}:
\begin{equation}
\text{IPM}_{\mathcal{H}}(p^{t=1}(\mathbf{x}),p^{t=0}(\mathbf{x})) = \sup_{h\in \mathcal{H}} \left|\int_{\mathcal{X}}h(\mathbf{x})(p^{t=1}(\mathbf{x})-p^{t=0}(\mathbf{x}))d\mathbf{x}\right|.
\end{equation}
\end{definition}

\begin{lemma}\label{lemma:counterfactual_error}
     Denote the expected loss on subset $S_{t}$ as $\epsilon_{\mathcal{S}_{t}}$, the constant $\mu_{t}=p(t)$, counterfactual covering radius as $\delta_{(t,1-t)}$. Assume that the conditional probability density function $p^{t}(y|\mathbf{x})$ is $\lambda_{t}$-Lipschitz, the squared loss function $l$ is $\lambda_{l}$-Lipschitz and $l$ is further bounded by $L_{l}$. Also, let $\mathcal{H} = \{h | h: \mathcal{X}\rightarrow \mathbb{R}\}$ be a family of functions and $f:\mathcal{X}\times\mathcal{T}\rightarrow\mathcal{Y}$ be the hypothesis. Assume that there exists a constant $\kappa>0$, such that $h_{f}(\mathbf{x},t):=\frac{1}{\kappa}l_{f}(\mathbf{x},t)\in \mathcal{H}$. The counterfactual expected loss is bounded as follows:
    \begin{equation}
       \epsilon_{CF}\leq\sum_{t\in\{0,1\}}\mu_{1-t}\left(\epsilon_{\mathcal{S}_{t}}+\delta_{(t,1-t)}(\lambda_{l}+\frac{1}{3}\lambda_{t}L^{\frac{3}{2}}_{l})\right) + \kappa_{\mathcal{H}},
    \end{equation} where constant $\kappa_{\mathcal{H}}=\kappa\cdot\text{IPM}_{\mathcal{H}}(p^{t=1}(\mathbf{x}), p^{t=0}(\mathbf{x}))$ describes the distributional discrepancy between the treatment groups' distritbuions $(p^{t=1}(\mathbf{x})\text{ and }p^{t=0}(\mathbf{x}))$ over the domain $\mathcal{X}$, e.g., its i.i.d realization set $\mathcal{S}$ induced by AL.
\end{lemma}

\begin{proof}[Proof of Lemma \ref{lemma:counterfactual_error}]

For notation simplicity, we denote $l^{t}_{y}=(y-\hat{f}^{t}(\mathbf{x}))^{2}$, and denote $p^{1-t}(\mathbf{x})=p(\mathbf{x}\mid 1-t)$, for $\mathbf{x}\sim p^{1-t}(\mathbf{x})$, we have the counterfactual loss on the group $1-t$ as: \begin{subequations}
    \begin{align}
    \epsilon^{t}_{CF}=&\mathbb{E}_{\mathcal{X}^{1-t}}\left[(y-\hat{f}^{t}(\mathbf{x}))^{2}\right]\\
        =&\mathbb{E}_{\mathcal{X}^{1-t}}\left[\mathbb{E}_{\mathcal{Y}^{t}}\left[(y-\hat{f}^{t}(\mathbf{x}))^{2}\mid\mathbf{x}\right]\right]\\
        %=&\int_{\mathcal{X}}(y^{1-t}-\hat{f}^{1-t}(\mathbf{x}))^{2}p(x\mid t)d\mathbf{x}\label{eq:counterfactual_decomposition}\\
        %=&\int_{\mathcal{X}}_{x}p^{t}(\mathbf{x})d\mathbf{x}\\
        =&\mathbb{E}_{\mathcal{X}^{1-t}}\left[\int_{\mathcal{Y}}l^{t}_{y}p^{t}(y|\mathbf{x})dy\right]\\
        =&\mathbb{E}_{\mathcal{X}^{1-t}}\left[\int_{\mathcal{Y}}l^{t}_{y}\left(p^{t}(y|\mathbf{x})-p^{t}(y|\mathbf{x}')+p^{t}(y|\mathbf{x}')\right)dy\right]\label{eq:counterfactual_radius}\\
        =&\mathbb{E}_{\mathcal{X}^{1-t}}\left[\int_{\mathcal{Y}}l^{t}_{y}p^{t}(y|\mathbf{x}')dy+\int_{\mathcal{Y}}l^{t}_{y}\left(p^{t}(y|\mathbf{x})-p^{t}(y|\mathbf{x}')\right)dy\right]\label{eq:first_and_second_term}
    \end{align}
    Note, that in (\ref{eq:counterfactual_radius}) we introduce the selected $\mathbf{x}'\in \mathcal{S}_{t}$ from the distribution $p^{t}(\mathbf{x})$ ($\mathbf{x}'$ is not directly sampled from $p^{t}(\mathbf{x})$ but from $\mathcal{S}_{t}$), which can be consider counterfactual sample w.r.t. $\mathbf{x}\sim p^{1-t}(\mathbf{x})$. That is, the sample $\mathbf{x}$ from the group $1-t$ is covered by the counterfactual sample $\mathbf{x}'\in \mathcal{S}_{t}$ from group $t$ within the counterfactual radius $\delta_{(t,1-t)}$.
    %The term $\epsilon_{\mathcal{S}_{t}}$ is the expected loss for the selected samples from group $t$ because the integral is over $\mathbf{x}'$ as shown in (\ref{eq:coreset_loss}).
\begin{comment}
        \text{The second term is bounded as follows:}
    \begin{align}
        &\int_{\mathcal{Y}}l^{1-t}_{x}\left(p^{t}(y|\mathbf{x})-p_{y\sim f^{1-t}(x')}(y)\right)dy\\
        =&\int_{\mathcal{Y}}l^{1-t}_{x}\left(p^{t}(y|\mathbf{x})-p^{1-t}(y|\mathbf{x})+p^{1-t}(y|\mathbf{x})-p_{y\sim f^{1-t}(x')}(y)\right)dy\\
        =&\int_{\mathcal{Y}}l^{1-t}_{x}\left(p^{t}(y|\mathbf{x})-p^{1-t}(y|\mathbf{x})\right)dy+\int_{\mathcal{Y}}l^{1-t}_{x}\left(p^{1-t}(y|\mathbf{x})-p_{y\sim f^{1-t}(x')}(y)\right)dy\\
        \leq&~\Delta_{p}\int_{\mathcal{Y}}l^{1-t}_{x}dy+\delta^{(t)\rightarrow(1-t)}\lambda^{f^{1-t}}\int_{\mathcal{Y}}l^{1-t}_{x}dy\\
        =&~\frac{1}{3}(\Delta_{p}+\delta^{(t)\rightarrow(1-t)}\lambda^{f^{1-t}})L^{\frac{3}{2}}_{l}
    \end{align}
\end{comment}
\end{subequations} 

Similar to the deduction in Appendix \ref{appendix:lemma_2}, the first term within the expectation in (\ref{eq:first_and_second_term}) is bounded by: \begin{subequations}
    \begin{align}
        \int_{\mathcal{Y}}l^{t}_{y}p^{t}(y|\mathbf{x}')dy&=\int_{\mathcal{Y}}(l^{t}_{y}-l^{t}_{y'}+l^{t}_{y'})p^{t}(y|\mathbf{x}')dy\\
        &=\int_{\mathcal{Y}}(l^{t}_{y}-l^{t}_{y'})p^{t}(y|\mathbf{x}')dy+\int_{\mathcal{Y}}l^{t}_{y'}\,p^{t}(y|\mathbf{x}')dy\\
        &\leq \delta_{(t,1-t)}\lambda_{l}+\int_{\mathcal{Y}}l^{t}_{y'}\,p^{t}(y|\mathbf{x}')dy\label{eq:loss_Lipschitz}\\
        &=\delta_{(t,1-t)}\lambda_{l}+\epsilon_{\mathcal{S}_{t}}(\mathbf{x}')\label{eq:coreset_loss}
    \end{align} 
    The inequality in (\ref{eq:loss_Lipschitz}) is because:
    \begin{align}
        \int_{\mathcal{Y}}(l^{t}_{y}-l^{t}_{y'})p^{t}(y|\mathbf{x}')dy&\leq |\mathbf{x}-\mathbf{x}'|\int_{\mathcal{Y}}\left|\frac{l^{t}_{y}-l^{t}_{y'}}{\mathbf{x}-\mathbf{x}'}\right|p^{t}(y|\mathbf{x}')dy\\
        &\leq|\mathbf{x}-\mathbf{x}'|\int_{\mathcal{Y}}\lambda_{l}p^{t}(y|\mathbf{x}')dy\\
        &\leq\delta_{(t,1-t)}\int_{\mathcal{Y}}\lambda_{l}p^{t}(y|\mathbf{x}')dy\\
        &=\delta_{(t,1-t)}\lambda_{l}\int_{\mathcal{Y}}p^{t}(y|\mathbf{x}')dy\\
        &=\delta_{(t,1-t)}\lambda_{l}\label{eq:integral_of_density}
    \end{align}
    The equality in (\ref{eq:integral_of_density}) is because the integral of the density across the domain is 1:
    \begin{align}
        \int_{\mathcal{Y}}p^{t}(y|\mathbf{x}')dy = 1
    \end{align}
\end{subequations}

The second term within the expectation in (\ref{eq:first_and_second_term}) is bounded by:
\begin{subequations}
    \begin{align}
        \int_{\mathcal{Y}}l^{t}_{y}\left(p^{t}(y|\mathbf{x})-p^{t}(y|\mathbf{x}')\right)dy&=|\mathbf{x}-\mathbf{x}'|\int_{\mathcal{Y}}l^{t}_{y}\frac{|p^{t}(y|\mathbf{x})-p^{t}(y|\mathbf{x}')|}{|\mathbf{x}-\mathbf{x}'|}dy\\
        \leq&~\delta_{(t,1-t)}\lambda_{t}\int_{\mathcal{Y}}l^{t}_{y}dy\\
        =&~\frac{1}{3}\delta_{(t,1-t)}\lambda_{t}L^{\frac{3}{2}}_{l}
    \end{align}
\end{subequations}

Combining all the inequalities together, we have:
\begin{subequations}
    \begin{align}
    \epsilon^{t}_{CF}=&\mathbb{E}_{\mathcal{X}^{1-t}}\left[\int_{\mathcal{Y}}l^{t}_{y}p^{t}(y|\mathbf{x}')dy+\int_{\mathcal{Y}}l^{t}_{y}\left(p^{t}(y|\mathbf{x})-p^{t}(y|\mathbf{x}')\right)dy\right]\\
    \leq&\mathbb{E}_{\mathcal{X}^{1-t}}[\epsilon_{\mathcal{S}_{t}}(\mathbf{x}')]+\delta_{(t,1-t)}\lambda_{l}+\frac{1}{3}\delta_{(t,1-t)}\lambda_{t}L^{\frac{3}{2}}_{l}\\
    =&\mathbb{E}_{\mathcal{X}^{1-t}}[\epsilon_{\mathcal{S}_{t}}(\mathbf{x}')]+\delta_{(t,1-t)}\left(\lambda_{l}+\frac{1}{3}\lambda_{t}L^{\frac{3}{2}}_{l}\right),
    \end{align}
\end{subequations} where the first term is the counterfactual loss on the subset $S_{t}$ induced by AL.

%Update 5th June: The density difference can be bounded with the Wasserstein distance, here is the term that we can take as a constant. Because the problem that we target on is on AL from the core-set perspective, such that we would like to approach the empirical loss (entire data) with a lot less data (core-set), so we bound the loss with factual covering radius and counterfactual radius, to allow error with higher covering bound.

%Expanding the addition from (\ref{eq:counterfactual_loss}) as follows:
Extending the the expected counterfactual error $\epsilon_{CF}$ by definition:

\begin{subequations}
    \begin{align}
        \epsilon_{CF} =& \int_{\mathcal{X}\times\mathcal{T}}l(\mathbf{x},t)p(\mathbf{x},1-t)d\mathbf{x}dt\\
        =&\int_{\mathcal{X}}(y^{t=1}-\hat{f}^{t=1}(\mathbf{x}))^{2}p(\mathbf{x},t=0)d\mathbf{x}+\int_{\mathcal{X}}(y^{t=0}-\hat{f}^{t=0}(\mathbf{x}))^{2}p(\mathbf{x},t=1)d\mathbf{x}\\
        =&\int_{\mathcal{X}}(y^{t=1}-\hat{f}^{t=1}(\mathbf{x}))^{2}p^{t=0}(\mathbf{x})p(t=0)d\mathbf{x}+\int_{\mathcal{X}}(y^{t=0}-\hat{f}^{t=0}(\mathbf{x}))^{2}p^{t=1}(\mathbf{x})p(t=1)d\mathbf{x}\\
        =&p(t=0)\cdot\mathbb{E}_{\mathcal{X}^{t=0}}\left[(y^{t=1}-\hat{f}^{t=1}(\mathbf{x}))^{2}\right] + p(t=1)\cdot\mathbb{E}_{\mathcal{X}^{t=1}}\left[(y^{t=0}-\hat{f}^{t=0}(\mathbf{x}))^{2}\right]\\
        %\leq&\int_{\mathcal{X}}(y^{t=1}-\hat{f}^{t=1}(\mathbf{x}))^{2}p^{t=1}(\mathbf{x})d\mathbf{x}+\int_{\mathcal{X}}(y^{t=0}-\hat{f}^{t=0}(\mathbf{x}))^{2}p^{t=0}(\mathbf{x})d\mathbf{x}\\
        \leq&\mu_{t=1}\left[\mathbb{E}_{\mathcal{X}^{t=1}}[\epsilon_{\mathcal{S}_{t=0}}(\mathbf{x}')]+\delta_{(0,1)}\left(\lambda_{l}+\frac{1}{3}\lambda_{t=0}L^{\frac{3}{2}}_{l}\right)\right]+\mu_{t=0}\left[\mathbb{E}_{\mathcal{X}^{t=0}}[\epsilon_{\mathcal{S}_{t=1}}(\mathbf{x}')]+\delta_{(1,0)}\left(\lambda_{l}+\frac{1}{3}\lambda_{t=1}L^{\frac{3}{2}}_{l}\right)\right]\\
        =&(\epsilon_{CF})_{\mathcal{S}}+ \mu_{t=1}\cdot\delta_{(0,1)}\left(\lambda_{l}+\frac{1}{3}\lambda_{t=0}L^{\frac{3}{2}}_{l}\right)+\mu_{t=0}\cdot\delta_{(1,0)}\left(\lambda_{l}+\frac{1}{3}\lambda_{t=1}L^{\frac{3}{2}}_{l}\right)\label{eq:counterfactual_density_deduction}
        %&\mu_{t=0}\mathbb{E}_{\mathcal{X}^{t=0}}\left[\int_{\mathcal{Y}}l^{t=0}_{y}\left(p^{t=1}(y|\mathbf{x})-p^{t=0}(y|\mathbf{x})\right)dy\right]+\mu_{t=1}\mathbb{E}_{\mathcal{X}^{t=1}}\left[\int_{\mathcal{Y}}l^{t=1}_{y}\left(p^{t=0}(y|\mathbf{x})-p^{t=1}(y|\mathbf{x})\right)dy\right].
    \end{align}
\end{subequations}

Noticing that the counterfactual loss over the selected set $\mathcal{S}$ induced by AL: $(\epsilon_{CF})_{\mathcal{S}}$ in (\ref{eq:counterfactual_density_deduction}) can be further bounded by the IPM, e.g., Wasserstein distance, by adapting the proof of "\textbf{Lemma 1}" of the main text in \cite{shalit2017estimating} to the subset $\mathcal{S}$ induced by AL. Let $p(t=1)=\mu$, thus $p(t=0)=1-\mu$ by noting that $p(t=1)+p(t=0)=1$, we bound the first term in (\ref{eq:counterfactual_density_deduction}) as follows: 
\begin{subequations}
    \begin{align}
    &(\epsilon_{CF})_{\mathcal{S}}-\left[(1-\mu)\cdot(\epsilon_{F})_{\mathcal{S}_{t=1}}+\mu\cdot(\epsilon_{F})_{\mathcal{S}_{t=0}}\right]\\
    =&\left[(1-\mu)\cdot(\epsilon_{CF})_{\mathcal{S}_{t=1}}+\mu\cdot(\epsilon_{CF})_{\mathcal{S}_{t=0}}\right]-\left[(1-\mu)\cdot(\epsilon_{F})_{\mathcal{S}_{t=1}}+\mu\cdot(\epsilon_{F})_{\mathcal{S}_{t=0}}\right]\\
    =&(1-\mu)\cdot[(\epsilon_{CF})_{\mathcal{S}_{t=1}}-(\epsilon_{F})_{\mathcal{S}_{t=1}}]+\mu\cdot[(\epsilon_{CF})_{\mathcal{S}_{t=0}}-(\epsilon_{F})_{\mathcal{S}_{t=0}}]\\
    =&(1-\mu)\int_{\mathcal{X}}l^{t=1}_{x}\left(p^{t=0}(\mathbf{x})-p^{t=1}(\mathbf{x})\right)d\mathbf{x}+\mu\int_{\mathcal{X}}l^{t=0}_{x}\left(p^{t=1}(\mathbf{x})-p^{t=0}(\mathbf{x})\right)d\mathbf{x}\\
        =&(1-\mu)\,\kappa\cdot\int_{\mathcal{X}}\frac{1}{\kappa}\,l^{t=1}_{x}\left(p^{t=0}(\mathbf{x})-p^{t=1}(\mathbf{x})\right)d\mathbf{x}+\mu\kappa\cdot\int_{\mathcal{X}}\frac{1}{\kappa}\,l^{t=0}_{x}\left(p^{t=1}(\mathbf{x})-p^{t=0}(\mathbf{x})\right)d\mathbf{x}\\
        %&(1-2\mu)\int_{\mathcal{X}}l^{t=1}_{x}p^{t=0}(\mathbf{x})d\mathbf{x}+(2\mu-1)\int_{\mathcal{X}}l^{t=0}_{x}p^{t=1}(\mathbf{x})d\mathbf{x}\\
        \leq&(1-\mu)\,\kappa\cdot\sup_{h\in \mathcal{H}}\left|\int_{\mathcal{X}}h(x)\left(p^{t=0}(\mathbf{x})-p^{t=1}(\mathbf{x})\right)d\mathbf{x}\right|+\mu\kappa\cdot\sup_{h\in \mathcal{H}}\left|\int_{\mathcal{X}}h(x)\left(p^{t=1}(\mathbf{x})-p^{t=0}(\mathbf{x})\right)d\mathbf{x}\right|\\
        =&\kappa\cdot\sup_{h\in \mathcal{H}}\left|\int_{\mathcal{X}}h(x)\left(p^{t=1}(\mathbf{x})-p^{t=0}(\mathbf{x})\right)d\mathbf{x}\right|=\kappa\cdot\text{IPM}_{\mathcal{H}}(p^{t=1}(\mathbf{x}), p^{t=0}(\mathbf{x}))\label{eq:bound_by_IPM}.
        %+(1-2\mu)\epsilon^{t=1}_{\text{F}}+(2\mu-1)\epsilon^{t=0}_{\text{F}}
    \end{align}
\end{subequations}

Thus, noting that $(\epsilon_{F})_{\mathcal{S}_{t=1}}$ is indeed the expected factual loss over the subset $\mathcal{S}_{t=1}$ induced by AL, thus simplying the notation to $\epsilon_{\mathcal{S}_{t}}$, we conclude the proof by plugging (\ref{eq:bound_by_IPM}) back to (\ref{eq:counterfactual_density_deduction}):
\begin{subequations}
    \begin{align}
        \epsilon_{CF}\leq\sum_{t\in\{0,1\}}\mu_{1-t}\left(\epsilon_{\mathcal{S}_{t}}+\delta_{(t,1-t)}(\lambda_{l}+\frac{1}{3}\lambda_{t}L^{\frac{3}{2}}_{l})\right) + \kappa_{\mathcal{H}},
    \end{align}
\end{subequations} where the constant $\kappa_{\mathcal{H}}=\kappa\cdot\text{IPM}_{\mathcal{H}}(p^{t=1}(\mathbf{x}), p^{t=0}(\mathbf{x}))$ is derived in the similar fashion in \cite{shalit2017estimating}.
%Note, that we need to make the counterfactual covering radius to cover as much counterfactual samples as possible, also, the factual covering radius from the counterfactual side should also covering as factual sample as much as possible, that means, the centered points from two groups should be better \textit{symmetrical}, which actually force the two groups' distribution closer.

\end{proof}

\textbf{Discussion}: Thus, we provide an adjustable bound depending on the size of the covering radius, where the lower bound value is by labeling all the pool set to make the covering radius the least, the gap originates from the labeling less data but with some covering radius. Note, that generalization bound by \citet{shalit2017estimating} quantifies well the risk upper bound given the fully labeled pool set. The work \cite{qin2021budgeted} builds upon \cite{shalit2017estimating}, but certain aspects remain underexplored, that is, it does not show the importance of the counterfactual covering radius. We fill this crucial theoretical gap by providing an informative and complete bound under the data-expanding context from scratch to derive a new bound tailored specifically to treatment effect estimation with AL. Furthermore, Lemma \ref{lemma:counterfactual_error} unveils the distinctive \textit{counterfactual covering radius} which originates from the unique nature of the counterfactual prediction. We visualize the counterfactual covering radius $\delta_{(t,1-t)}$ in Figure \ref{fig:factual_covering_10} and \ref{fig:factual_covering_01} to facilitate the conceptualization of Lemma \ref{lemma:counterfactual_error} in the treatment effect estimation setting under AL paradigm, where the full coverage on the counterfactual class is required.

\subsection{Proof of Theorem \ref{theorem:2opt}\label{appendix:theorem_2}}

\begin{definition}
\label{definition:optimal_partition}
    Let $\mathcal{S}^{*}_{(t,t)}$ of size $B_{(t,t)}$ denote the optimal (OPT) subset for treatment group $t$. For each point $v^{t}\in \mathcal{S}^{*}_{(t,t)}$, let the cluster of $v^{t}$ be $\mathcal{C}^{*}_{v^{t}}=\{u^{t}\in\mathcal{D}_{t}: d(u^{t},v^{t})=\min_{v'\in \mathcal{S}^{*}_{(t,t)}}d(u^{t},v')\}$. As such, we have partitions $\mathcal{C}^{*}_{v^{t}_1},\mathcal{C}^{*}_{v^{t}_2},\dots,\mathcal{C}^{*}_{v^{t}_{B_{(t,t)}}}$, where each point $u^{t}\in\mathcal{D}_{t}$ is placed in the closest $\mathcal{C}^{*}_{v^{t}_i}$ w.r.t. $v^{t}_{i}\in \mathcal{S}^{*}_{(t,t)}$.
\end{definition}

\textbf{Theorem \ref{theorem:2opt}.} \emph{Under Assumption \ref{assumption:strong_ingore}, the sum of the covering radii returned by Algorithm \ref{alg:fccs} is upper-bounded by $2\times\sum_{t\{0,1\}}(\text{OPT}_{\delta_{(t,t)}}+\text{OPT}_{\delta_{(t,1-t)}})$
}

\begin{proof}[Proof of Theorem \ref{theorem:2opt}]
    Let the output of the Algorithm \ref{alg:fccs} be $\mathcal{S}$, specifically, the total budget $B$, additive arithmetically, is split into four parts with $B=B_{(1,1)}+B_{(1,0)}+B_{(0,0)}+B_{(0,1)}$, i.e., with each part to acquire designated point to reduce one of the four radius $\delta_{(1,1)}$, $\delta_{(1,0)}$, $\delta_{(0,0)}$, and $\delta_{(0,1)}$ at a time. Thus, we have $\mathcal{S}=\mathcal{S}_{(1,1)}\cup\mathcal{S}_{(1,0)}\cup\mathcal{S}_{(0,0)}\cup\mathcal{S}_{(0,1)}$, and noting that $\mathcal{S}_{1}=\mathcal{S}_{(1,1)}\cup\mathcal{S}_{(1,0)}$, and $\mathcal{S}_{0}=\mathcal{S}_{(0,0)}\cup\mathcal{S}_{(0,1)}$. To bound each of the radius:
    \begin{itemize}
        \item For $u^{t=1}\in\mathcal{D}_{1}$ to reduce $\delta_{(1,1)}$:\begin{equation}
        d\,(u^{t=1},\mathcal{S}_{1})\leq d\,(u^{t=1},\mathcal{S}_{(1,1)})\leq2\times OPT_{\delta_{(1,1)}},
    \end{equation} where the first inequality is because $\mathcal{S}_{(1,1)}\subset\mathcal{S}_{1}$, and the second inequality is by Lemma \ref{lemma:2opt_factual}.
        \item For $u^{t=0}\in\mathcal{D}_{0}$ to reduce $\delta_{(1,0)}$ :\begin{equation}
        d\,(u^{t=0},\mathcal{S}_{1})\leq d\,(u^{t=0},\mathcal{S}_{(1,0)})\leq2\times OPT_{\delta_{(1,0)}},
    \end{equation} where the first inequality is because $\mathcal{S}_{(1,0)}\subset\mathcal{S}_{1}$, and the second inequality is by Lemma \ref{lemma:2opt_counterfactual}.
        \item For $u^{t=0}\in\mathcal{D}_{0}$ to reduce $\delta_{(0,0)}$ :\begin{equation}
        d\,(u^{t=0},\mathcal{S}_{0})\leq d\,(u^{t=0},\mathcal{S}_{(0,0)})\leq2\times OPT_{\delta_{(0,0)}},
    \end{equation} where the first inequality is because $\mathcal{S}_{(0,0)}\subset\mathcal{S}_{0}$, and the second inequality is by Lemma \ref{lemma:2opt_factual}.
        \item For $u^{t=1}\in\mathcal{D}_{1}$ to reduce $\delta_{(0,1)}$ :\begin{equation}
        d\,(u^{t=1},\mathcal{S}_{0})\leq d\,(u^{t=0},\mathcal{S}_{(0,1)})\leq2\times OPT_{\delta_{(0,1)}},
    \end{equation} where the first inequality is because $\mathcal{S}_{(0,1)}\subset\mathcal{S}_{0}$, and the second inequality is by Lemma \ref{lemma:2opt_counterfactual}.
    \end{itemize}

    Since for all $u$ the above holds, thus it follows that 
    
    \begin{subequations}
        \begin{align}
            &\max_{u\in \mathcal{D}_{1}\backslash \mathcal{S}_{1}}d\,(u^{t=1},\mathcal{S}_{1})
    +\max_{u\in \mathcal{D}_{0}\backslash \Tilde{\mathcal{S}}_{0}}d\,(u^{t=0},\mathcal{S}_{1})
        +\max_{u\in \mathcal{D}_{0}\backslash \mathcal{S}_{0}}d\,(u^{t=0},\mathcal{S}_{0})
        +\max_{u\in \mathcal{D}_{1}\backslash \Tilde{\mathcal{S}}_{1}}d\,(u^{t=1},\mathcal{S}_{0})\\
        \leq&2\times OPT_{\delta_{(1,1)}}+2\times OPT_{\delta_{(1,0)}}+2\times OPT_{\delta_{(0,0)}}+2\times OPT_{\delta_{(0,1)}}\\
        \\\implies&\delta_{(1,1)}+\delta_{(1,0)}+\delta_{(0,0)}+\delta_{(0,1)}\leq2\times\sum_{t\{0,1\}}\left(\text{OPT}_{\delta_{(t,t)}}+\text{OPT}_{\delta_{(t,1-t)}}\right)
        \end{align}
    \end{subequations}

\end{proof}

In the following, we define $d(u,\mathcal{Q}):=\min_{v\in \mathcal{Q}}d(u,v)$ where $\mathcal{Q}$ can be any set, e.g., the optimal solution $\mathcal{S}^{*}_{(t,t)}$, also we denote the $OPT$ as the minimal covering radius returned by the optimal set $\mathcal{S}^{*}$, i.e., $OPT=\max_{u\in\mathcal{D}}\min_{v\in \mathcal{S}^{*}}d(u,v)$.

\begin{lemma}
   Without loss of generality to Definition \ref{definition:optimal_partition}, we have $\forall u,w\in \mathcal{C}_{v}\text{ for }v\in \mathcal{Q},\text{ then } d(u,w)\leq2\times\text{OPT}$\label{lemma:triangular_inequality}
\end{lemma}

\begin{proof}[Proof of Lemma \ref{lemma:triangular_inequality}]
\begin{subequations}
    \begin{align}
    d(u,w)&\leq d(u,v)+d(w,v)\label{eq:triangular}\\
    &=d(u,\mathcal{Q})+d(w,\mathcal{Q})\label{eq:partition_definition}\\
    &\leq \textit{OPT} + \textit{OPT} = 2\times\textit{OPT},\label{eq:2opt}
    \end{align}
\end{subequations} where the inequality (\ref{eq:triangular}) is by the triangular inequality, the equality (\ref{eq:partition_definition}) is by Definition \ref{definition:optimal_partition}, and we complete the proof with the inequality (\ref{eq:2opt}) due to $d(u,\mathcal{Q})=\min_{v\in \mathcal{Q}}d(u,v)\leq\max_{u\in\mathcal{D}}\min_{v\in \mathcal{Q}}d(u,v)=\textit{OPT}.$
\end{proof}

Note, that the minimal covering radius $\delta_{(\cdot,\cdot)}$ under the optimal solution $\mathcal{S}^{*}_{(\cdot,\cdot)}$ is denoted as $\text{OPT}_{\delta_{(\cdot,\cdot)}}$.

\begin{lemma}
    With budget $B_{(t,t)}$ Subset $S_{(t,t)}$ for the treatment group $t$ returned by Algorithm \ref{alg:fccs} is a $2-OPT_{\delta_{(t,t)}}$ for covering the set $\mathcal{D}_{t}$ of the treatment group $t$.\label{lemma:2opt_factual}
\end{lemma}

\begin{proof}[Proof of Lemma \ref{lemma:2opt_factual}]
    When covering the treatment group $t$, we dose not have any assumption for the data distribution, thus we split the proof into two scenarios in a similar fashion by \citet{dinitz2019lecture4}, i.e., $\forall v\in \mathcal{S}^{*}_{(t,t)}, \mathcal{S}_{(t,t)}\cap \mathcal{C}^{*}_{v}\neq\varnothing$ and $\exists v\in \mathcal{S}^{*}_{(t,t)}, \mathcal{S}_{(t,t)}\cap \mathcal{C}^{*}_{v}=\varnothing$:
    \begin{itemize}
        \item $\forall v\in \mathcal{S}^{*}_{(t,t)}, \mathcal{S}_{(t,t)}\cap \mathcal{C}^{*}_{v}\neq\varnothing$: Let $w\in \mathcal{S}_{(t,t)} \cap \mathcal{C}^{*}_{v}$, for $u\in\mathcal{D}_{t}$, let $u\in\mathcal{C}^{*}_{v}$ with $v\in \mathcal{S}^{*}_{(t,t)}$ (noting that $\mathcal{C}^{*}_{v}\subset\mathcal{D}_{t}$), we have:
            \begin{align}
                d(u,\mathcal{S}_{(t,t)})\leq d(u,w)\leq 2\times\text{OPT}_{\delta_{(t,t)}}\label{eq:scenario_1}
            \end{align}
        where the first inequality is because $w\in \mathcal{S}_{(t,t)}$, and the second inequality is by Lemmma \ref{lemma:triangular_inequality} for $u,w\in \mathcal{C}^{*}_{v}$. Note, that Eq. (\ref{eq:scenario_1}) holds $\forall u\in\mathcal{D}_{t}$ as $\mathcal{C}^{*}_{v}\subset\mathcal{D}_{t}$, then $\delta_{(t,t)}=\max_{u\in\mathcal{D}}d(u,\mathcal{S}_{(t,t)})\leq 2\times\text{OPT}_{\delta_{(t,t)}}$.
        
        \item $\exists v\in \mathcal{S}^{*}_{(t,t)}, \mathcal{S}_{(t,t)}\cap \mathcal{C}^{*}_{v}=\varnothing$: Since $|\mathcal{S}_{(t,t)}|=|\mathcal{S}^{*}|=B_{(t,t)}$, by the pigeonhole principle, $\exists v'\in \mathcal{S}^{*}_{(t,t)}, s.t. |\mathcal{S}_{(t,t)}\cap \mathcal{C}^{*}_{v^{'}}|\geq2$. Thus, assume that for $z,m\in \mathcal{S}_{(t,t)}\cap \mathcal{C}^{*}_{v^{'}}$ and $z$ is added to $\mathcal{S}_{(t,t)}$ before $m$. Let $\mathcal{S}_{(t,t)}'=\mathcal{S}_{(t,t)}\backslash m$, then for $u\in\mathcal{D}_{t}$, we have:
        \begin{subequations}
            \begin{align}
                d(u,\mathcal{S}_{(t,t)})&\leq d(u,\mathcal{S}_{(t,t)}')\label{eq:scenario_2_subset}\\
                &\leq d(m,\mathcal{S}_{(t,t)}')\label{eq:scenario_2_greedy}\\
                &\leq d(m,z)\label{eq:scenario_2_definition}\\
                &\leq 2\times\text{OPT}_{\delta_{(t,t)}}\label{eq:scenario_2_triangular},
            \end{align}
        \end{subequations} where inequality in (\ref{eq:scenario_2_subset}) is because $\mathcal{S}_{(t,t)}'\subset \mathcal{S}_{(t,t)}$ (the radius decreases monotonically with larger set), the inequality in (\ref{eq:scenario_2_greedy}) is because the greedy selection with larger distance with set $\mathcal{S}_{(t,t)}'$ ($m$ is selected before $u$), the inequality in (\ref{eq:scenario_2_definition}) is by definition that $d(m,\mathcal{S}_{(t,t)}')=\min_{v\in \mathcal{S}_{(t,t)}'}d(m,v)$ and $z\in \mathcal{S}_{(t,t)}'$, the last inequality in (\ref{eq:scenario_2_triangular}) is by Lemma \ref{lemma:triangular_inequality}.
    \end{itemize}

    Since for all $u\in\mathcal{D}_{t}$, we have $d(u,\mathcal{S}_{(t,t)})\leq2\times OPT_{\delta_{(t,t)}}$, then it follows that $\max_{u\in\mathcal{D}_{t}}d(u,\mathcal{S}_{(t,t)})\leq2\times OPT_{\delta_{(t,t)}}$ to conclude the proof for \ref{lemma:2opt_factual} by enumerating all the scenarios.
    
\end{proof}

\begin{lemma}\label{lemma:2opt_counterfactual}
     With budget $B_{(t,1-t)}$ Subset $\mathcal{S}_{(t,1-t)}$ for the treatment group $t$ returned by Algorithm \ref{alg:fccs} is a $2-OPT_{\delta_{(t,1-t)}}$ for covering the set $\mathcal{D}_{1-t}$ of the treatment group $1-t$.
\end{lemma}

\begin{proof}[Proof of Lemma \ref{lemma:2opt_counterfactual}]
    Note, that the reduction of the counterfactual covering radius $\delta_{(t,1-t)}$ requires the acquisition from group $t$, which is complex because the radius $\delta_{(t,1-t)}$ is calculated by $\delta_{(t,1-t)}=\max_{i\in \mathcal{D}_{1-t}\backslash \Tilde{S}_{1-t}}\min_{j\in \mathcal{S}_{(t,1-t)}} d(\mathbf{x}^{1-t}_{i},\mathbf{x}^{t}_{j})$, noting that the proxy collection $\Tilde{S}_{1-t}$ is from the treatment group $1-t$, however, the acquisition targeting querying the sample from group $t$ to reduce the counterfactual covering radius $\delta_{(t,1-t)}$. 
    
    Thus, by the greedy nature of the Algorithm \ref{alg:fccs}, when $\delta_{(t,1-t)}$ is the one to be reduced, the mentality for the query step is to calculate the proxy point $a_{1-t}=\argmax_{i\in \mathcal{D}_{1-t}\backslash \Tilde{S}_{1-t}}\min_{j\in \mathcal{S}_{(t,1-t)}} d(\mathbf{x}^{1-t}_{i},\mathbf{x}^{t}_{j})$ from group $1-t$, then find the nearest point $a_{t}\in\mathcal{D}_{t}$ to $a_{1-t}$ as the factual query to expand $\mathcal{S}_{(t,1-t)}$. If there always exists $d(a_{t},a_{1-t})=0$, the radius reduction can be real quick.

\begin{definition}
\label{definition:counterfactual_optimal_partition}
    Let cluster $\mathcal{F}^{*}_{v^{t}}=\{u^{1-t}\in\mathcal{D}_{1-t}: d(u^{1-t},v^{t})=\min_{v'\in \mathcal{S}^{*}_{(t,1-t)}}d(u^{1-t},v')\}$ for $v^{t}\in\mathcal{S}^{*}_{(t,1-t)}$.
\end{definition}
    
    The proof adopts the similar mentality as shown in Proof of Lemma \ref{lemma:2opt_factual}, For now we show the proof for the first scenario and the second scenarios follows.
    
    $\forall v\in \mathcal{S}^{*}_{(t,1-t)}, \Tilde{\mathcal{S}}_{1-t}\cap \mathcal{F}^{*}_{v}\neq\varnothing$: Let $w\in \Tilde{\mathcal{S}}_{1-t} \cap \mathcal{F}^{*}_{v}$, let $u\in\mathcal{F}^{*}_{v}$ with $v\in \mathcal{S}^{*}_{(t,1-t)}$ (noting that $\mathcal{F}^{*}_{v}\subset\mathcal{D}_{1-t}$), we have:
    \begin{subequations}
    \begin{align}
        d(u^{1-t},\mathcal{S}_{(t,1-t)})&=d(u^{1-t},\Tilde{\mathcal{S}}_{1-t})\label{eq:a}\\
         &\leq d(u^{1-t},w^{1-t})\label{eq:b}\\
         &\leq d(u^{1-t},v^{t})+d(u^{1-t},v^{t})\label{eq:c}\\
         &\leq d(u^{1-t},\mathcal{S}^{*}_{(t,1-t)})+d(w^{1-t},\mathcal{S}^{*}_{(t,1-t)})\label{eq:d}\\
         &\leq 2\times\text{OPT}_{\delta_{(t,1-t)}}\label{eq:e},
    \end{align}
    \end{subequations}
    where the equality (\ref{eq:a}) is by the Assumption (Strong Ignorability) \ref{assumption:strong_ingore} (i.e., $0<p(t|\mathbf{x})<1$), s.t., for $\Tilde{\mathcal{S}}_{1-t}$ returned by Algorithm \ref{alg:fccs}, we can have identical set $\mathcal{S}_{(t,1-t)}\in\mathcal{D}_{t}$ to the proxy collection $\Tilde{\mathcal{S}}_{1-t}$, the inequality in (\ref{eq:b}) is because $w^{1-t}\in\Tilde{\mathcal{S}}_{1-t}$, the inequality (\ref{eq:c}) is by the triangular inequality, the inequality (\ref{eq:d}) is by Definition \ref{definition:counterfactual_optimal_partition}, and inequality (\ref{eq:e}) is by $d(u^{1-t},\mathcal{S}^{*}_{(t,1-t)})=\min_{v\in \mathcal{S}^{*}_{(t,1-t)}}d(u^{1-t},v)\leq\max_{u^{1-t}\in\mathcal{D}_{1-t}}\min_{v\in \mathcal{S}^{*}_{(t,1-t)}}d(u^{1-t},v)=\textit{OPT}_{\delta_{(t,1-t)}}.$
\end{proof}

\subsection{Proof of Theorem \ref{theorem:coverage}\label{appendix:theorem_3}}

\textbf{Assumption \ref{assumption:full_coverage}}\emph{
    Given the fixed covering radius $\delta_{(t,t)}$ and $\delta_{(t,1-t)}$, there exists the optimal solution $\mathcal{S}^{*}_{t}$, $\mathcal{S}^{*}_{t}\subset \mathcal{S}^{*}$ for treatment group $t$, such that $\mathcal{A}_{F}^{t=1}\bigcup\mathcal{A}_{CF}^{t=1}=\mathcal{D}.$
}

\begin{definition}\label{definition:algorithm_2}
    Let the pool set be $\mathcal{D}$ of size $n$, $\mathcal{S}_{t}$ of size $B_{t}$ be the solution for group $t$ returned by Algorithm \ref{alg:fccm}. A family of sets $\mathcal{U}=\{\mathcal{U}_{i}\}^{v_{m}}_{i=v_{1}}$, where $\forall i, \,\mathcal{U}_{i}\subset\mathcal{D}$. \textit{Note, that the subscript $i$ of $\mathcal{U}_{i}$ denotes a single selected point $v_{i}$ in $\mathcal{S}_{t}\subseteq[v_{m}]$ due to the fact that each $v_{i}$ as the center with the fixed radius covers a set of point, i.e., by definition $\mathcal{U}_{{i}}=\mathcal{A}_{(t,t)}(v_{i})\cup\mathcal{A}_{(t,1-t)}(v_{i})$}. Let the uncovered set be $\Omega_{r}=\mathcal{D}\backslash\bigcup_{i\in\mathcal{S}_{t}}\mathcal{U}_{i}$ up to iteration $r$, and assume that Algorithm \ref{alg:fccm} were to pick $\mathcal{U}'_{1}\,,\mathcal{U}'_{2}\,,...,\mathcal{U}'_{k}$, for which $\mathcal{U}'_{i}$ is one of the set in $\mathcal{U}$. Denote the set covered by optimal solution $\mathcal{S}^{*}_{t}$ as $\Theta$, $\omega_{r}=|\mathcal{U}'_{r}\cap\Omega_{r-1}|$, and $\eta_{i}=|\Theta|-\sum_{j\leq i}\omega_{j}$. 
\end{definition}

\textbf{Theorem \ref{theorem:coverage}.} \emph{Under Assumption \ref{assumption:full_coverage}, Algorithm \ref{alg:fccm} is a $(1-\frac{1}{e})$ -- approximation for the full coverage constraint on the equally weighted graph and unscaled out-degree.
}

\begin{proof}[Proof of Theorem \ref{theorem:coverage}]
    We prove the $(1-\frac{1}{e})-$approximation for the full coverage constraint by Algorithm \ref{alg:fccm} by extending the method by \citet{dinitz2019lecture4} for solving the conventional Max $k-$Cover Problem into our \textit{Factual and Counterfactual Coverage Maximization} for the data-efficient treatment effect estimation problem. The objective is to query $\mathcal{S}=\mathcal{S}_{1}\cup\mathcal{S}_{0}$ that maximizes the mean coverage $P(\mathcal{A})$.

    As defined in Section \ref{section:fccm}, $\mathcal{A}_{F}^{t=1}=\bigcup_{\mathbf{x}\in \mathcal{S}_{1}} \mathcal{A}_{(1,1)}(\mathbf{x})$ and $\mathcal{A}_{CF}^{t=1}=\bigcup_{\mathbf{x}\in \mathcal{S}_{1}} \mathcal{A}_{(1,0)}(\mathbf{x})$, by further in Definition \ref{definition:algorithm_2}, for the factual ($\mathcal{A}^{t=1}_{F}$) and counterfactual ($\mathcal{A}^{t=1}_{CF}$) covered region by $S_{1}$, we have:
    \begin{subequations}
        \begin{align}
            \mathcal{A}^{t=1}_{F}\cup\mathcal{A}^{t=1}_{CF}&=|\bigcup_{j\leq B_{t}}\mathcal{U}_{j}'|\label{eq:by_definition_1}\\
            &=\sum_{j\leq B_{t}}\omega_{j}\label{eq:by_definition_2}\\
            &=|\Theta|-\eta_{B_{t}}\label{eq:by_definition_3}\\
            &\geq|\Theta|-|\Theta|\,e^{-1}=(1-e^{-1})\,|\Theta|,\label{eq:by_lemma}
        \end{align}
    \end{subequations}

    where the equality in (\ref{eq:by_definition_1}) is by definition of the covered region in Section \ref{section:fccm} and the set $\mathcal{U}_{i}$ in Definition \ref{definition:algorithm_2}, the equality in (\ref{eq:by_definition_2}) and (\ref{eq:by_definition_3}) is by Definition \ref{definition:algorithm_2}, and the inequaltiy in (\ref{eq:by_lemma}) is by Lemma \ref{lemma:sets_approximation}.
    
    Under Assumption \ref{assumption:full_coverage}, there exists optimal solution s.t. $\Theta=\mathcal{D}$, which further implies:
    \begin{equation}
        \mathcal{A}^{t=1}_{F}\cup\mathcal{A}^{t=1}_{CF}\,\geq\,(1-e^{-1})|\mathcal{D}|\implies P(\mathcal{A}^{t=1}_{F}\cup\mathcal{A}^{t=1}_{CF})\geq\frac{(1-e^{-1})|\mathcal{D}|}{|\mathcal{D}|}=1-e^{-1}.\label{lemma:coverage_inequality_1}
    \end{equation}

    Without loss of generality, the proof above applies for proving that 
    \begin{equation}
        P(\mathcal{A}^{t=0}_{F}\cup\mathcal{A}^{t=0}_{CF})\,\geq\,1-e^{-1}.\label{lemma:coverage_inequality_0}
    \end{equation}

    To conclude, as $\mathcal{}$ we have:
    \begin{subequations}
        \begin{align}
            P(\mathcal{A})&=\frac{1}{4}\,P(\mathcal{A}^{t=1}_{F})+\frac{1}{4}\,P(\mathcal{A}^{t=1}_{CF})+\frac{1}{4}\,P(\mathcal{A}^{t=0}_{F})+\frac{1}{4}\,P(\mathcal{A}^{t=0}_{CF})\label{eq:coverage_decompositon}\\
            &=\frac{1}{2}\,P(\mathcal{A}^{t=1}_{F}\cup\mathcal{A}^{t=1}_{CF})+\frac{1}{2}\,P(\mathcal{A}^{t=0}_{F}\cup\mathcal{A}^{t=0}_{CF})\label{eq:independence}\\
            &\geq\,1-e^{-1}\label{eq:coverage_inequality},
        \end{align}
    \end{subequations} where the equality in (\ref{eq:coverage_decompositon}) is by the definition for the mean coverage $P(\mathcal{A})$ and note that our maximization goal in Eq. (\ref{eq:full_coverage_constraint}) leave out the constant coefficient $1/4$, which does not affect the ultimate goal. The equality in (\ref{eq:independence}) is by the independence between the factual covering and counterfactual covering, the inequality in (\ref{eq:coverage_inequality}) is by conclusion in Eq. (\ref{lemma:coverage_inequality_1}) and (\ref{lemma:coverage_inequality_0}).
\end{proof}

\begin{lemma}\label{lemma:general_set_inequality}
    For set $\mathcal{P}$, $\mathcal{Q}$ and let $\mid\mathcal{P}\mid\,\geq\,\mid\mathcal{Q}\mid$, we have $\mid\mathcal{P}\backslash\mathcal{Q}\mid\,\geq\,\mid\mathcal{P}\mid-\mid\mathcal{Q}\mid$.
\end{lemma}

\begin{proof}[Proof of Lemma \ref{lemma:general_set_inequality}] $\mid\mathcal{P}\backslash\mathcal{Q}\mid\,=\,\mid\mathcal{P}\mid-\mid\mathcal{P}\cap\mathcal{Q}\mid\,\geq\,\mid\mathcal{P}\mid-\mid\mathcal{Q}\mid$ due to the fact that $\mid\mathcal{P}\cap\mathcal{Q}\mid\leq\mid\mathcal{Q}\mid$ for $\mid\mathcal{P}\mid\,\geq\,\mid\mathcal{Q}\mid$.
\end{proof}

\begin{lemma}
    $\eta_{i}\leq\left|\Theta\backslash\bigcup_{j\leq i}\mathcal{U}'_{j}\right|$.\label{lemma:set_inequality}
\end{lemma}

\begin{proof}[Proof of Lemma \ref{lemma:set_inequality}]
    \begin{equation}
            \eta_{i}=|\Theta|-\sum_{j\leq i}\omega_{j}
            =|\Theta|-|\bigcup_{j\leq i}\mathcal{U}'_{j}|\leq|\Theta\backslash\bigcup_{j\leq i}\mathcal{U}'_{j}|.
    \end{equation}
The first and second equality is straightforward by definition and observation, and the inequality is by Lemma \ref{lemma:general_set_inequality} due to the fact that the optimal cover $\Theta$ is the larger.
\end{proof} 

\begin{lemma}\label{lemma:sets_approximation}
    $\eta_{B_{t}}\leq|\Theta|\,e^{-1}$.
\end{lemma}

\begin{proof}
    [Proof of Lemma \ref{lemma:sets_approximation}]
    Construct the subtraction:
    \begin{subequations}
        \begin{align}
             &\eta_{i}-\eta_{i-1}=|\Theta|-\sum_{j\leq i}\omega_{j} - \left(|\Theta|-\sum_{j\leq i-1}\omega_{j}\right)=-\omega_{i}\\
             \implies& \eta_{i}=\eta_{i-1}-\omega_{i}
        \end{align}
    \end{subequations}
    Lemma \ref{lemma:set_inequality} implies that optimal solution $\Theta$ covers at least $\eta_{i}$ uncovered samples (the uncovered samples are w.r.t. to the covered set by $\mathcal{S}_{t}$) with $B_{t}$ sets, thus by the pigeonhole principle, when $\eta_{i}$ filled into $B_{t}$ sets, each set assigned averagely $\eta_{i}/B_{t}$, however, there are less than $B_{t}$ sets availables for the Algorithm \ref{alg:fccm} to query, thus there exists the set covers at least $\eta_{i}/B_{t}$ uncovered samples. By the greedy nature of the Algorithm \ref{alg:fccm}, we have $\omega_{i+1}\geq\frac{\eta_{i}}{B_{t}}$, then it further implies that:
    \begin{subequations}
        \begin{align}
            &\eta_{i}=\eta_{i-1}-\omega_{i}
            \leq\eta_{i-1}-\frac{\eta_{i-1}}{B_{t}}
            =\eta_{i-1}(1-\frac{1}{B_{t}})\\
            \implies&\frac{\eta_{i}}{\eta_{i-1}}=\frac{1}{B_{t}}
        \end{align}
    \end{subequations}

    For $\omega_{0}=0$ implies that $\eta_{0}=|\Theta|$, and maximally with budget $B_{t}$, we have $\eta_{B_{t}}$, and then construct the following:
    \begin{equation}
        \frac{\eta_{B_{t}}}{\eta_{0}}=\underbrace{\frac{\eta_{B_{t}}}{\eta_{B_{t}-1}}\times\frac{\eta_{B_{t}-1}}{\eta_{B_{t}-2}}\times\cdots\times\frac{\eta_{2}}{\eta_{1}}\times\frac{\eta_{1}}{\eta_{0}}}_{B_{t}\text{ quantities above}}=(1-\frac{1}{B_{t}})^{B_{t}}
    \end{equation}

    Noting that:
    \begin{equation}
        \lim_{B_{t}\rightarrow\infty}(1-\frac{1}{B_{t}})^{B_{t}}=\frac{1}{e}\implies \frac{\eta_{B_{t}}}{\eta_{0}}\leq\frac{1}{e}.
    \end{equation}

    Thus, it can be concluded that $\eta_{B_{t}}\leq|\Theta|\,e^{-1}$.
\end{proof}

\section{Related Work\label{appendix:related_works}}

\textbf{Treatment effect estimation.} Many early works in causal effect estimation (a.k.a., the treatment effect estimation) focus on group-level estimation, e.g., conditional average treatment effect (CATE). The widely used inverse probability weighting method \cite{rosenbaum1983central,imbens2015causal} and the doubly robust model \cite{robins1994estimation} are designed to mitigate the selection bias in CATE estimation, but are not generalizable to the unseen individuals or groups without labels. So far, various methods \cite{shalit2017estimating,louizos2017causal,alaa2017bayesian,yao2018representation,yoon2018ganite,shi2019adapting,zhang2020learning,kallus2020deepmatch,jesson2020identifying,wang2024optimal} have been proposed due to the proliferation of deep learning (DL). These parametric models are good at modeling the individual-level causal effect and are generalizable to unseen instances. Furthermore, the strong expressive power of such deep models can handle the high-dimensional data and relax the pivot assumptions in causal effect estimation, e.g., unconfoundedness assumption, by learning the deconfounded latent representations via neural mapping for treated and control groups. Additionally, there is another branch of work that investigate the treatment effect estimation under interference (e.g., the violation of the SUTVA assumptions) \cite{rakesh2018linked,ma2021causal,ma2022learning,lin2023estimating,lin2024treatment,chen2024doubly,lin2025scalable}, which is out-of-scope to the foucs of this paper.

\textbf{Active learning.} The concept of active learning (AL) dates back over a century \cite{smith1918standard}. Over time, it has evolved into a prominent branch of machine learning research \cite{settles2009active, ren2021survey, zhan2022comparative}. The primary goal of AL is to optimize model performance in a cost-efficient manner, achieving low model risk while minimizing the number of labeled samples required. AL methods are typically categorized into three main scenarios: query synthesis \cite{wang2015active}, stream-based \cite{fujii2016budgeted}, and pool-based approaches \cite{wu2018pool}. This paper focuses on pool-based AL, particularly in regression problems, where key acquisition strategies include uncertainty-based sampling \cite{gal2017deep}, density-based querying \cite{sener2018active}, and hybrid methods \cite{ash2019deep}. For example, Bayesian Active Learning by Disagreement (BALD) \cite{gal2017deep} uses epistemic uncertainty to select unlabeled samples, while core-set \cite{sener2018active} prioritizes samples based on their maximum distance to the nearest neighbor in the hidden space. ACS-FW \cite{pinsler2019bayesian} combines core-set and Bayesian approaches, balancing sample diversity and uncertainty in batch-mode acquisition. Although general AL methods are not specifically designed for CEE, benchmarking these methods can yield valuable insights.

\textbf{Treatment effect estimation with active learning.} Thus far, some progress has been made in this area of research. For instance, \cite{sundin2019active} proposes a querying criterion based on the estimated S-type error rate—the probability that the model incorrectly infers the sign of the treatment effect. However, this work focuses on estimating the correct sign of the treatment effect, which differs from the risk metric used in our study. For research aligned with the same risk metric, \citet{qin2021budgeted} introduce a theoretical framework that extends the upper-bound formulation from \cite{shalit2017estimating} mainly by a core-set approach \cite{tsang2005core}. Despite this, their proposed algorithm QHTE does not adequately address distribution alignment during data acquisition. To mitigate acquisition imbalance, Causal-BALD \cite{jesson2021causal} adopts an information-theoretic perspective, introducing the $\mu\rho$BALD criterion. This criterion scales the acquisition metric inversely with counterfactual variance, encouraging the selection of samples that align with similar counterfactuals when certain treatments are underrepresented. This represents a notable improvement over its predecessor, $\mu$BALD, an uncertainty-based softmax-BALD method \cite{kirsch2021stochastic}. However, Causal-BALD depends heavily on accurate uncertainty quantification and computationally intensive training using complex estimators, such as deep kernel learning \cite{wilson2016deep}. Recently, \citet{wen2024progressive} introduced a straightforward yet effective algorithm, MACAL, which reduces distributional discrepancies while remaining model-independent. Despite its advantages, MACAL requires querying data in pairs (one from the treated group and one from the control group), limiting its generalizability in scenarios where optimality can be achieved by querying from only one treatment group. Furthermore, while MACAL includes convergence analysis for sub-objectives, it lacks guarantees on overall risk upper-bound convergence. Additionally, some studies \cite{deng2011active, addanki2022sample, connolly2023task, ghadiri2024finite} leverage AL for efficient experimental trial design, where treatment information is applied only after sample acquisition, rather than being included in the initial pool. This setup differs fundamentally from our focus, where treatment information is available from the start.

% Reset the algorithm counter
\setcounter{algorithm}{0} % Reset to 0
\renewcommand{\thealgorithm}{1} % Ensure the numbering is fixed to 1

\begin{algorithm}[h!]
\caption{Greedy Radius Reduction\label{appendix:alg_1}}
\begin{algorithmic}[1]
   \STATE \textbf{Input:} Pool set $D=\mathcal{D}_{1}\cup \mathcal{D}_{0}$; random initialization $S^{\text{init}}=\mathcal{S}_{1}\cup \mathcal{S}_{0}$, where $\mathcal{S}_{1}$ and $\mathcal{S}_{0}$ are the random initialization set for the treated and control group respectively; budget $B$; distance metric $d(\cdot,\cdot)$
   \STATE $\mathcal{S}=\mathcal{S}^{\text{init}},\Tilde{S}_{1}=\varnothing,\Tilde{S}_{0}=\varnothing$
    \WHILE{$|\mathcal{S}|<|S^{\text{init}}|+B$} 
    
    %\STATE Calculate $\delta_{(1,1)}, \delta_{(1,0)}, \delta_{(0,0)}, \text{and }\delta_{(0,1)}$

    \STATE $\delta = \max\{\delta_{(1,1)}, \delta_{(1,0)}, \delta_{(0,0)}, \delta_{(0,1)}\}$

    \IF{$\delta == \delta_{(1,1)}$}

    \STATE$a=\argmax_{i\in \mathcal{D}_{1}\backslash \mathcal{S}_{1}}\min_{j\in \mathcal{S}_{1}} d(\mathbf{x}^{t=1}_{i},\mathbf{x}^{t=1}_{j})$\hfill\COMMENT{\textcolor{gray}{To reduce $\delta_{(1,1)}$ by querying data point from group $t=1$}}

    \ELSIF{$\delta == \delta_{(0,0)}$}

    \STATE$a=\argmax_{i\in \mathcal{D}_{0}\backslash \mathcal{S}_{0}}\min_{j\in \mathcal{S}_{0}} d(\mathbf{x}^{t=0}_{i},\mathbf{x}^{t=0}_{j})$\hfill\COMMENT{\textcolor{gray}{To reduce $\delta_{(0,0)}$ by querying data point from group $t=0$}}

    \ELSIF{$\delta == \delta_{(1,0)}$}

    \STATE $a'=\argmax_{i\in \mathcal{D}_{0}\backslash\Tilde{S}_{0}}\min_{j\in \mathcal{S}_{1}} d(\mathbf{x}^{t=0}_{i},\mathbf{x}^{t=1}_{j})$\hfill\COMMENT{\textcolor{gray}{To reduce $\delta_{(1,0)}$ by querying data point from group $t=0$}}
    \STATE $b=\argmin_{i\in \mathcal{D}_{1}}d(\mathbf{x}^{t=1}_{i},a')$\hfill\COMMENT{\textcolor{gray}{Find the nearest point from $t=1$ to $a'$ to reduce counterfactual radius $\delta_{(1,0)}$}}
    
    %We should check if $a$ already in $\mathcal{S}_{1}$ or not, if yes, $\delta^{(1)\rightarrow(0)}$ can not improve anymore. We should also exclude $a'$ from $\mathcal{D}_{0}$ Adding from following:

    \IF{$b\notin \mathcal{S}_{1}$} \STATE$a=b, \Tilde{S}_{0}=\Tilde{S}_{0}\cup\{a'\}$\hfill\COMMENT{\textcolor{gray}{Counterfactual radius $\delta_{(1,0)}$ can be reduced, add $b$ into the training set}}
    
    \ELSE \STATE \text{Repeat Line 4-25 by excluding $\delta_{(1,0)}$ from the $\max$ function}\hfill\COMMENT{\textcolor{gray}{Counterfactual radius $\delta_{(1,0)}$ cannot be reduced, go back to other reducible covering radii from the largest one}}
    \ENDIF
    
    \ELSIF{$\delta == \delta_{(0,1)}$}

    \STATE $a'=\argmax_{i\in \mathcal{D}_{1}\backslash\Tilde{S}_{1}}\min_{j\in \mathcal{S}_{0}} d(\mathbf{x}^{t=1}_{i},\mathbf{x}^{t=0}_{j})$\hfill\COMMENT{\textcolor{gray}{To reduce $\delta_{(0,1)}$ by querying data point from group $t=1$}}
    \STATE $b=\argmin_{i\in \mathcal{D}_{0}}d(\mathbf{x}^{t=0}_{i},a')$\hfill\COMMENT{\textcolor{gray}{Find the nearest point from $t=0$ to $a'$ to reduce counterfactual radius $\delta_{(0,1)}$}}
    
    %We should check if $a$ already in $\mathcal{S}_{1}$ or not, if yes, $\delta^{(1)\rightarrow(0)}$ can not improve anymore. We should also exclude $a'$ from $\mathcal{D}_{0}$ Adding from following:

    \IF{$b\notin \mathcal{S}_{0}$} \STATE$a=b, \Tilde{S}_{1}=\Tilde{S}_{1}\cup\{a'\}$\hfill\COMMENT{\textcolor{gray}{Counterfactual radius $\delta_{(0,1)}$ can be reduced, add $b$ into the training set}}
    
    \ELSE \STATE \text{Repeat Line 4-25 by excluding $\delta_{(0,1)}$ from the $\max$ function}\hfill\COMMENT{\textcolor{gray}{Counterfactual radius $\delta_{(0,1)}$ cannot be reduced, go back to other reducible covering radii from the largest one}}
    \ENDIF
    
    \ENDIF

    % The following is for the detailed t=0 
    %\IF{$\delta_{(0,0)}\geq \delta_{(0,1)}$}
    %\STATE $a=\argmax_{i\in \mathcal{D}_{0}\backslash \mathcal{S}_{0}}\min_{j\in \mathcal{S}_{0}} d(\mathbf{x}^{t=0}_{i},\mathbf{x}^{t=0}_{j})$
    
    %\ELSE
    
    %\STATE $a'=\argmax_{i\in \mathcal{D}_{1}\backslash\Tilde{S}_{1}}\min_{j\in \mathcal{S}_{0}} d(\mathbf{x}^{t=1}_{i},\mathbf{x}^{t=0}_{j})$, \

    %$b=\argmin_{i\in \mathcal{D}_{0}}d(\mathbf{x}^{t=1}_{i},a')$
    %\IF{$b\notin \mathcal{S}_{0}$}
    %\STATE$a=b, \Tilde{S}_{1}=\Tilde{S}_{1}\cup\{a'\}$
    %\ELSE \STATE$a=\argmax_{i\in \mathcal{D}_{0}\backslash \mathcal{S}_{0}}\min_{j\in \mathcal{S}_{0}} d(\mathbf{x}^{t=0}_{i},\mathbf{x}^{t=0}_{j})$
    %\ENDIF
    
    %\ENDIF

    \STATE $\mathcal{S}=\mathcal{S}\cup\{a\}$
    
    \ENDWHILE
    \STATE \textbf{Output:} $\mathcal{S}$
\end{algorithmic}
\end{algorithm}

\newpage
\section{Additional Details}

\subsection{Algorithm 1\label{appendix:detailed_alg_1}}

\textbf{Proxy collection $\Tilde{S}_{1-t}$}. To promote sample diversity in the acquired data, the reduction of the counterfactual covering radius $\delta_{(t,1-t)}$ queries the data from group $t$, however, such radius is calculated by $\delta_{(t,1-t)}=\max_{i\in \mathcal{D}_{1-t}\backslash \Tilde{S}_{1-t}}\min_{j\in \mathcal{S}_{t}} d(\mathbf{x}^{1-t}_{i},\mathbf{x}^{t}_{j})$, where a direct acquisition from the treatment group $1-t$ is performed. But our actual acquisition targets the query from group $t$ to reduce the counterfactual covering radius $\delta_{(t,1-t)}$ by covering the counterfactual samples. Thus, by the greedy nature of the Algorithm \ref{alg:fccs}, when $\delta_{(t,1-t)}$ is the one to be reduced, the mentality for the query step is to calculate the proxy point $a'=\argmax_{i\in \mathcal{D}_{1-t}\backslash \Tilde{S}_{1-t}}\min_{j\in \mathcal{S}_{(t,1-t)}} d(\mathbf{x}^{1-t}_{i},\mathbf{x}^{t}_{j})$ from group $1-t$, then find the nearest point $a_{t}\in\mathcal{D}_{t}$ to $a'$ as the eventual query to expand $\mathcal{S}_{t}$, and adding the proxy pint $a'$ into the proxy collection $\Tilde{S}_{1-t}$ as a already marked position. 

\textbf{Heads-up:} Note that querying the factual sample to reduce the counterfactual covering radius is not a necessity to help reduce the counterfactual covering radius, because a direct acquisition on the counterfactual samples can indeed reduce the counterfactual covering radius by the rigorous math definition. However, doing so cannot promote the sample diversity in the counterfactual group, for example, querying the counterfactual sample fall under the neighborhood of an acquired counterfactual sample to reduce the counterfactual covering radius can cause redundancy under limited budget.

\subsection{Algorithm 2\label{appendix:more_explanation_alg_2}}

The data acquisition performed in this algorithm is mainly repeating the two steps: 1). Pick the node with the highest scaled out-degree (which is recalculated each round for new point selection); 2). remove the incoming factual/out-going counterfactual edges to the picked node and its neighbors. Note that the scaled out-degree is by multiplying the out-degree with an coefficient $c$ that is directly associated with the covering radii and the covered points to further balance the distribution discrepancy. That is, given the union of factual and counterfactual covering ball induced by center $\mathbf{x}$: $\mathcal{A}_{(t,t)}(\mathbf{x})\bigcup\mathcal{A}_{(t,1-t)}(\mathbf{x})$, three main possible scenarios are: \begin{figure*}[h!]
  \centering
  \subfigure[$\mathcal{A}_{(t,t)}(\mathbf{x})\neq\varnothing\land \mathcal{A}_{(t,1-t)}(\mathbf{x})\neq\varnothing$]{\includegraphics[width=0.28\textwidth]{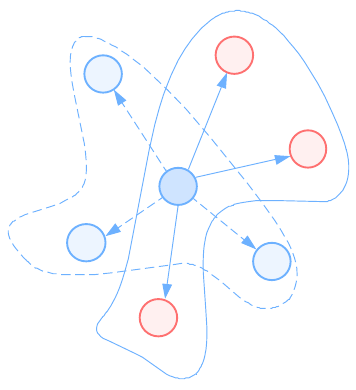}\label{fig:local_neighbour_1}}\quad\quad
  \subfigure[$\mathcal{A}_{(t,t)}(\mathbf{x})=\varnothing\land \mathcal{A}_{(t,1-t)}(\mathbf{x})\neq\varnothing$]{\includegraphics[width=0.28\textwidth]{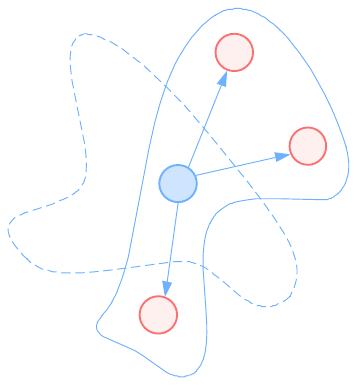}\label{fig:local_neighbour_2}}\quad\quad
  \subfigure[$\mathcal{A}_{(t,t)}(\mathbf{x})\neq\varnothing\land \mathcal{A}_{(t,1-t)}(\mathbf{x})=\varnothing$]{\includegraphics[width=0.28\textwidth]{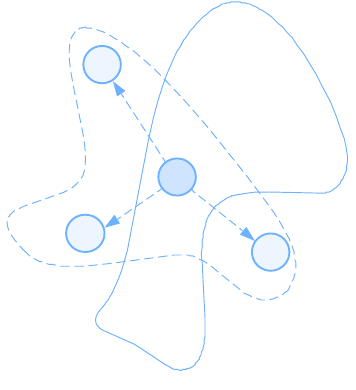}\label{fig:local_neighbour_3}}
  %\vspace{-0.2cm}
  \caption{Visualization of the factual covering (dashed manifold) and the counterfactual covering (solid manifold) on the local neighborhood from the center $\mathbf{x}$. Note that the zero-neighbor scenario is omitted as it is the least preferred to maximize the coverage, \label{fig:local_neighbours}}
\end{figure*} 

Thus, the scaling coefficient $c$ is calculated as: 
\begin{equation}
c(\mathcal{\mathbf{x}})=\zeta(\mathcal{\mathbf{x}})(1-\zeta(\mathcal{\mathbf{x}})), \forall\zeta(\mathcal{\mathbf{x}})=\frac{|\mathcal{A}_{(t,t)}(\mathbf{x})|}{|\mathcal{A}_{(t,t)}(\mathbf{x})\bigcup\mathcal{A}_{(t,1-t)}(\mathbf{x})|}\in[0,1],
\end{equation} where the minimum and maximum of the coefficient $c$ is obtained respectively by $\zeta=0$ (Figure \ref{fig:local_neighbour_2} and \ref{fig:local_neighbour_3}) and $\zeta=\frac{1}{2}$ (Figure \ref{fig:local_neighbour_1}). Noting that the shape of the coefficient $c$ w.r.t. $\zeta$ over the interval [0,1] is a downward-opening parabola, which gives the maximum $c$ on the evenly divided scenario, with $c$ declines when $\zeta$ departing from $\frac{1}{2}$. Additionally, when $c=0$, scenario \ref{fig:local_neighbour_2} is preferred to scenario \ref{fig:local_neighbour_3} from overlapping perspective with higher number of counterfactual neighbors prioritized.

\subsection{Dataset Details\label{appendix:dataset}}

\textbf{Toy:} The 2-dimensional dataset is generated by creating multiple clusters for each treatment group $t$. Let center of cluster $v^{t}_{i}=(x^{t}_{i,1},x^{t}_{i,2})$ to be drawn randomly from the uniform distribution:
\begin{equation}
    x^{t}_{i,1},\,x^{t}_{i,2}\sim\text{Uniform}([-9+\beta_{t},9+\beta_{t}]),
\end{equation} where the $\beta_{t}$ is a offset to create larger distribution discrepancy between two treatment groups. Let the collection of the centers of size $k$ for group $t$ be $\mathcal{V}^{t}_{k}=\{v^{t}_{i}\}^{k}_{i=1}$, the randomly added center $v^{t}_{k}$ should satisfy:
\begin{equation}
    d_{v'\in\mathcal{V}^{t}_{k-1}}(v^{t}_{k},v')\geq1.5\times0.9^{j},\label{eq:new_centers}
\end{equation} where $j$ counts from 0 with unit increment each time when Eq. (\ref{eq:new_centers}) is not satisfied for the randomly generated center $v^{t}_{k}$ over 100 times. That is, we reduce the minimum distance to a smaller one if the qualified sample cannot be found from the remaining pool set.

Given the centers, $\mathcal{V}^{t}_{n'_{t}}$, we define the mean $\mu^{t}_{j}=\mathcal{V}^{t}_{n'_{t}}[j]$ and variance $\sigma^{2}_{j}=1$ and generating $\mathbf{X}^{t}(i)\in\mathbb{R}^{n_{t}\times2}$ for cluster $i$:
\begin{equation}
    \mathbf{X}^{t}_{j}(i)\sim\mathcal{N}(\mu^{t}_{j},\sigma^{2}_{j}),\,\forall\,j\leq n_{t}\implies \mathbf{X}^{t}=\bigcup_{i\leq n'_{t}}\mathbf{X}^{t}(i).
\end{equation} Thus, set $\beta_{1}=2$ and $\beta_{0}=-2$, number of clusters $n'_{1}=50$ and $n'_{0}=30$, number of samples for each cluster $n_{1}=n_{0}=200$, such that we generate covariate matrix $\mathbf{X}^{t=1}\in\mathbb{R}^{10000\times2}$ for treatment group $t=1$ and covariate matrix $\mathbf{X}^{t=0}\in\mathbb{R}^{6000\times2}$ for treatment group $t=1$. See visualization in main text Figure \ref{fig:data_distribution} for the generated data. Furthermore, let $\mathbf{T}\in\mathbb{R}^{(n'_{t}\cdot n_{t})\times1}$ be the treatment indicator, to simulate the response curve $y$, we have: \begin{equation}
    y^{t}_{i}=\sin\,(1.5\times x^{t}_{i,1})+\cos\,(1.5\times x^{t}_{i,2}) +5t,\forall\,(\mathbf{x}^{t}_{i},t)\in(\mathbf{X}^{t},\mathbf{T}).
\end{equation}

\textbf{IBM \cite{shimoni2018benchmarking}:} This dataset is based on the real-world 177 covariates from a cohort of 100,000 individuals, from the publicly available Linked Births and Infant Deaths Database. The generated response curve is based on the randomly selected 25,000 individuals out of the 100,000 base, and the potential outcomes have 10 different simulations according to \cite{shimoni2018benchmarking}. 

\textbf{CMNIST \cite{jesson2021quantifying}:} This dataset contains 60,000 image samples (10 classes) of size 28$\times$28, which are adapted from MINIST \cite{lecun1998mnist} benchmark. CMNIST is completely distinct from the previous tabular datasets by leveraging the image data for the treatment effect estimation. The potential outcomes are simulated 10 times and generated by projecting the digits into a 1-dimensional latent manifold as described in \cite{jesson2021quantifying}.

\subsection{Further Discussions\label{appendix:discussion_ablation}}
\begin{figure}[t!]
  \centering
  \subfigure[IBM]{\includegraphics[width=0.33\textwidth]{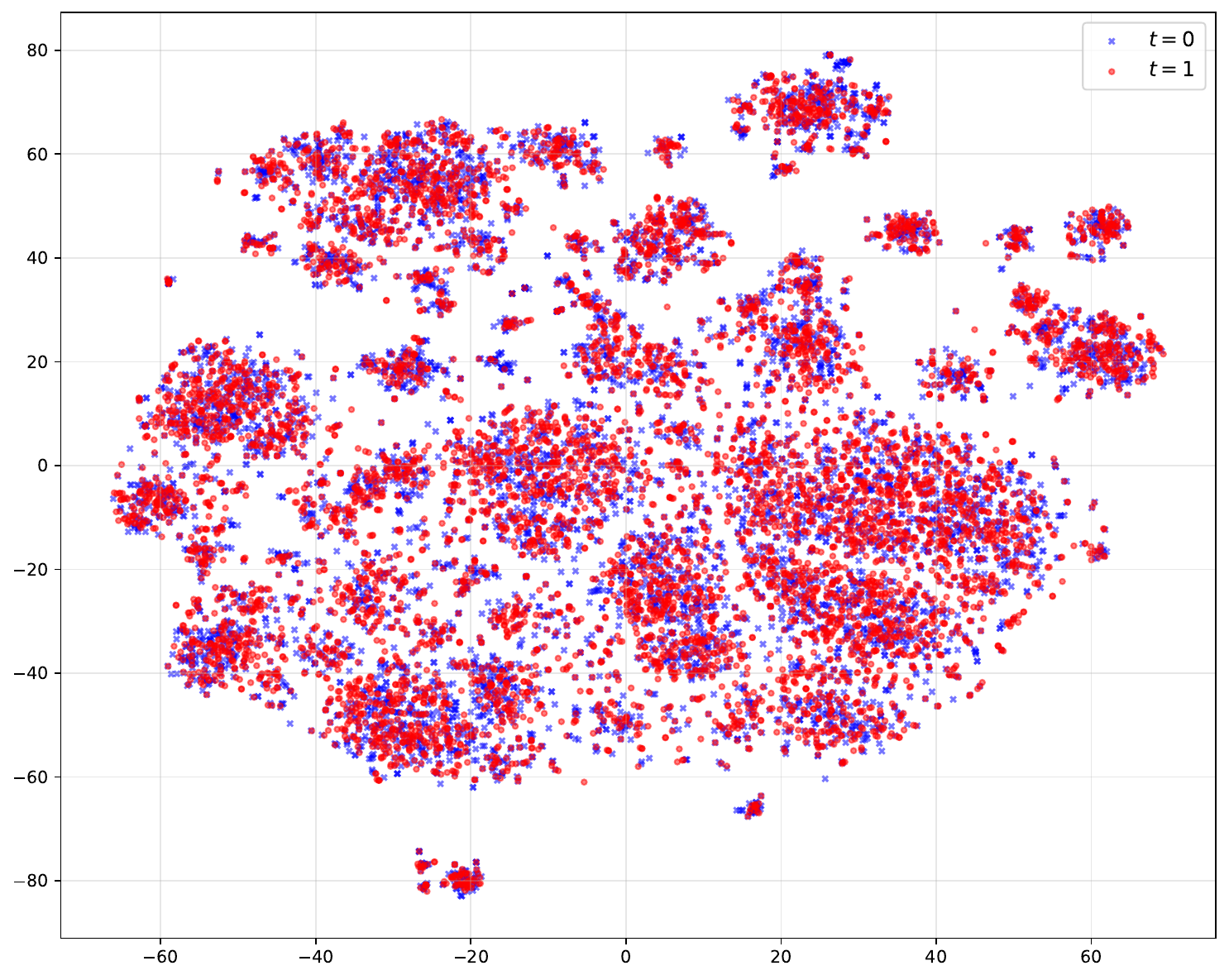}\label{fig:ibm_distribution}}
    \subfigure[$t=1$]{\includegraphics[width=0.33\textwidth]{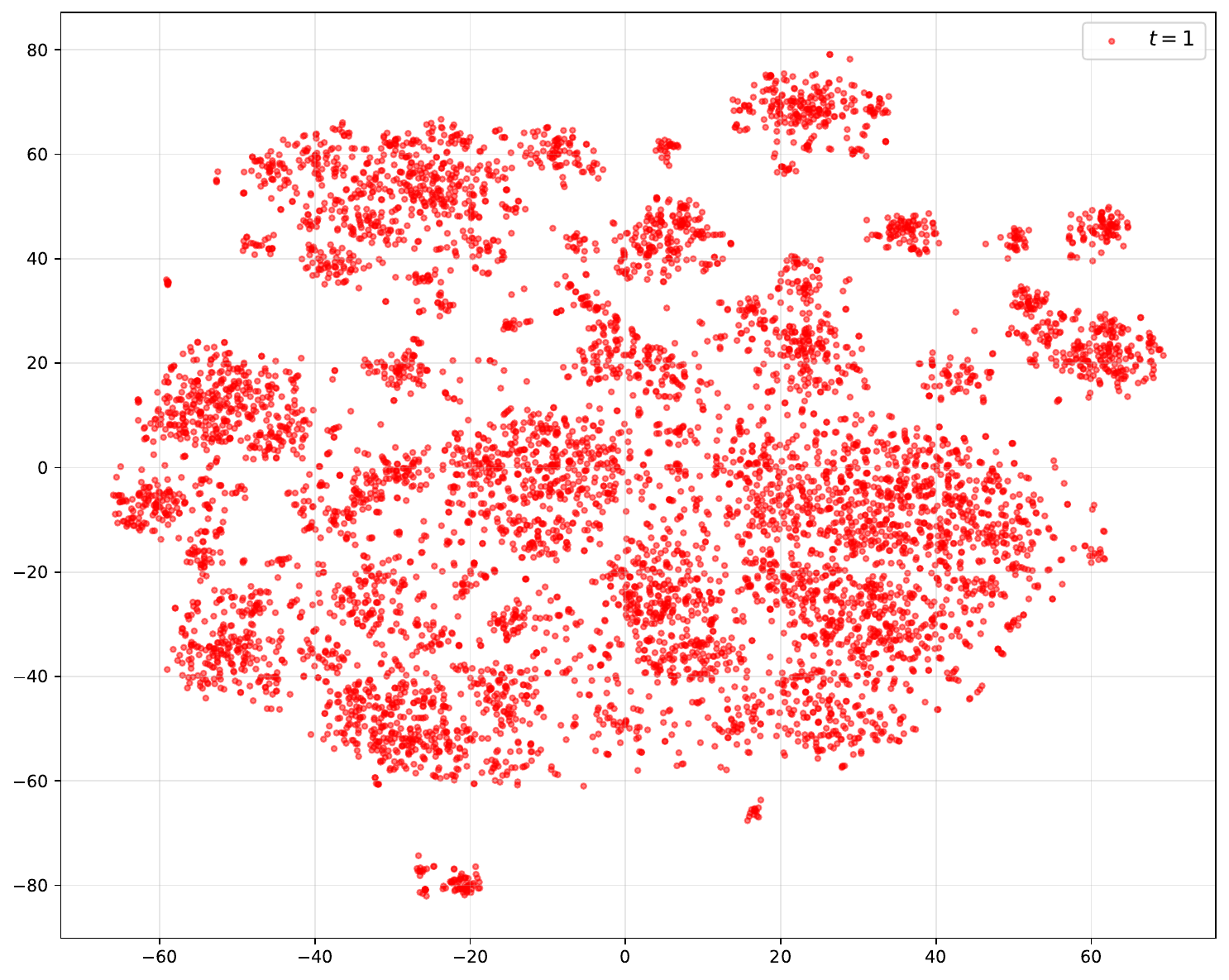}\label{fig:ibm_distribution_1}}
  \subfigure[$t=0$]{\includegraphics[width=0.33\textwidth]{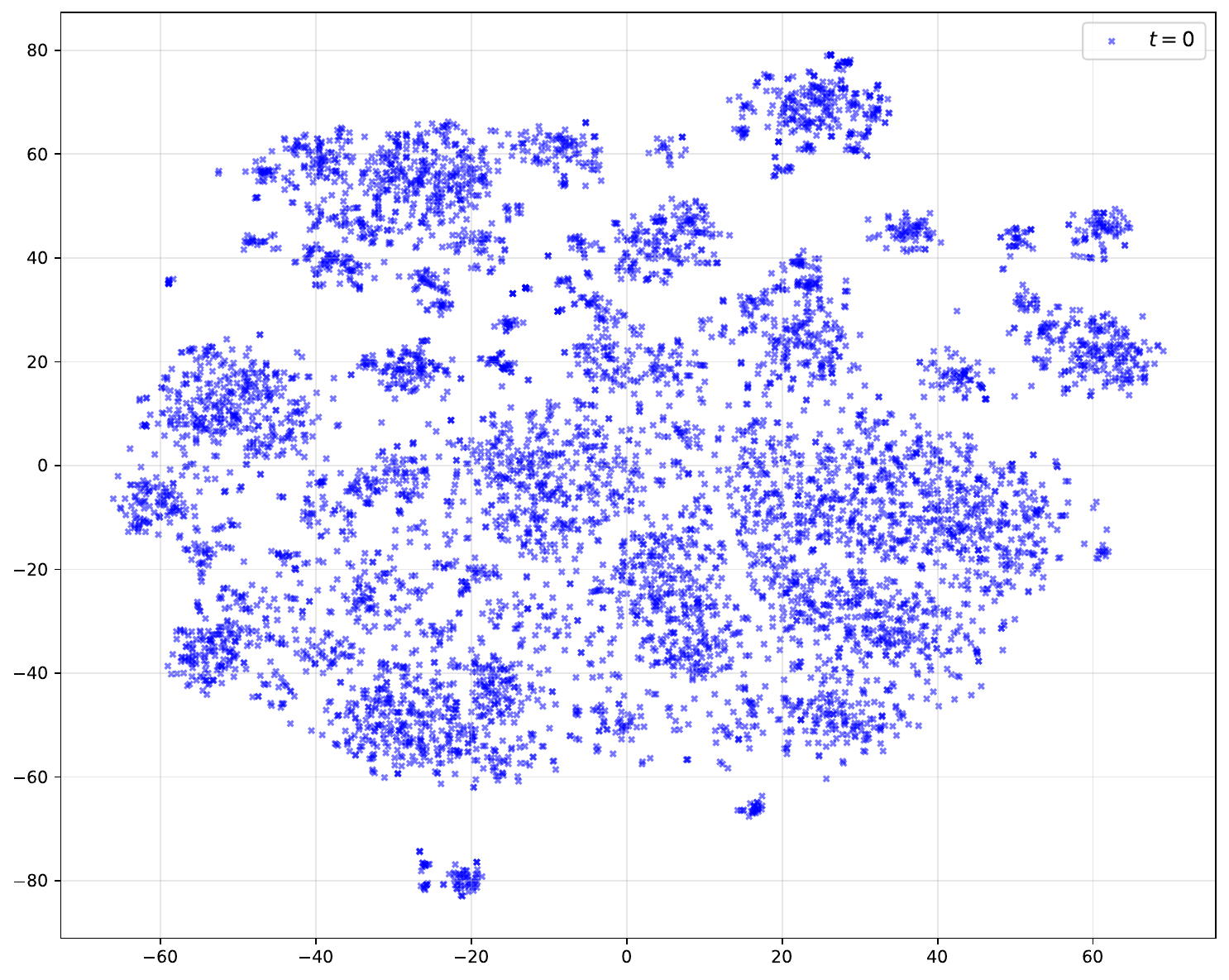}\label{fig:ibm_distribution_0}}
  \subfigure[CMNIST]{\includegraphics[width=0.33\textwidth]{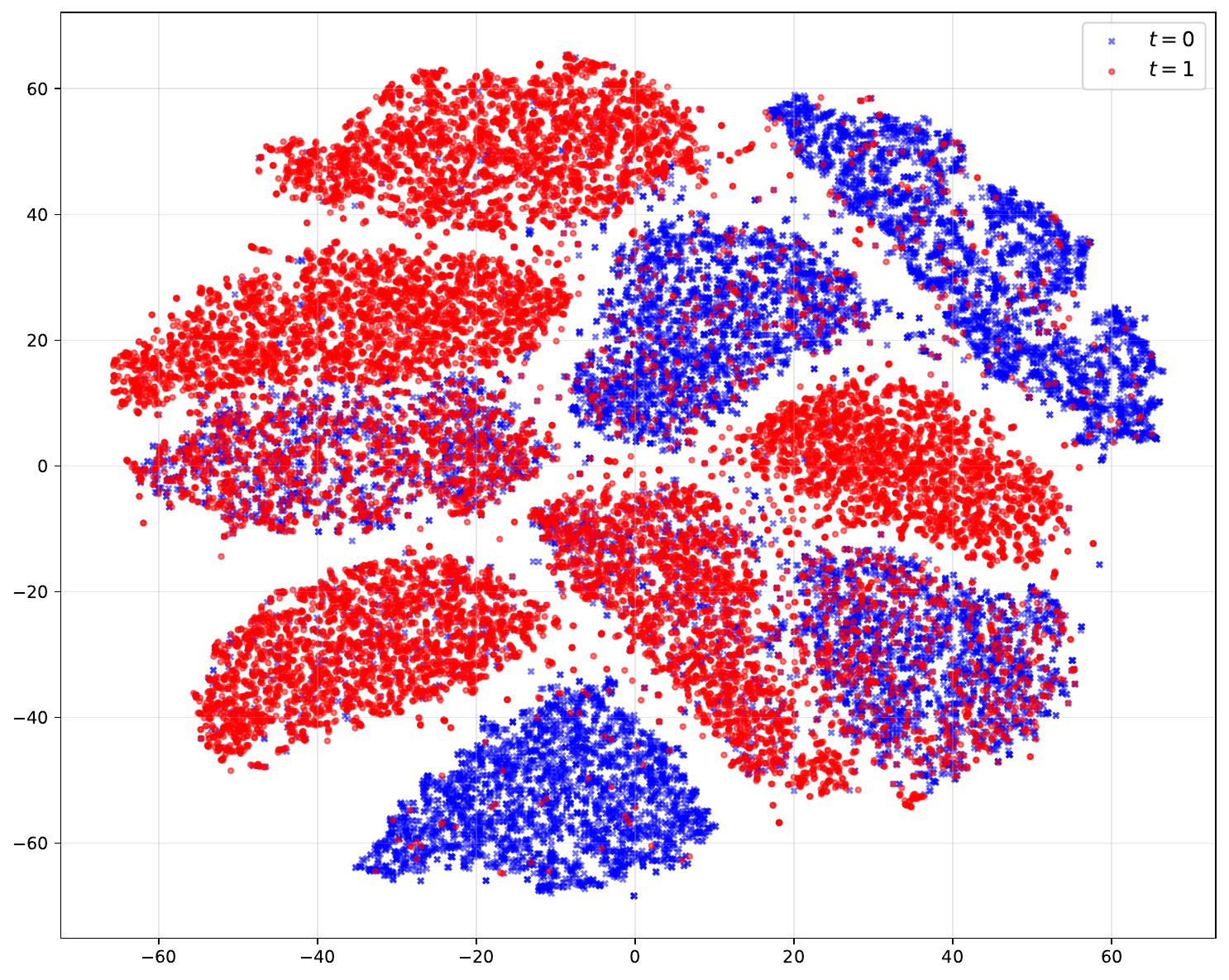}\label{fig:cmnist_distribution}}
    \subfigure[$t=1$]{\includegraphics[width=0.33\textwidth]{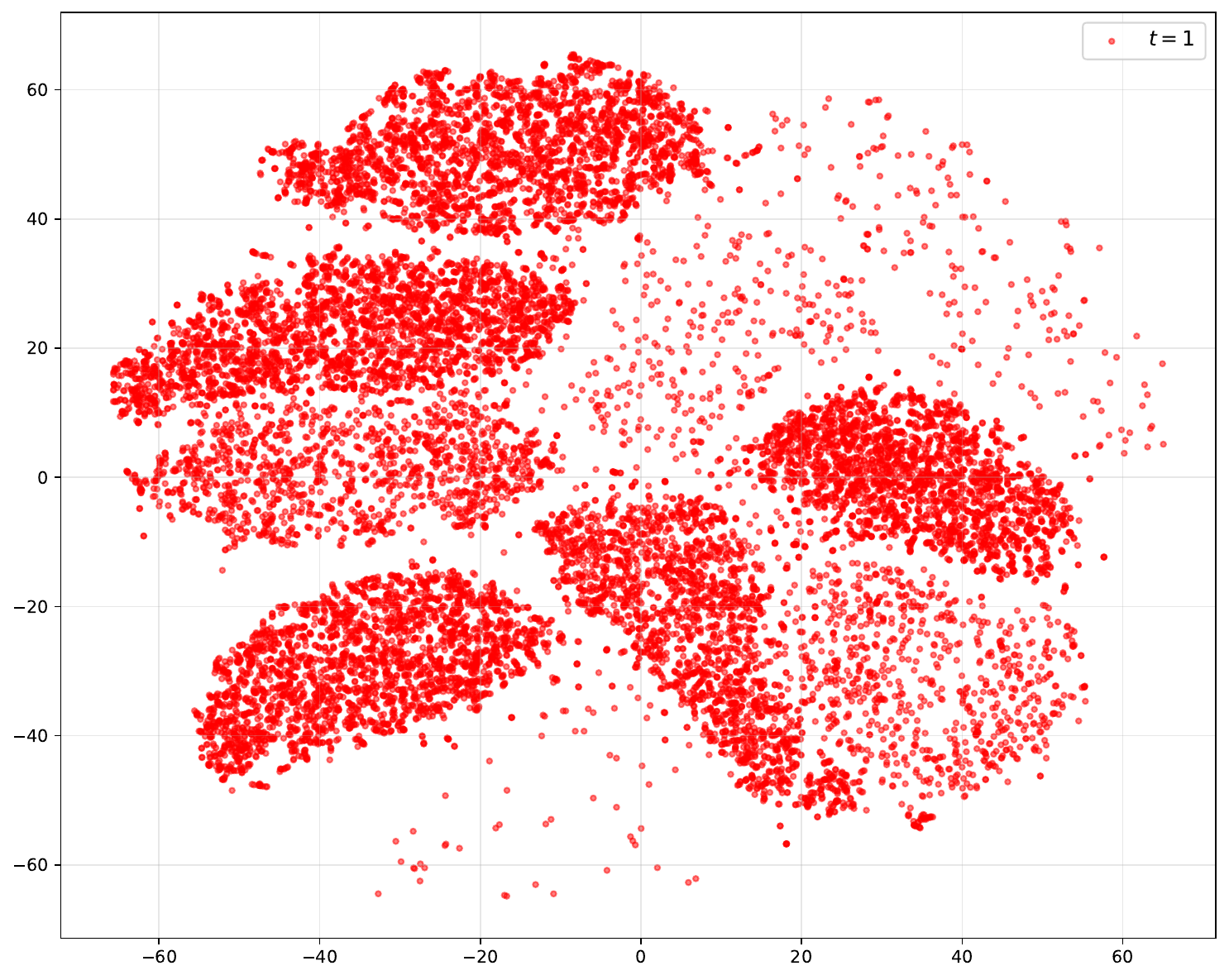}\label{fig:cmnist_distribution_1}}
  \subfigure[$t=0$]{\includegraphics[width=0.33\textwidth]{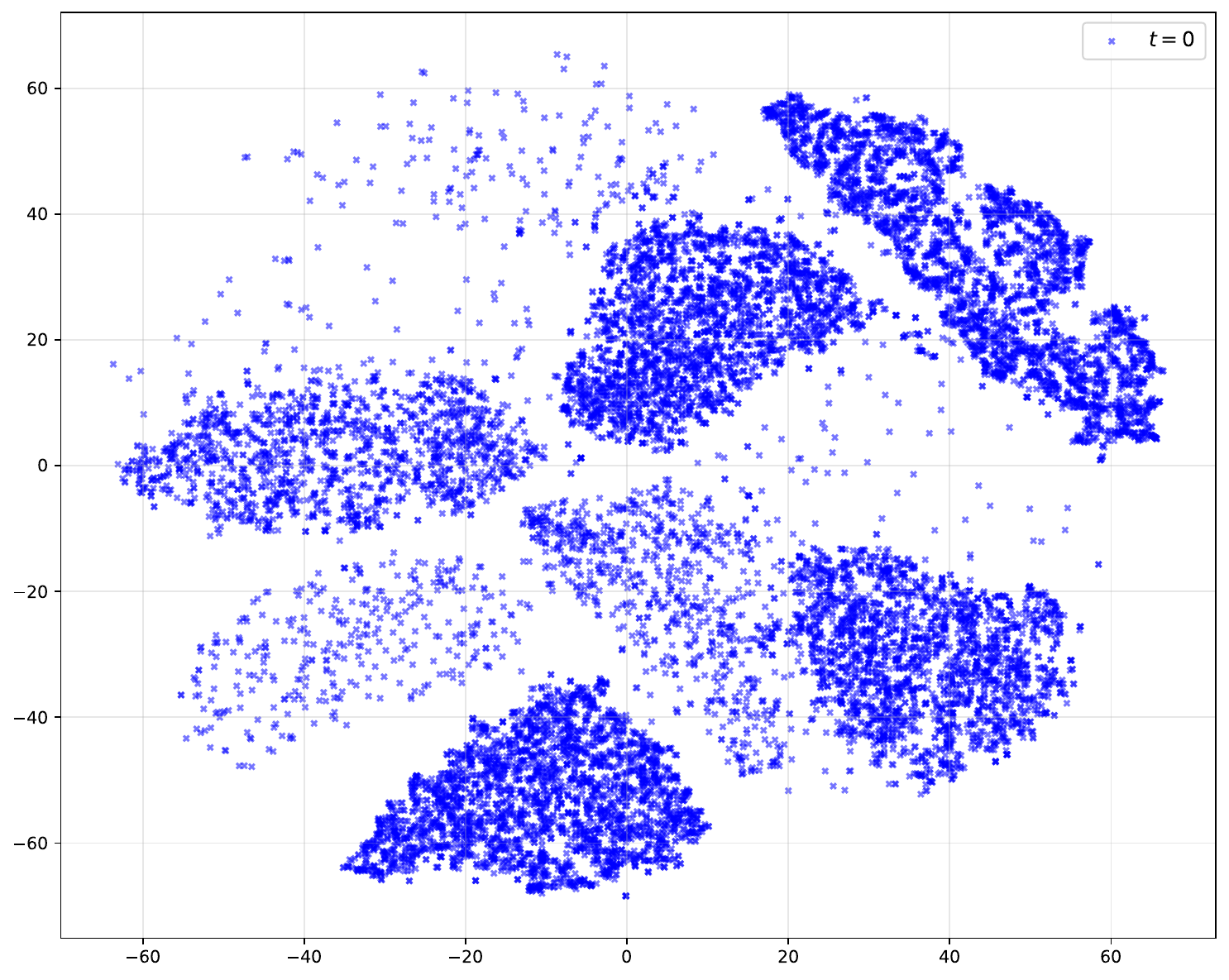}\label{fig:cmnist_distribution_0}}
  \caption{High dimensional data projected into 2-dimensional space via t-SNE, with colors indicating \textcolor{purple}{overlapping}, \textcolor{red}{group $t=1$}, and \textcolor{blue}{group $t=0$}. Left: entire data distribution; Mid: distribution on group $t=1$; Right: distribution on group $t=0$.}
  \label{fig:ablation_analysis}
\end{figure} 

It is seen that in Section \ref{section:ablation} for ablation study, the performance difference shown in Table \ref{table:ablation_study} indicates that the no measurable and significant performance gain on IBM and CMNIST dataset, respectively. Here, we qualitatively discuss the causes for this phenomenon from the data distribution perspective.

The entire data distribution for IBM is plot in Figure \ref{fig:ibm_distribution}, where a well-blended distribution is seen. We further split the entire data distribution into sole plots on Figure \ref{fig:ibm_distribution_1} ($t=1$) and Figure \ref{fig:ibm_distribution_0} ($t=0$). It is observed that data distributions on different treatment groups are almost identical, which is aligned to the dataset description in the original paper that generates the IBM benchmark \cite{shimoni2018benchmarking}. Thus, the high-density region is built on the overlapping region since two groups are fully overlapped. Therefore, given all the samples on group $t=0$ known, when data acquisition starts on group $t=1$, algorithm only needs to acquire the sample from the high-density regions as the overlapping condition is synchronously satisfied. Thus, FCCM and $\text{FCCM}^{-}$ is indistinguishable on such data distribution. In the meanwhile, it is observed that the entire CMNIST data distribution (Figure \ref{fig:cmnist_distribution}) embeds significantly less overlapping regions with the covariates extracted from the MNIST \cite{lecun1998mnist} benchmark which has 10 classes of data (10 digits). Unlike IBM \cite{shimoni2018benchmarking}, the data generating process in \cite{jesson2021quantifying} creates significant distribution discrepancy for two treatment groups, where the high-density regions for group $t=1$ (Figure \ref{fig:cmnist_distribution_1}) and group $t=0$ (Figure \ref{fig:cmnist_distribution_0}) are well-distinguished, leaving less overlapping regions on the entire data distribution. Therefore, our proposed method -- FCCM prioritizes the data acquisition toward the high-density and overlapping region, while $\text{FCCM}^{-}$ only focus on high-density region where less overlapping is seen, thus resulting in significant performance difference as seen in Table \ref{table:ablation_study}.

\subsection{Practicability of the Assumptions\label{appendix:discussions_on_assumptions}}

Strong ignorability: The validity of the SI can be approximated by carefully selecting sufficient relevant covariates and constructing a more balanced dataset.

Lipschitz continuity: Theorem 3.3 assumes the Lipschitzness of the $p^{t}(y|\mathbf{x})$, this can be a practical assumption if the regression model $f$, e.g., a well-regularized neural network (NN), learns a smooth mapping from $\mathbf{x}$ to $y$, which implies the Lipschitzness for $p^{t}(y|\mathbf{x})$. Also, the NN $f$ is differentiable w.r.t. $\mathbf{x}$, thus the squared loss $l_{f}$ is also bounded and sufficiently differentiable w.r.t. $\mathbf{x}$, further implying the Lipschitzness of $l_{f}$.

Constant $\kappa$: The existence of the constant $\kappa$ is thoroughly discussed by \citet{shalit2017estimating} in Appendix A.3 and A.4.

Additionally, we discuss consequence for the proposed theorems when the above-mentioned assumptions are not satisfied:

Lipschitz continuity for Theorem \ref{theorem:overall}: If the Lipschitzness does not hold, the multiplicative constant will be unbounded, and thus the reduction of the radii may not help control the risk upper bound. Strong ignorability for Theorem \ref{theorem:2opt}: Strong ignorability provides an ideal scenario for acquiring the counterfactual samples, where a quick bound reduction by Algorithm 1 is seen in Figure \ref{fig:ideal_convergence}. If strong ignorability cannot be guaranteed in real-world data, e.g., CMNIST, the reduction of the bound will be significantly slower and less effective for Algorithm \ref{alg:fccs}, as shown in Figure \ref{fig:realistic_convergence}. Thus, it motivates us to propose Algorithm \ref{alg:fccm} that can handle compromised data distributions more effectively via a slight trade-off on coverage.

\subsection{Hyperparameters Tuning\label{appendix:model_trianing}}

We conduct all the experiments with 24GB NVIDIA RTX-3090 GPU on Ubuntu 22.04 LTS platform with the 12th Gen Intel i7-12700K 12-Core 20-Thread CPU. As stated in the main text, for fair comparison, we take the consistent hyperparameters tuned in \cite{jesson2021causal,wen2024progressive} for the estimators: \textbf{DUE-DNN} \cite{van2021feature} and \textbf{DUE-CNN} \cite{van2021feature} shown in Table \ref{table:estimator_search_space}. Additionally, we search the best hyperparameters, i.e., covering radius $\delta$ and edge weight $\alpha$ for counterfactual linkage, for Algorithm \ref{alg:fccm} with the validation set shown in Table \ref{table:algorithm_search_space}. Note, that in Section \ref{section:estimating_radius}, we approximating a narrower range around 95\% threshold to further determine the covering radius for the best performance. 
\begin{table}[h!]
\centering
\begin{minipage}[t]{0.35\textwidth} % Adjust width as needed
\centering
\caption{Hyperparameters for Estimators}
\begin{tabularx}{\textwidth}{l>{\centering\arraybackslash}X>{\centering\arraybackslash}X} % Adjusted column widths
\toprule\label{table:estimator_search_space}
\text{Hyperparameters} & \text{DNN} & \text{CNN} \\
\midrule
Kernel  & RBF & Matern \\
Inducing Points  & 100 & 100 \\
Hidden Neurons  & 200 & 200 \\
Depth  & 3 & 2 \\
Dropout Rate  & 0.1 & 0.05 \\
Spectral Norm  & 0.95 & 3.0 \\
%Batch Size  & 100 & 64 \\
Learning Rate  & 1e\textsuperscript{-3} & 1e\textsuperscript{-3} \\
\bottomrule
\end{tabularx}
\end{minipage}%
\hspace{0.02\textwidth} % Horizontal space between the two tables
\begin{minipage}[t]{0.6\textwidth} % Adjust width as needed
\centering
\caption{Hyperparameters for Algorithm \ref{alg:fccm}}
\begin{tabularx}{\textwidth}{l>{\centering\arraybackslash}X>{\centering\arraybackslash}X} % Adjusted column widths
\toprule\label{table:algorithm_search_space}
\text{Hyperparameters} & \text{Search Space} & Tuned \\
\midrule
$\delta_{(1,1)}$ for TOY & [0.11, 0.12, 0.13] & 0.11  \\
$\delta_{(1,0)}$ for TOY & [0.11, 0.12, 0.13] & 0.11  \\
$\delta_{(1,1)}$ for IBM  & [0.11, 0.13, 0.15] & 0.11 \\
$\delta_{(1,0)}$ for IBM  & [0.11, 0.13, 0.15] & 0.11 \\
$\delta_{(1,1)}$ for CMNIST  & [0.40, 0.45, 0.50] & 0.50 \\
$\delta_{(1,0)}$ for CMNIST  & [0.40, 0.45, 0.50] & 0.40 \\
Edge weight $\alpha$ & [1.0, 2.5, 5.0] & 2.5  \\
\bottomrule
\end{tabularx}
\end{minipage}
\end{table}

\newpage
\subsection{Main Results under Highest Resolution\label{appendix:high_resolution}}
\begin{figure}[h!]
    \centering
    \subfigure[2\% increment on TOY dataset]{
    \includegraphics[width=0.31\linewidth]{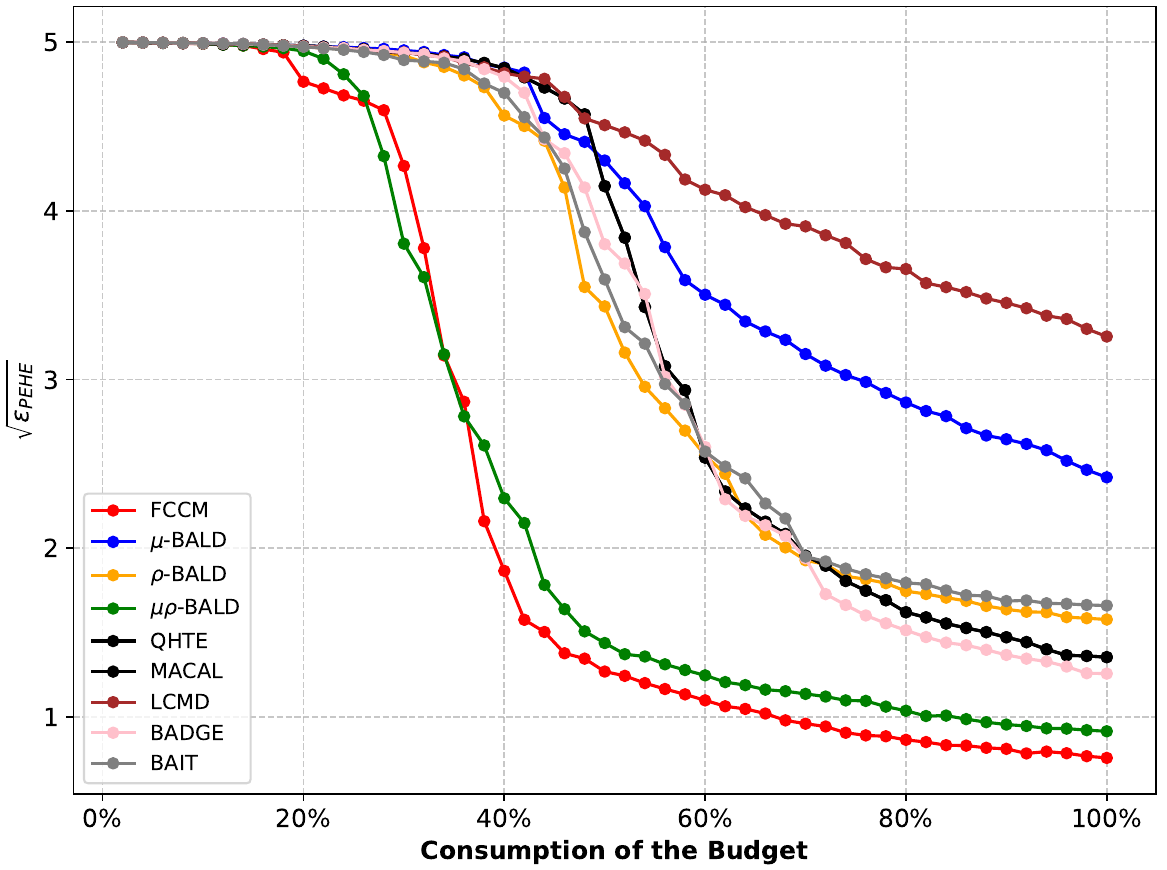}}
    \subfigure[2\% increment on IBM dataset]{
    \includegraphics[width=0.31\linewidth]{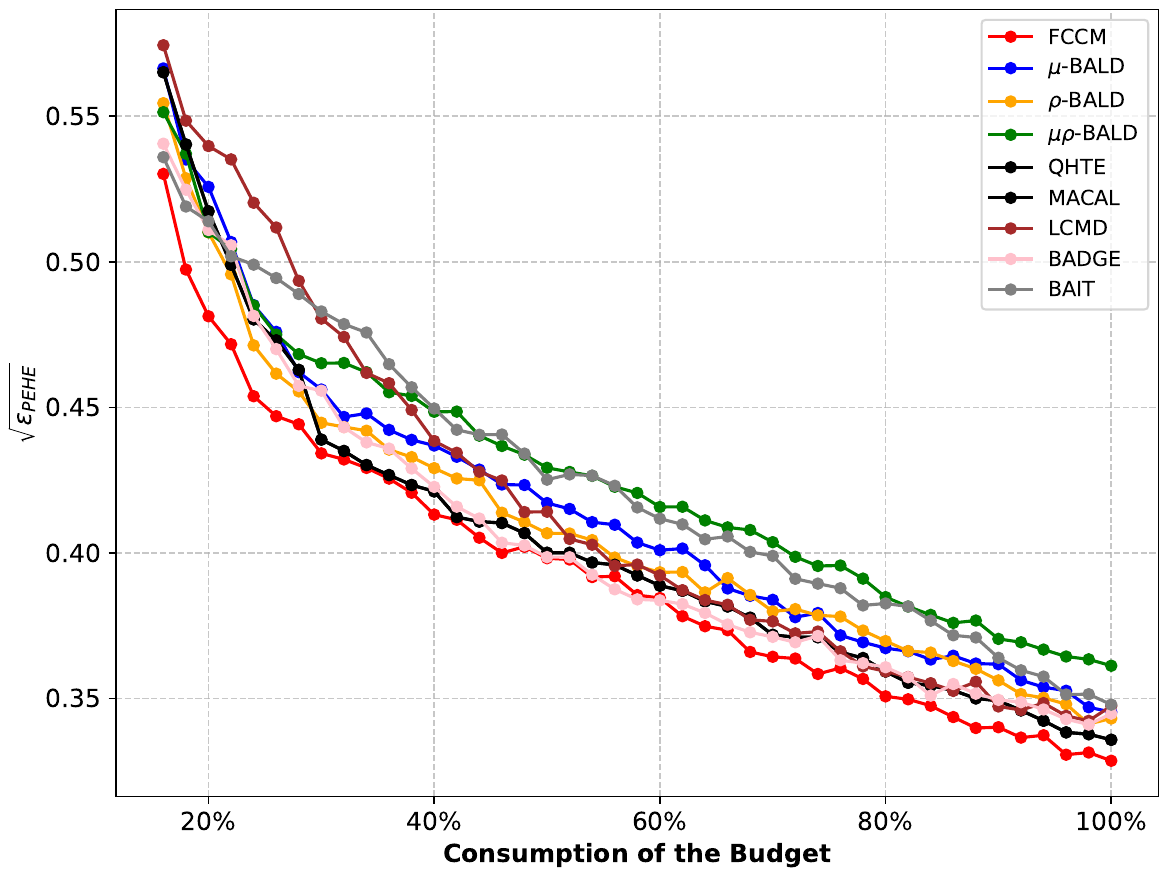}}
    \subfigure[2\% increment on CMNIST dataset]{
    \includegraphics[width=0.31\linewidth]{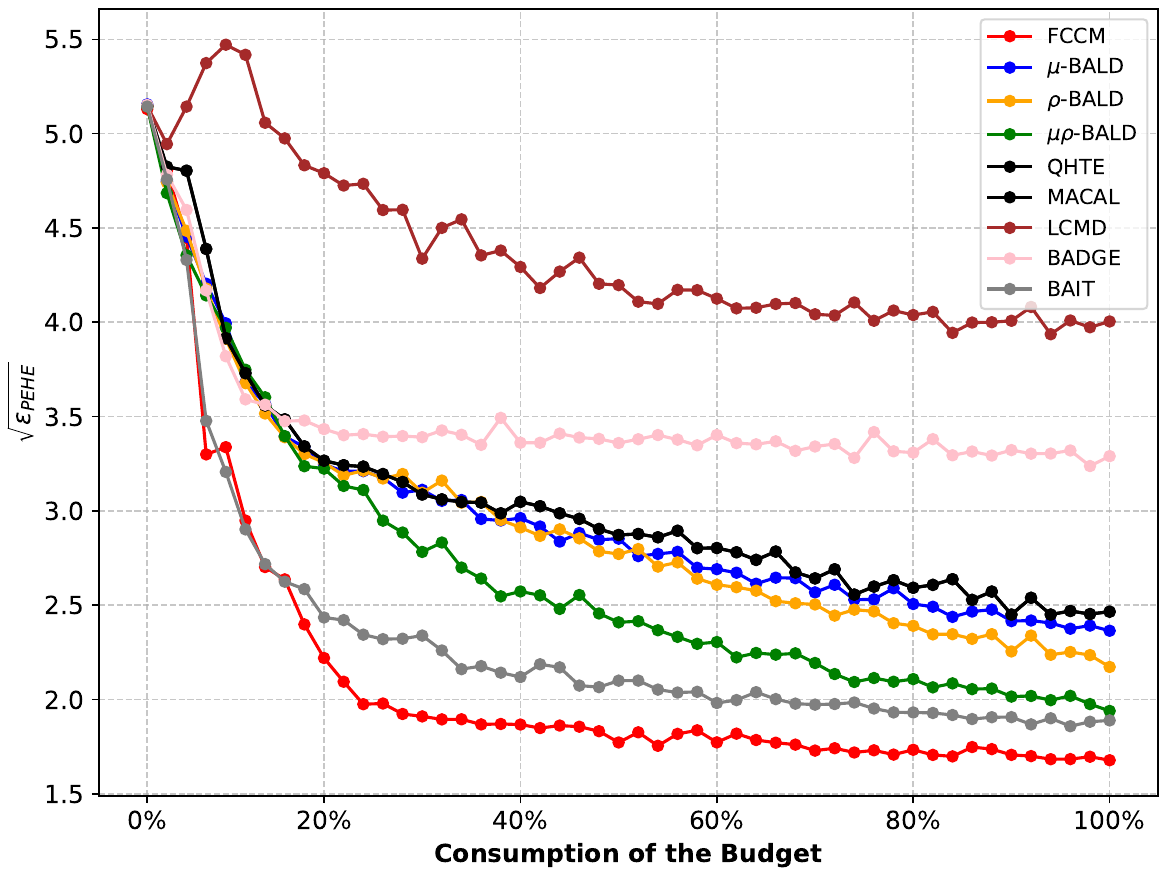}}
    %\caption{Caption}
    \label{fig:enter-label}
\end{figure}

\begin{comment}
\setcounter{table}{0}
\begin{table}[h!]
\centering
\caption{Ablation Study of the Counterfactual Covering Radii on $\sqrt{\epsilon_{PEHE}}$}
\small
\begin{tabularx}{.7\textwidth}{X X *{5}{>{\centering\arraybackslash}X}} % Added an extra X column for "Method"
\toprule
\multicolumn{1}{l}{\multirow{2}{*}{Dataset}} & \multicolumn{1}{l}{\multirow{2}{*}{Method}} & \multicolumn{5}{c}{Consumption of the Total Budget} \\ % Added "Method" column
\cmidrule(l){3-7}
& & \multicolumn{1}{c}{1/5} & \multicolumn{1}{c}{2/5} & \multicolumn{1}{c}{3/5} & \multicolumn{1}{c}{4/5} & \multicolumn{1}{c}{5/5} \\
\toprule
\multirow{3}{*}{TOY} & FCCM- & 4.7680 & 2.2496 & 1.3372 & 1.0545 & 0.9024 \\  
& FCCM & 4.7664 & 1.8655 & 1.0978 & 0.8637 & 0.7565 \\ 
\cmidrule(l){2-7}
& Gain &+0\%&+17\%&+18\%&+18\%&+16\%\\
\midrule
\multirow{3}{*}{IBM} & FCCM- & 0.4745 & 0.4088 & 0.3797 & 0.3512 & 0.3291 \\  
& FCCM & 0.4813 & 0.4132 & 0.3845 & 0.3507 & 0.3286 \\ 
\cmidrule(l){2-7}
& Gain &-1\%&-1\%&-1\%&+0\%&+0\%\\
\midrule
\multirow{3}{*}{CMNIST} & FCCM- & 2.9250 & 2.6627 & 2.4652 & 2.3105 & 2.2073 \\  
& FCCM & 2.2207 & 1.8681 & 1.7735 & 1.7344 & 1.6790 \\ 
\cmidrule(l){2-7}
& Gain &+24\%&+30\%&+28\%&+25\%&+24\%\\
\bottomrule
\end{tabularx}
\end{table}
\end{comment}

\subsection{Sensitivity Study\label{appendix:sensitivity_study}}
Note that the acquisition on treatment sample $t=1$ is insensitive on $\delta_{(0,0)}$ and $\delta_{(0,1)}$ in our setting, as all control samples ($t=0$) are seen. For $\alpha$, our setting of $\alpha=2.5$ has an overall lower error across different acquisition budgets.
\begin{figure}[h!]
    \centering
    \subfigure[Analysis for the weight $\alpha$ on TOY]{
    \includegraphics[width=0.31\linewidth]{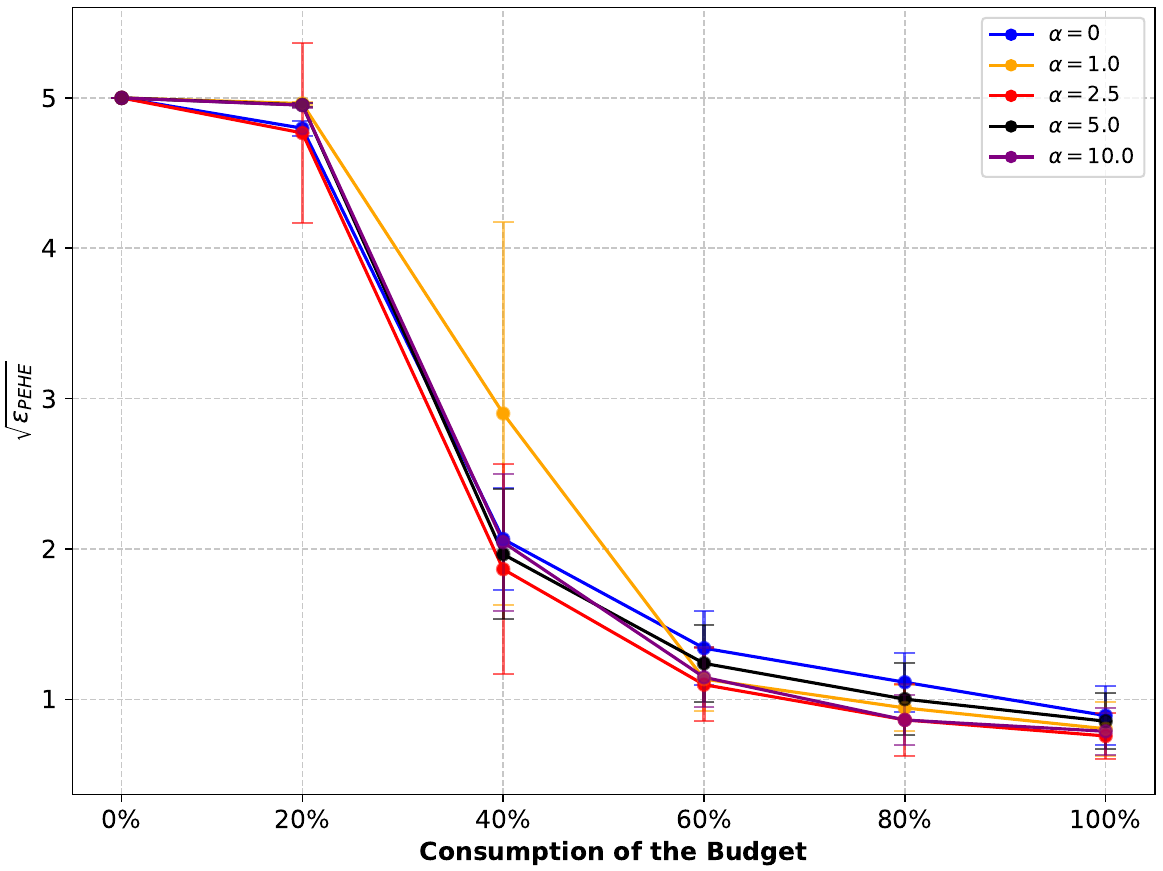}}
    \subfigure[Analysis for the weight $\alpha$ on IBM]{
    \includegraphics[width=0.31\linewidth]{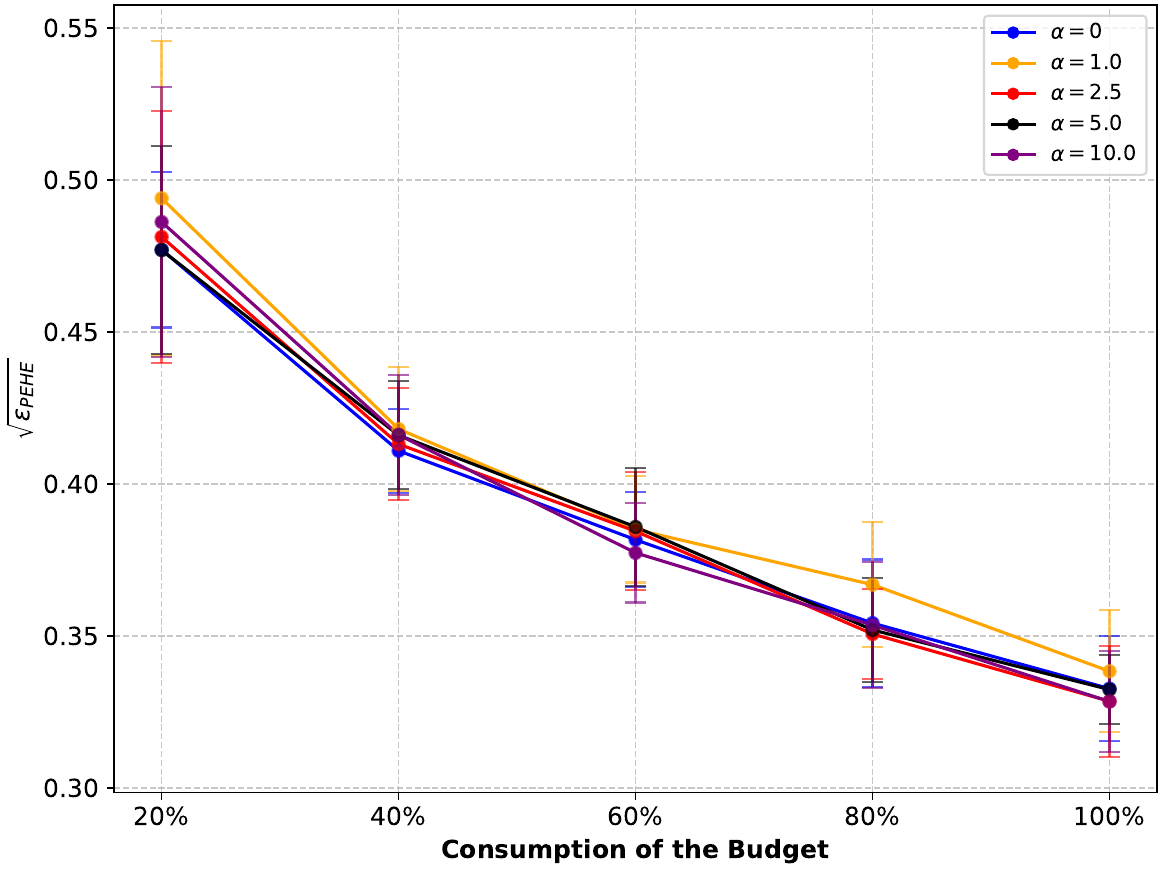}}
    \subfigure[Analysis for the weight $\alpha$ on CMNIST]{\includegraphics[width=0.31\linewidth]{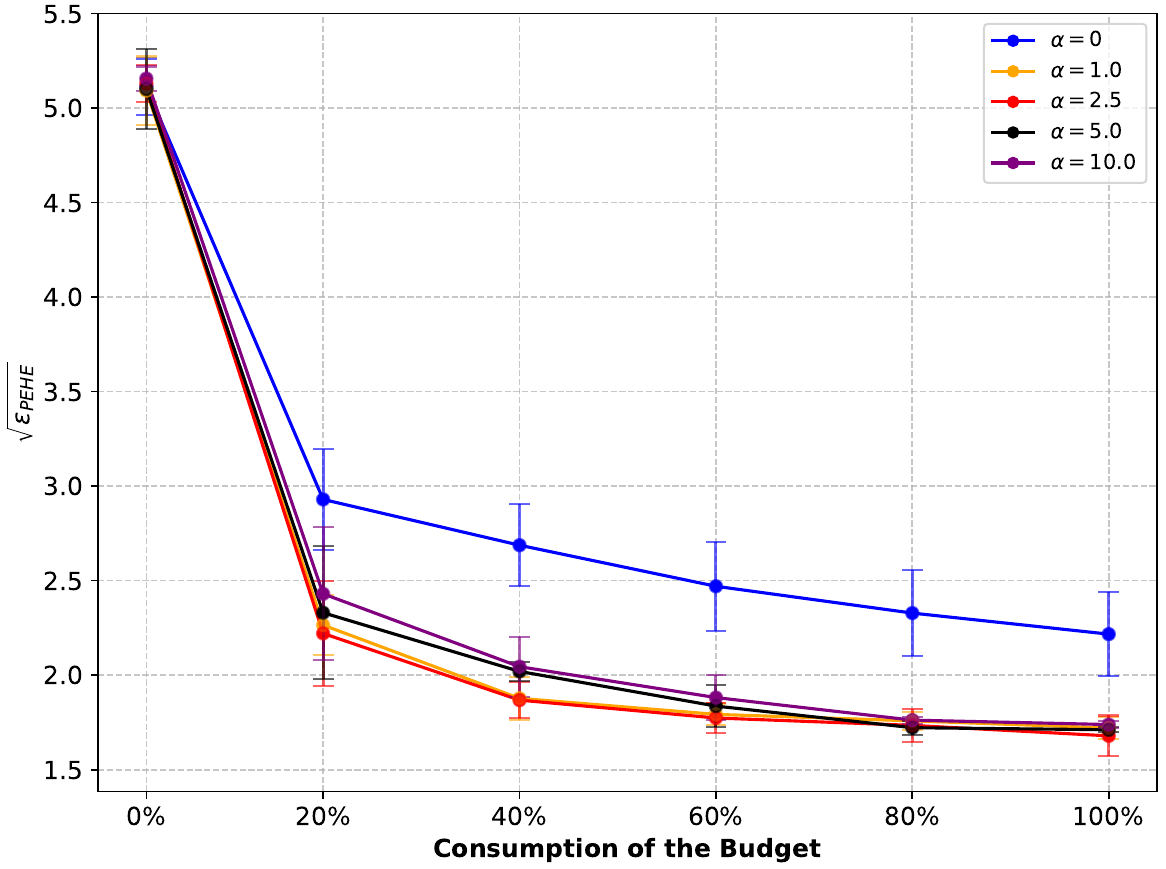}}\\
    \subfigure[Analysis for $\delta_{(1,1)}$ on TOY]{
    \includegraphics[width=0.31\linewidth]{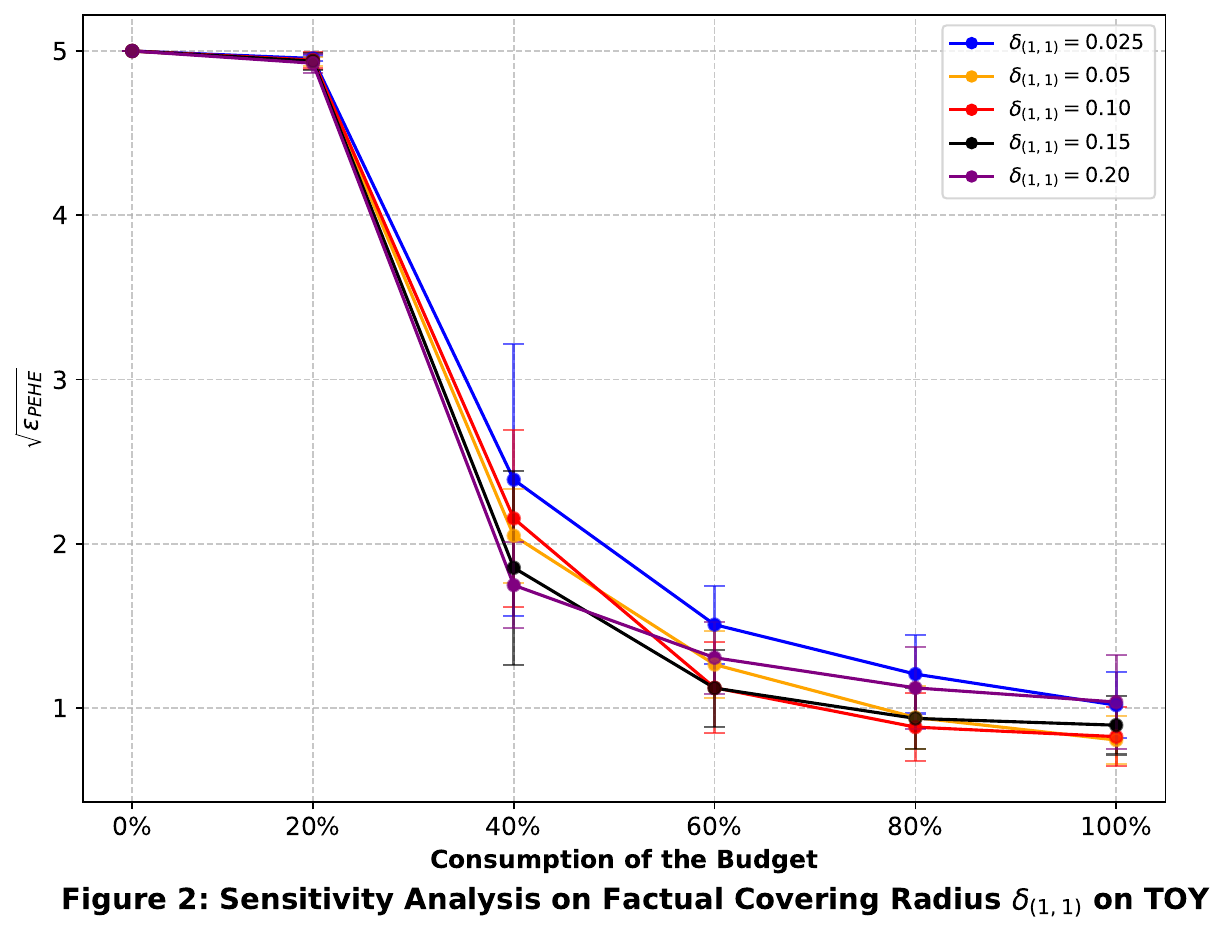}}
    \subfigure[Analysis for $\delta_{(1,1)}$ on IBM]{\includegraphics[width=0.31\linewidth]{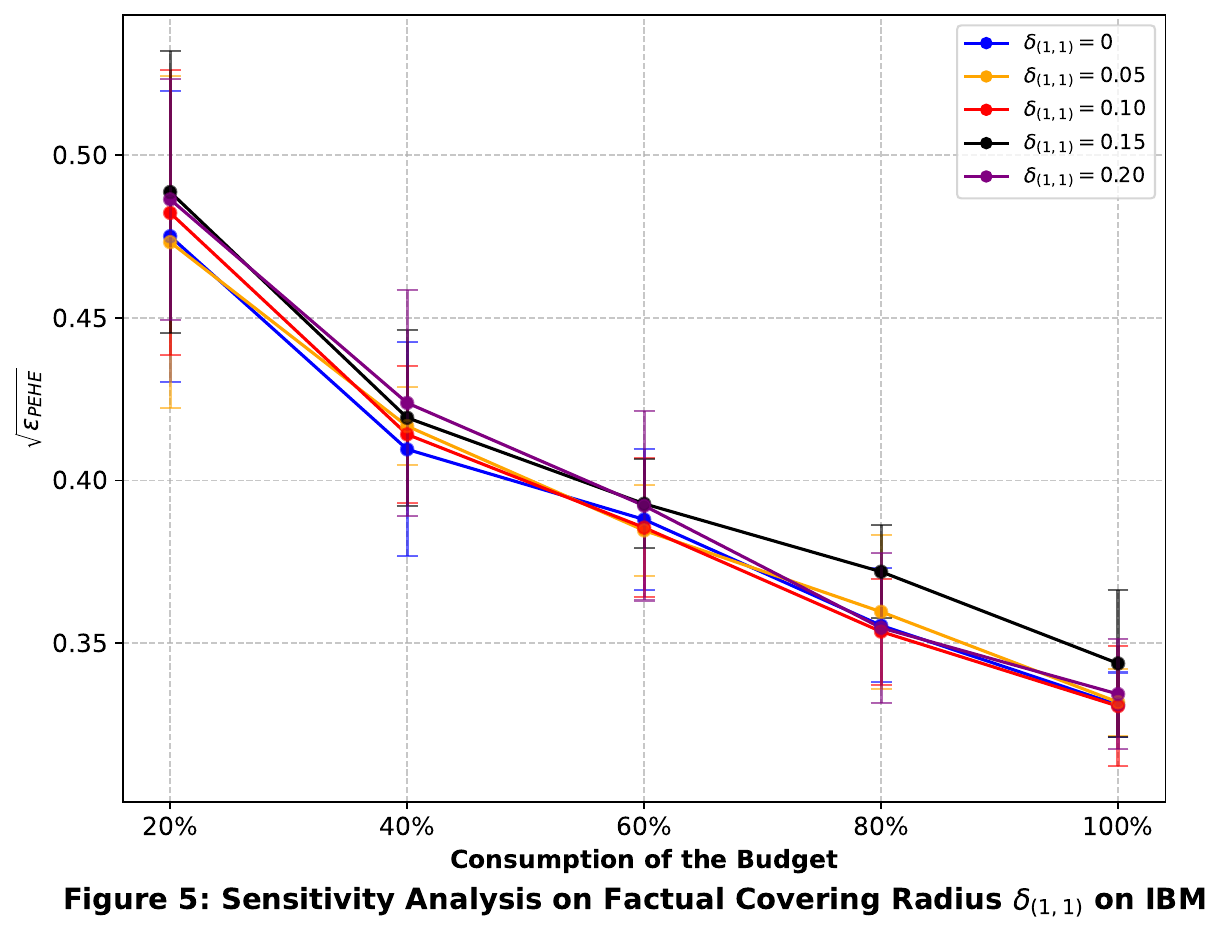}}
    \subfigure[Analysis for $\delta_{(1,1)}$ on CMNIST]{\includegraphics[width=0.31\linewidth]{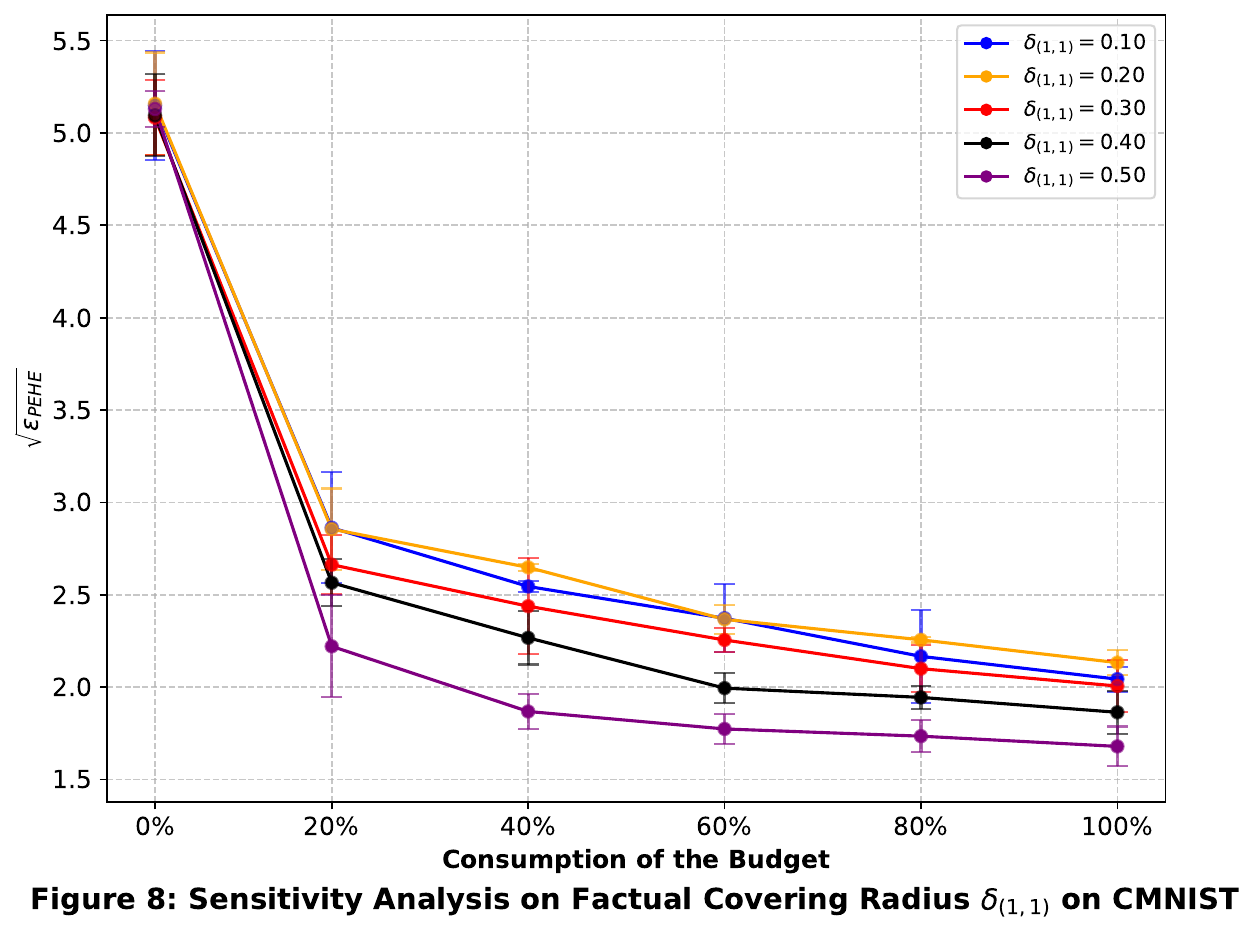}}\\
    \subfigure[Analysis for $\delta_{(1,0)}$ on TOY]{
    \includegraphics[width=0.31\linewidth]{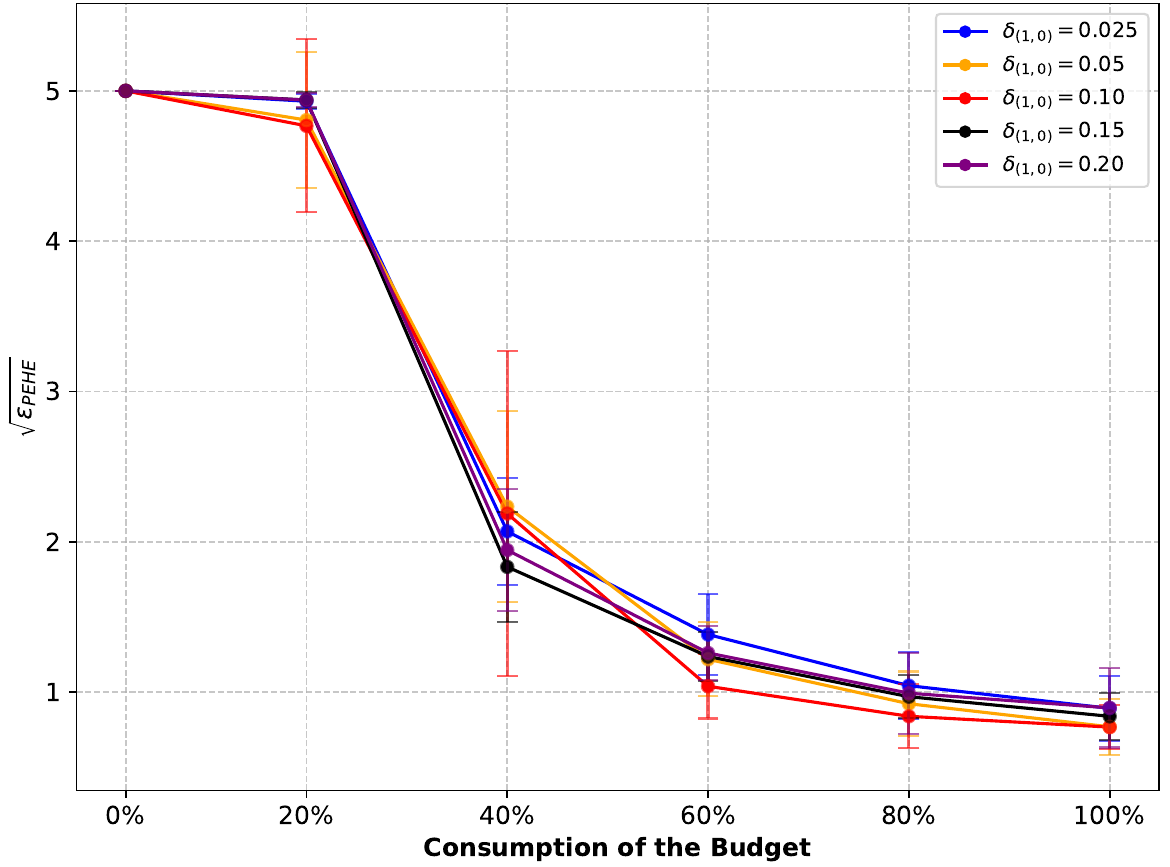}}\subfigure[Analysis for $\delta_{(1,0)}$ on IBM]{
    \includegraphics[width=0.31\linewidth]{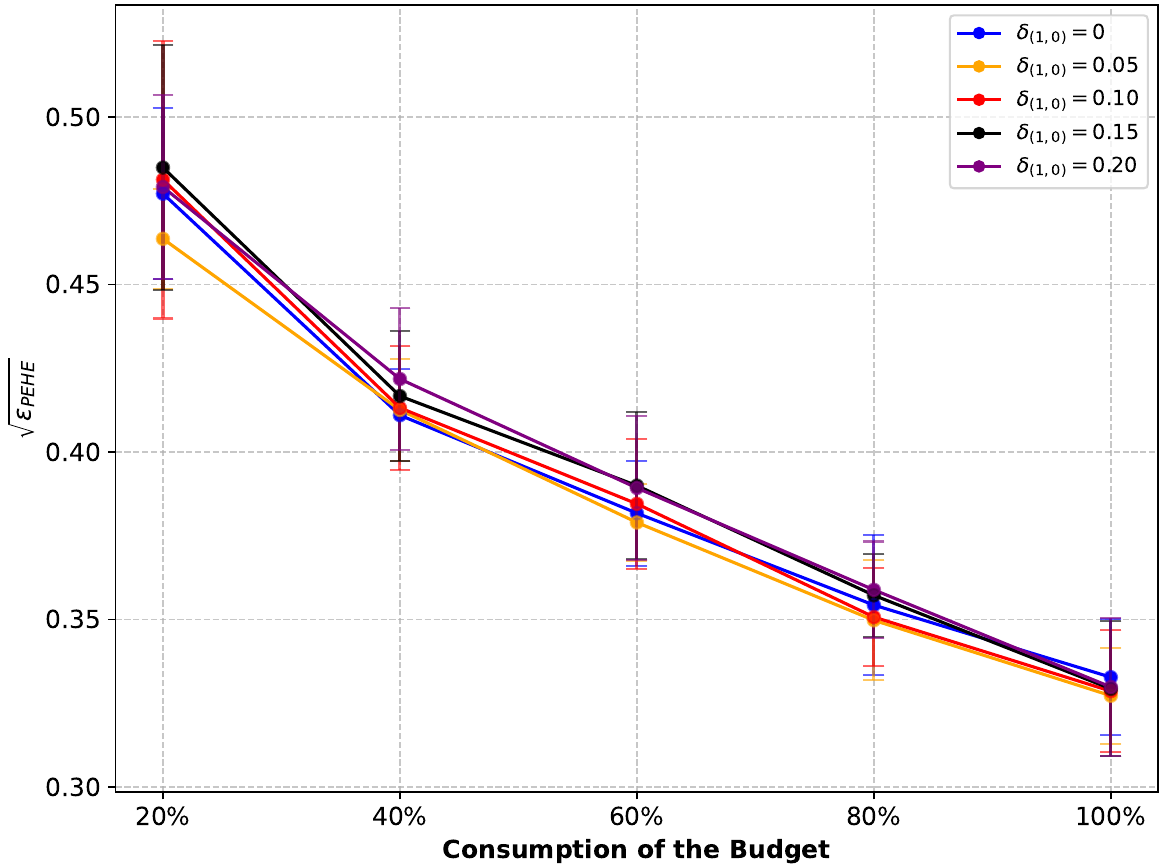}}\subfigure[Analysis for $\delta_{(1,0)}$ on CMNIST]{\includegraphics[width=0.31\linewidth]{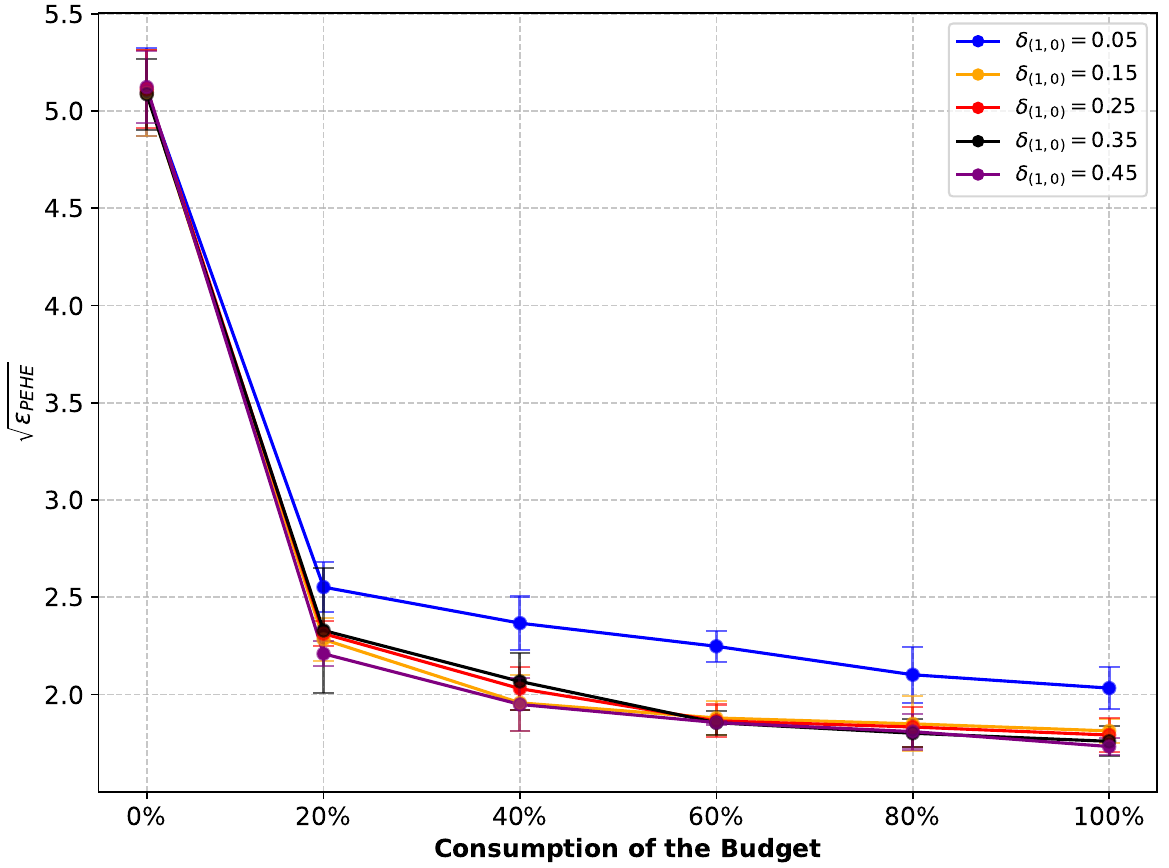}}\\
    %\caption{Caption}
    \label{fig:enter-label-}
\end{figure}

\end{document}